\documentclass[pdflatex,iicol]{sn-jnl}% Math and Physical Sciences Numbered Reference Style
\usepackage{graphicx}%
\usepackage{multirow}%
\usepackage{amsmath,amssymb,amsfonts,mathtools}%
\usepackage{amsthm}%
\usepackage{mathrsfs}%
\usepackage[title]{appendix}%
\usepackage{xcolor}%
\usepackage{textcomp}%
\usepackage{manyfoot}%
\usepackage{booktabs}%
\usepackage{algorithm}%
\usepackage{algorithmicx}%
\usepackage{algpseudocode}%
\usepackage{listings}%
\theoremstyle{thmstyleone}%
\newtheorem{theorem}{Theorem}%  meant for continuous numbers
\newtheorem{lemma}{Lemma}%  meant for continuous numbers
\theoremstyle{thmstyletwo}%

\theoremstyle{thmstylethree}%
\newtheorem{definition}{Definition}%

\begin{document}

\title[Explainable Binary Classification of Separable Shape Ensembles]{Explainable Binary Classification of Separable Shape Ensembles}

%%=============================================================%%
%% GivenName	-> \fnm{Joergen W.}
%% Particle	-> \spfx{van der} -> surname prefix
%% FamilyName	-> \sur{Ploeg}
%% Suffix	-> \sfx{IV}
%% \author*[1,2]{\fnm{Joergen W.} \spfx{van der} \sur{Ploeg} 
%%  \sfx{IV}}\email{iauthor@gmail.com}
%%=============================================================%%

\author*[1]{\fnm{Zachary} \sur{Grey}}\email{zachary.grey@nist.gov}
\equalcont{These authors contributed equally to this work.}

\author[2]{\fnm{Nicholas} \sur{Fisher}}\email{nicholfi@pdx.edu}
\equalcont{These authors contributed equally to this work.}

\author[3]{\fnm{Andrew} \sur{Glaws}}\email{andrew.glaws@nrel.gov}

\affil*[1]{\orgdiv{Applied \& Computational Mathematics Div.}, \orgname{National Institute of Standards \& Technology}, \orgaddress{\street{325 S. Broadway}, \city{Boulder}, \postcode{80305}, \state{CO}, \country{USA}}}

\affil[2]{\orgdiv{Dept. of Mathematics and Statistics}, \orgname{Portland State University}, \orgaddress{\street{1825 SW Broadway}, \city{Portland}, \postcode{97201}, \state{OR}, \country{USA}}}

\affil[3]{\orgdiv{Computational Science Center}, \orgname{National Laboratory of the Rockies}, \orgaddress{\street{15013 Denver West Parkway}, \city{Golden}, \postcode{80401}, \state{CO}, \country{USA}}}

%%==================================%%
%% Sample for unstructured abstract %%
%%==================================%%

\abstract{Scientists, engineers, biologists, and technology specialists universally leverage image segmentation to extract shape ensembles containing many thousands of curves representing patterns in observations and measurements. These large curve ensembles facilitate inferences about important changes when comparing and contrasting images. We introduce novel pattern recognition formalisms combined with inference methods over large ensembles of segmented curves. Our formalism involves accurately approximating eigenspaces of composite integral operators to motivate discrete, dual representations of curves collocated at quadrature nodes. Approximations are projected onto underlying matrix manifolds and the resulting separable shape tensors constitute rigid-invariant decompositions of curves into generalized (linear) scale variations and complementary (nonlinear) undulations. With thousands of curves segmented from pairs of images, we demonstrate how data-driven features of separable shape tensors inform explainable binary classification utilizing a product maximum mean discrepancy; absent labeled data, building interpretable feature spaces in seconds without high performance computation, and detecting discrepancies below cursory visual inspections.}


%%\pacs[JEL Classification]{D8, H51}


\maketitle

\section{Introduction}
\label{sec:intro}

Precise measurement and quantification of shape is ubiquitous in imaging science. For example, applications involve engineering design~\cite{Grey2017,Doronina2023}, materials science~\cite{atindama2023restoration,bachmann2010inferential,Fan2020,FAN2021116810}, medical imaging~\cite{mang2019claire,durrleman2014morphometry}, functional data analysis~\cite{tucker2013generative, srivastava2016functional}, biology~\cite{hartman2023elastic,srivastava2010shape,hagwood2013testing}, and ecology~\cite{lynch2023satellite,gonccalves2021fine}. Many of these applications benefit from separate parametrizations of shape planar translations (\emph{where is it?}), scale variations (\emph{how big is it?}), rotations/reflections (\emph{how is it oriented?}), and everything else that remains (\emph{how nonlinear is it?}). In several applications, a choice of invariance against one or multiple of types of similarity transforms~\cite{kendall2009shape}---among other transformations~\cite{bauer2014overview}---is beneficial for defining \emph{shape}. And in a variety of applications, scientists and engineers simply want to ask a binary question: \emph{is this the same as that}? But often necessitate some explanation: \emph{if they're different, why?} We refer to solutions of this general class of problems as \emph{explainable binary classifications}.

Applications involve a variety of imaging modalities including but are not limited to: electron microscopy, X-ray computed tomography, optical imagery from satellite, video, or cameras, photo-luminescent cell microscopy, solar ultraviolet imaging, modeled or measured level-sets of Lagrangian averaged vorticity deviations, images generated by artificial intelligence (AI), etc. 

In the first example, electron backscatter diffraction (EBSD)~\cite{schwartz2009electron} of material samples can offer high resolution image segmentation of the material's microstructure~\cite{SAVILLE2021102118,stoudt2020location}. Denoised images~\cite{atindama2023restoration} and subsequent segmentation~\cite{bachmann2010texture,FAN2021116810} lead to a large ensemble of planar curves as data. An example of a segmented EBSD image given by deformed ice, utilizing an open-source software called MTEX~\cite{Hielscher:cg5083,Fan2020,FAN2021116810}, is shown in Figure~\ref{fig:eg_grains}. The segmented image in Figure~\ref{fig:eg_grains} is juxtaposed to an example `grain boundary' or, simply, `grain' representing a single extracted planar curve from the tiled ensemble. 

Beyond this particular imaging modality, our interpretations are extensible to \textit{any imaging combined with subsequent segmentation informing a large ensemble of planar curves}. Additional examples of alternative image modalities are shown in~\ref{fig:eg_other_images}.

\subsection{Problem Statement}

\begin{quote}
    \emph{Are the curves/shapes in one image, up to rigid motions, the same or different from that of another and, if they are different, are differences due to linear deformations (generalized scale variations) or nonlinear deformations (undulations)?}
\end{quote}

\begin{figure*}
    \centering
    \includegraphics[width=0.45\textwidth]{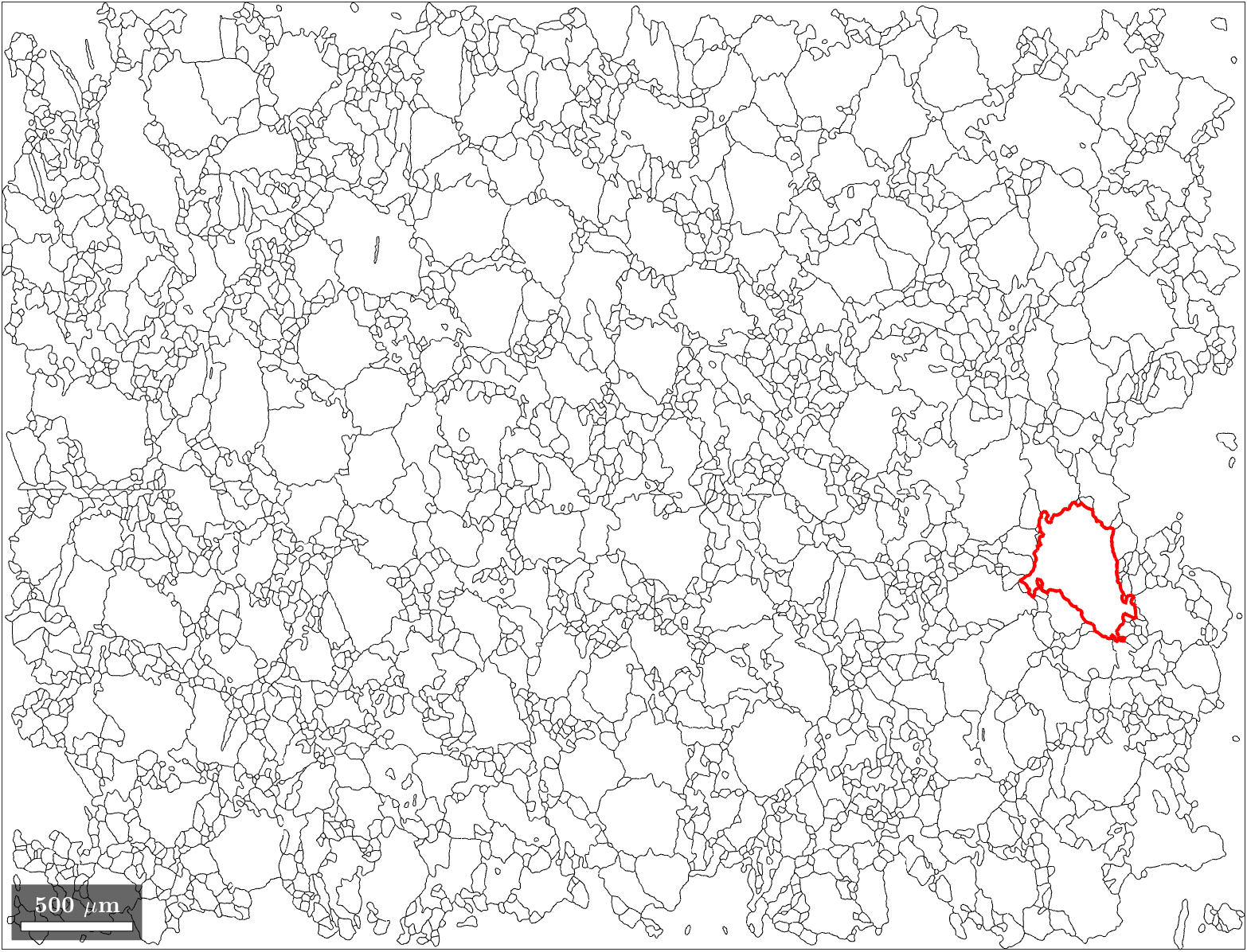}
    \hspace{0.5cm}
    \includegraphics[width=0.275\textwidth]{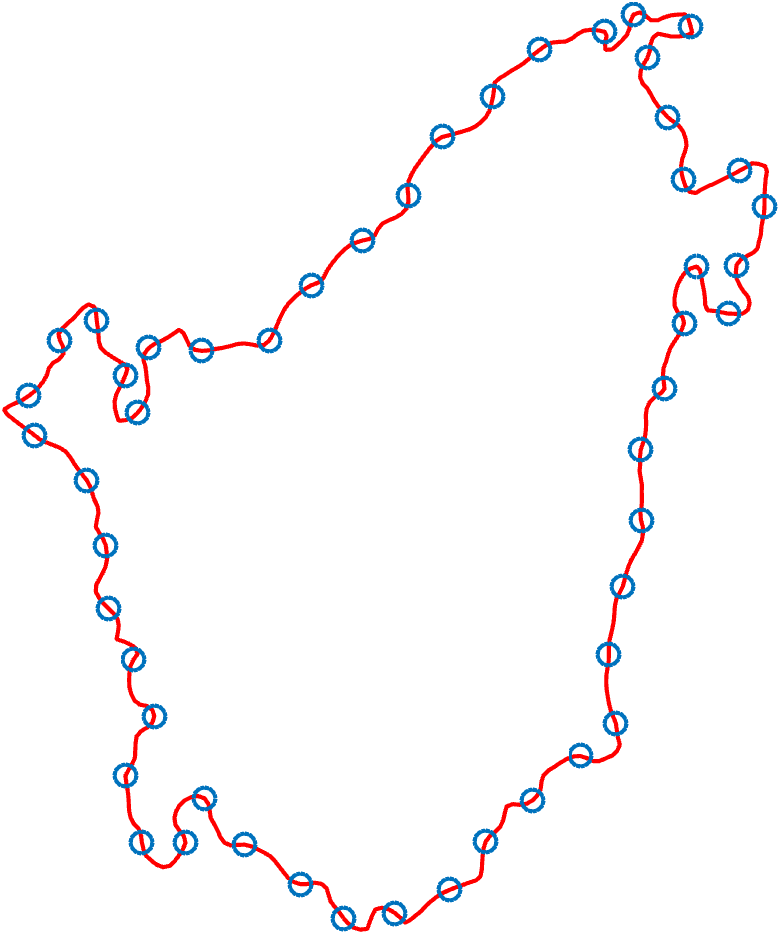}
    \caption{An example ensemble of thousands of grain boundaries from an EBSD image~\cite{FAN2021116810,Fan2020} (left) and an example segmented grain boundary (right) with arc-length reparametrization landmarks (blue circles) generated by an interpolating curve (red). Data is available online~\cite{Fan2020}. The micron bar in the lower left corner of the EBSD grain boundaries image reads $500$ micrometers.}
    \label{fig:eg_grains}
\end{figure*}

\begin{figure*}
    \centering
    \includegraphics[width=0.465\textwidth]{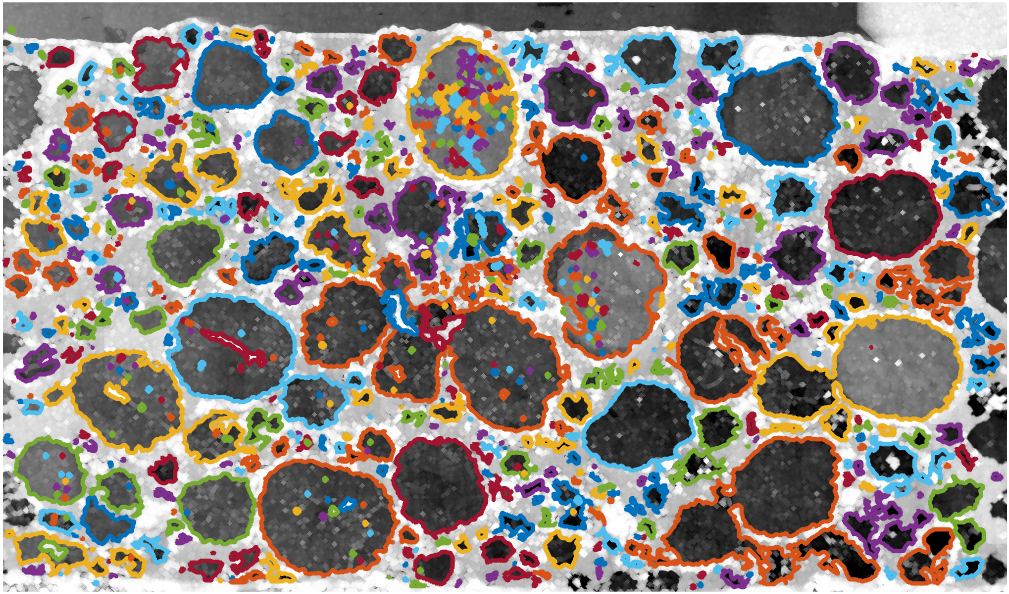}
    \hspace{.25cm}
    \includegraphics[width=0.275\textwidth]{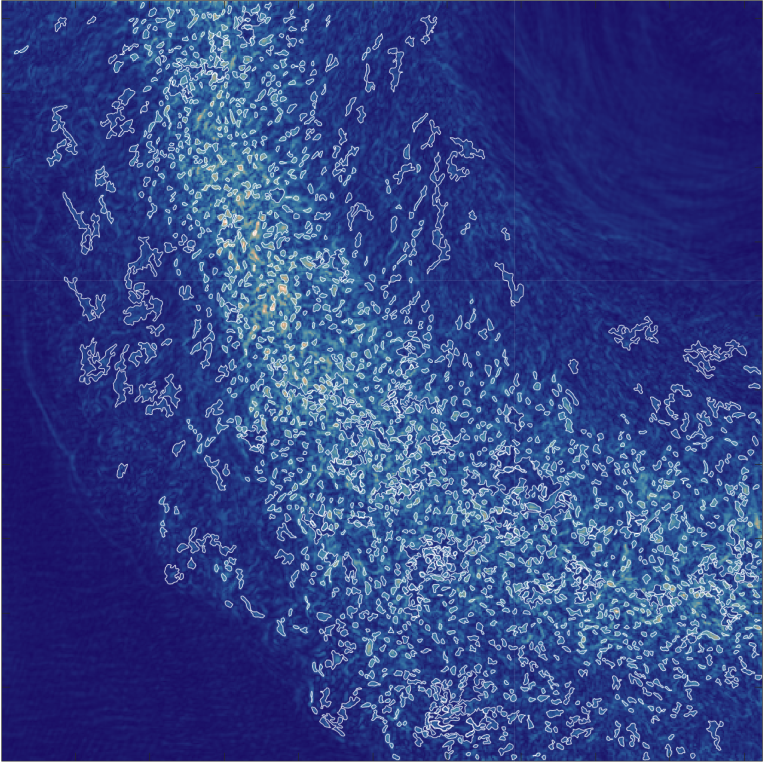}
    \hspace{.25cm}
    \includegraphics[width=0.2\textwidth]{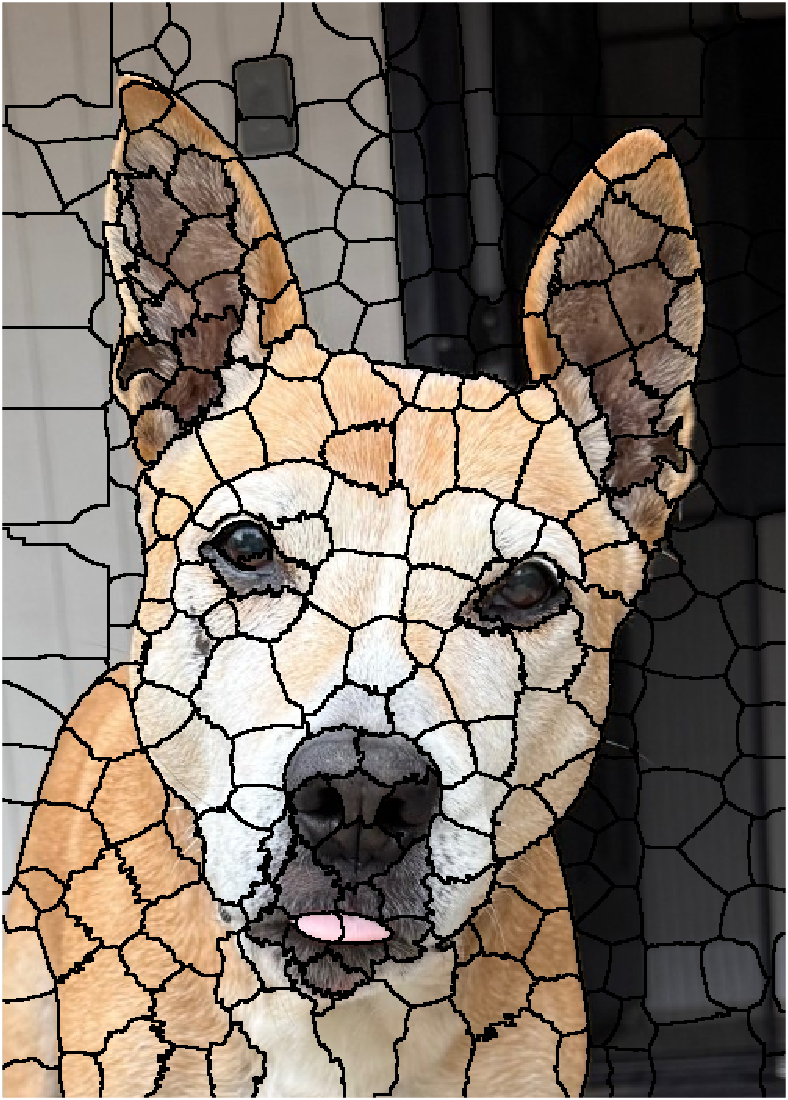}
    \caption{(left) A scanning electron microscopy of a lithium-ion battery cross-section, (middle) Lagrangian coherent structures in the lower-left quadrant of a cyclone modeled by large-eddy simulation, (right) a `superpixel' segmentation of canine named Penni.}
    \label{fig:eg_other_images}
\end{figure*}

\noindent We present a first attempt to guide binary classification of imaging modalities with explainable methods and formal interpretations. This is in contrast to a more general class of AI models. AI models may accomplish a similar task \emph{if previously trained against subjective hand-labeled data} but often lack quantifiable explanations due to the opaque nature of the `learned features' defined over an ambiguous `latent space.' 

The notion of `difference' in this context necessitates an invariance to rigid motions and is statistical given the stochastic nature of the imaging and the means by which curves and patterns are extracted utilizing segmentation algorithms, as depicted in~Figure~\ref{fig:eg_grains}. In other words, noisy images and shapes extracted by segmentation are considered random. A na\"ive deterministic perspective that assesses pointwise differences between images is uninformative since no two sampled images and subsequent image processing are ever the same despite similarities in patterns encoded in the images.

Moreover, scientists often utilize expertise and visual inspection to suggest how simple characteristics---e.g., principal lengths, area, perimeter, size, sphericity, and other intuitive features---vary between patterns in one image to the next~\cite{stoudt2020location, ASTM_E562, ASTM_E930, ASTM_E112}. While such quantities can provide a cursory description of the ensemble characteristics, these handpicked features of shape do not constitute comprehensive representations of shape and descriptions of differences can collapse. For example, it is straight forward to find examples of two shapes with identical hand-picked shape characteristics but other distinguishing features. Despite this, certain commercial standards~\cite{kazakov2019astm, ASTM_E562, ASTM_E930, ASTM_E112} have been authored to detect hand-picked features. We expect our shape representations will offer dramatic improvements and insights for future standards across a variety of imaging modalities beyond the materials science examples.

The stochastic notion of difference in this application is based on (maximum mean) \emph{discrepancy}: are the approximated distributions of \emph{representative} shape features over the ensemble in one image statistically significant from the ensemble in another? To answer this question, we first i) represent shapes augmented by \emph{separating important distinguishing features}, ii) discretize to \emph{approximate feature distributions} of the shapes from both images, and finally iii) \emph{test for statistically significant discrepancies} in distributions between two ensembles. This final step is referred to as a two-sample test~\cite{gretton2012kernel}. For the specific materials application under consideration, this corresponds to \emph{statistically comparing thousands of shapes in two ensembles given two EBSD images}. %Moreover, we augment the statistical test by also determining if those statistically significant differences are a result of linear or nonlinear features of shape distortions which, in combination, provably draw the same logical conclusions in the aggregate.

It is important to note that shape is only one significant piece of the quite literal puzzles in Figures~\ref{fig:eg_grains}-\ref{fig:eg_other_images}. However, for this work, we ignore the role of textures intrinsic to the tiling---i.e., the spatial arrangement of the shapes---and instead treat the ensemble of shapes as independent realizations from some generating distribution. %Related efforts are focused towards unifying both the topology of the spatial tiling and features of shapes in image. 

\subsection{Technical Outline}
In subsequent introductory sections, we review related work and motivate a class of shapes for consideration which is sufficient in representing segmented curves from images.

In Section~\ref{sec:FIE}, we formalize interpretations transforming continuous boundary curves as preshapes into separately encode principled features of \textit{generalized scale} and \textit{undulation}. Specifically, in Thm.~\ref{theorem:eigfuncs}, we offer a novel interpretation that separable shape tensors are approximations of dual representations over a vector-valued reproducing kernel Hilbert space (RKHS). 

In Section~\ref{sec:numerics}, we utilize discrete approximations of the formally developed shape features to test four logical combinations of statistical hypotheses with data. Numerically, the discretizations are weighted according to a reinterpretation of Nystr\"om method, coined \textit{square-root quadrature decomposition} (SRQD), and matrix decompositions are motivated as approximations of curve (dual) representations. Finally, we utilize discrete approximations to map into parameter distributions over underlying matrix manifolds as a feature space informed by the large ensembles of curves in the images. We then discuss the implications of an explainable binary classification informed by maximum mean discrepancy (MMD) over separate parameter distributions supported on the learned feature space.

Section~\ref{sec:experiments} discusses implications of all four logical decisions over the separable space using the materials EBSD data as an example. Conclusions and future work follow in Section~\ref{sec:conclude}.

\subsection{Contributions}
We emphasize the following contributions of our explainable shape classification method:
\begin{enumerate}
    \item At a minimum, only pairs of images with ample segmented curves are required to inform decisions,
    \item the method does not require training against hand-labeled data, (intrinsic discrepancies)
    \item the presentation offers theoretical interpretations of learned features spaces as (dual) representations of curves, and 
    \item explanations are provably logical statistical classifications over separable parameters of shape features.
\end{enumerate}
 
 In contrast with AI methods, i) we do not require challenging minimization over vast swaths of hand-labeled training data (i.e., efficient implementation), ii) our feature/latent space is provably an approximation of dual representations of shape (i.e., mathematically interpretable), and iii) our classification is provably logical (i.e., trustworthy explanations).

Our numerical experiments emphasize four scenarios where these tools deliver specific benefits: i) Figure~\ref{fig:compare_grains_111} demonstrates that the method is not suspected to be overly sensitive to statistical rejection, ii) Figure~\ref{fig:compare_grains_100} demonstrates that the method emulates conclusions consistent with human observations, iii) Figure~\ref{fig:compare_grains_010} demonstrates that the method detects differences below human-scale visual observations and, iv) Figure~\ref{fig:compare_grains_000} demonstrates that the method is capable of detecting a presumed faulty measurement.

Our primary contribution is a formal interpretation that discrete transformations are collocations of dual representations presented in~\cite{micheli2013matrix}. In other words, we relate novel interpretations of the SST configuration space over finite-dimensional matrix manifolds to a functional analysis over an infinite-dimensional RKHS. This interpretation facilitates a flexible choice of kernel which can be used to control the regularity of a Hilbert space of curves~\cite{micheli2013matrix} as hypothesized in any application. This motivates novel explanations and extensions between landmark configurations and the $\mathbb{R}^d$-valued RKHSs in~\cite{micheli2013matrix} representing infinite-dimensional features of curves. This is an innovation to classical configuration spaces~\cite{kendall2009shape} which have no demonstrated connection to an infinite-dimensional space of curves.

\subsection{Notation}
In general, bold fonts distinguish vector-valued objects from capitalized fonts which represent matrix-valued objects. We reuse common conventions for the follow spaces of matrices:
\begin{itemize}
    \item $GL(d,\mathbb{R})$ is the space of $d$-by-$d$ full rank matrices with real entries while $GL_+(d,\mathbb{R})$ is the subgroup of $GL(d,\mathbb{R})$ with positive determinant.
    \item $Gr(d,n)$ is the Grassmannian, the space of $d$-dimensional subspaces of $\mathbb{R}^n$.
    \item $O(d)$ is the space of $d$-by-$d$ orthogonal matrices while $SO(d)$ is the subgroup of $O(d)$ with determinant equal to one.
    \item $S^d_{++}$ is the cone of $d$-by-$d$ symmetric positive definite matrices.
\end{itemize}
Elsewhere, we define new notation explicitly and use ``$\coloneqq$'' to indicate that the identification is by definition.

\subsection{Separable Shape Tensors} \label{subsec:SST}
We begin by reviewing related work~\cite{grey2023separable, bryner20142d, srivastava2016functional} for defining finite representations of curve data as \textit{configuration spaces}. Following the development of SST~\cite{grey2023separable}, initially applied to aerodynamic shape representations~\cite{doronina2022grassmannian}, we let $X = (\boldsymbol{x}_1,\dots,\boldsymbol{x}_n)^{\top}$ be the collection of $n$ unique \emph{landmarks} informed by some vector-valued $\boldsymbol{c}\in \mathcal{C}_d$, $\boldsymbol{c}:\mathcal{I} \subset \mathbb{R} \rightarrow \mathbb{R}^d$, such that
\begin{equation}\label{eq:data}
    \boldsymbol{x}_i = \boldsymbol{c}(s_i), \quad i=1,\dots,n.
\end{equation}

Note that $X$ is assumed to be an element of the non-compact Stiefel manifold $\mathbb{R}^{n \times d}_*$~\cite{absil2008optimization}. That is, we assume $X \in \mathbb{R}^{n \times d}_*$ to exclude the degenerate cases of points or lines in $\mathbb{R}^d$ as objects representing curves~\cite{grey2023separable}---necessitating that $X$ is full-rank. In other words, \textit{we assume the component functions $\boldsymbol{c} = (c_1,\dots,c_d)$ of the curve are linearly independent}. 

In practice, we are often only given sequences of landmarks $(\boldsymbol{x}_i)$ and have no knowledge of the `true' underlying $\boldsymbol{c}$ nor the parameter $s$ beyond assuming the identification above. However, this characterization of the data implies landmarks are ordered according to some orientation of the unknown curve thus constituting a sequence, $(\boldsymbol{x}_i)$, encoded over the row index of $X$. Thus, it is reasonable to build a variety of interpolations and approximations---e.g., piecewise spline interpolations over chordal parametrization---which can be shown to converge to the unknown $\boldsymbol{c}$ with an increasing number of given landmarks~\cite{grey2023separable, floater2005arc}.

The method of SST proceeds by introducing a separable form of the curve into features of \emph{undulation} and \emph{scale/rotation/shear/reflection} as linear deformations for an aerodynamic design. Theoretical treatments of `discrete shapes' as full-rank matrices~\cite{bryner20142d} or complex vectors~\cite{kendall2009shape} offer geometric interpretations of the thin singular value decomposition (SVD),
\begin{equation} \label{eq:landmark_stnd}
    (X - B(X))^{\top} \,\,\overset{SVD}{=}\,\, U\Sigma V^{\top},
\end{equation}
with $B(X) = (1/n)\boldsymbol{1}_{n,n}X$, an unbiased centering operation where $\boldsymbol{1}_{n,n}\in \mathbb{R}^{n\times n}$ is a matrix of ones. 

Given the SVD of $X- B(X)$ representing a discrete shape, we can elaborate on related decompositions into complementary factors of (discrete) undulation and invertible $d$-by-$d$ linear deformations---i.e., in a complementary sense, an `undulation' is the set of all shape representations which cannot be achieved by linear transforms as right group actions on $X - B(X)$. That is, temporarily ignoring non-deforming translations, for $\boldsymbol{c}(s) = M^{\top}\widetilde{\boldsymbol{c}}(s)$,
\begin{align} \label{eq:linear_deforms}
    X^{\top} &= (\boldsymbol{c}(s_1),\dots,\boldsymbol{c}(s_n))\\ \nonumber
    &= M^{\top}(\widetilde{\boldsymbol{c}}(s_1),\dots,\widetilde{\boldsymbol{c}}(s_n)) \\ \nonumber
    &= (\widetilde{X}M)^{\top} \nonumber
\end{align}
where the right action of $M \in GL(d,\mathbb{R})$ on $\widetilde{X} = (\widetilde{\boldsymbol{c}}(s_1),\dots,\widetilde{\boldsymbol{c}}(s_n))^{\top}\in \mathbb{R}^{n \times d}_*$ is akin to linear deformations of the preshape. Given (centered) $X$, we seek a decomposition $X = \widetilde{X}M$ to describe how $X$ \textit{undulates} as $[\widetilde{X}]$, an equivalence class modulo right group action over $GL(d,\mathbb{R})$, versus how it linearly deforms over subsets of $M \in GL(d,\mathbb{R})$. Therefore, we define discrete shape \textit{undulation} as $[\widetilde{X}] \in Gr(d,n)$ such that $Gr(d,n) \cong \mathbb{R}^{n \times d}_*/GL(d,\mathbb{R})$ as presented in~\cite{absil2008optimization}. 

Two geometric interpretations of the SVD~\eqref{eq:landmark_stnd} are called \emph{landmark standardizations}~\cite{bryner20142d, grey2023separable, srivastava2016functional} to compute \textit{representative undulations} $\widetilde{X}$ (up to rotations and reflections):

\begin{itemize}
    \item The affine standardization~\cite{bryner20142d} $X = \widetilde{X}M$ such that 
    $\widetilde{X} \coloneqq V$ and $ M \coloneqq \Sigma U^{\top} \in GL_+(d,\mathbb{R})$.
    \item The polar standardization~\cite{grey2023separable} $X = \widetilde{X}P$ such that
    $\widetilde{X} \coloneqq VU^{\top}$ and $P \coloneqq U\Sigma U^{\top}\in S^d_{++}$.
\end{itemize}
The latter polar standardization offers an interesting decomposition for a variety of problems which benefit from describing undulation and scale variations independent of rotations and reflections. In other words, following the development of~\cite{fletcher2003statistics}, $S^d_{++}$ is identified with the set of equivalence classes given by $GL_+(d,\mathbb{R})/SO(d)$. Thus $[\widetilde{X}]$ encodes nonlinear undulating features of shape over the quotient topology $\mathbb{R}^{n \times d}_*/GL(d,\mathbb{R})$ while $P$ encodes generalized (anisotropic) \emph{scales} independent of rotations and reflections over the quotient topology $GL_+(d,\mathbb{R})/SO(d)$.

Landmark standardization is thought to have a strong connection to the work of Kendall et al.~\cite{kendall2009shape} albeit with distinct notions of scale. In our case, \emph{the definition of scale is generalized} to all of $S^d_{++}$ as opposed to simple dilation of curves as a scalar multiplication to modulate overall size~\cite{srivastava2010shape, kendall2009shape, srivastava2016functional}. 

Given the SVD \eqref{eq:landmark_stnd} and sought equivalence classes, the pairs $\lbrace([\widetilde{X}], P)\rbrace \in Gr(d,n) \times S^d_{++}$ inform a \emph{data-driven} submanifold of the proposed product matrix-manifold to parametrize separable representations of shape tensors,
\begin{equation}\label{eq:SST}
    X(\boldsymbol{t},\boldsymbol{\ell}) = \widetilde{X}(\boldsymbol{t})P(\boldsymbol{\ell}).
\end{equation}
Parameters $\boldsymbol{t} \in \mathbb{R}^r$ are informed over a reduced dimension submanifold, $1\leq r < d(n-d)$, according to a tangent space principal components analysis (tangent PCA)~\cite{fletcher2004principal}. Additionally, $\boldsymbol{\ell}\in \mathbb{R}^3$ informed by a separate tangent PCA, govern generalized \emph{scale variations} independent of rotations and reflections. The resulting representation~\eqref{eq:SST} parametrizes separate differences in anisotropic scale variations, $P(\boldsymbol{\ell})$, from those of undulations, $[\widetilde{X}](\boldsymbol{t})$, utilizing a corresponding representative matrix, $\widetilde{X}(\boldsymbol{t})$, for computations~\cite{edelman1998geometry, grey2023separable}. In our examples for this work, the parameter learning is informed by the aggregate ensemble of thousands of curves from both images being compared.

An important caveat of this finite representation is that shapes are not precluded from generating self-intersections---i.e., discrete shapes are generally identified with immersions. Moreover, the existing analysis is predicated on hand-picked discretizations of interpolating curves, called \emph{fixed reparametrizations}~\cite{grey2023separable}, and lacks a formal interpretation of the entries of $\widetilde{X}$---other than achieving desirable aerodynamic parametrizations and subsequent meshing. This begs the question of how to relate the finite representation of~\eqref{eq:SST} to an (infinite dimensional) analysis of continuous curves to offer improved discretizations and continuous extensions as immersions.

%-------------------------------------------------------------------------
\subsection{Preshapes as Boundary Curves}~\label{subsec:preshapes}
As before, preshapes are defined as open or closed curves $\boldsymbol{c}:\mathcal{I} \rightarrow \mathbb{R}^d$ assumed to be immersions and/or embeddings~\cite{micheli2013matrix,bauer2014overview,srivastava2010shape,mumford2006riemannian,younes2008metric,hartman2023elastic}---e.g., for $d=2$, $\boldsymbol{c}$ in our application is a planar embedding with $\mathcal{I}$ as the unit circle or equivalent compact, connected $1$-manifold. Applications are benefited by representations of curves over regularized spaces which control noisy measurement oscillations, e.g., those with increasing Lipschitz constant of curves $\mathcal{C}_d(\lambda_r) \subseteq \mathcal{C}_d(\lambda_{r+1})$ such that $0<\lambda_r \leq \lambda_{r+1}<\infty$ and
        \begin{equation}\label{eq:reg_Hilbert_space}
             \mathcal{C}_d(\lambda_r) \coloneqq \lbrace \boldsymbol{c} \,:\, \Vert \boldsymbol{c}(u) - \boldsymbol{c}(s)\Vert \leq \lambda_r\vert u - s\vert\rbrace,
        \end{equation}
for all $u,s\in \mathcal{I}$, to explicitly control nonlinear \emph{undulations}. Notice, for boundedness it sufficient to assume the individual components of curve $\boldsymbol{c} = (c_1, c_2, \dots, c_d)$ are Lipschitz, such that $\vert c_i(u) - c_i(s) \vert \leq \lambda_{i,r}\vert u - s\vert$ with $\lambda_r \leq \Vert (\lambda_{1,r}, \lambda_{2,r},\dots, \lambda_{d,r})^\top\Vert <\infty$.

We will generally refer to these curves as \emph{integrable curves} for some $d\mu$, such that $\int_{\mathcal{I}} \boldsymbol{c} \,d\mu < \infty$ and $\int_{\mathcal{I}} \Vert \dot{\boldsymbol{c}} \Vert\, d\mu < \infty$ where $\dot{\boldsymbol{c}}$ represents the $d$ component-wise derivative over $\mathcal{I} \subset \mathbb{R}$, with respect to some curve parameter. In other words, by Rademacher's theorem~\cite{rudin1987real}, we argue $\boldsymbol{c} \in \mathcal{C}_d(\lambda_r)$ with derivative $\dot{\boldsymbol{c}}$ almost everywhere admits integrable speed, $\Vert \dot{\boldsymbol{c}}\Vert$, over bounded $\mathcal{I}$.

\subsubsection{Arc-length reparametrization}
Conflating the assumed differentiability almost everywhere over a bounded domain $\mathcal{I}$ to weakly differentiable (absolutely) continuous curves is a convenience for reviewing arc-length reparametrizations as arc-length measures. The core developments in this work only require integrable curves over a bounded domain with the arc-length measure, which is sufficient to assume they are at least Lipschitz continuous per~\eqref{eq:reg_Hilbert_space} but does not necessitate that curves be more regular---similar to arguments in~\cite{srivastava2010shape,bruveris2016optimal}.

More precisely, the arc-length function,
\begin{equation}
    \xi^{-1}(t) = \int_{0}^t \Vert \dot{\boldsymbol{c}}(s) \Vert d\mu(s),
\end{equation}
can be expressed according to uniform measure $d\mu(s) = \rho(s)ds$, 
\begin{equation}
    \rho(s) = \begin{cases}
                    1, \,\, s \in \mathcal{I}\\
                    0, \,\, s \notin \mathcal{I},
              \end{cases}
\end{equation} 
where, arbitrarily assuming unit length curves, $\mathcal{I} \coloneqq [0,1]$. Via variable substitution, $u = \gamma(s)$ given any nonnegative $\gamma:\mathcal{I} \rightarrow\mathcal{I}$ such that $\gamma(0) = 0$ and $\dot{\gamma} > 0$ implies $ds = \dot{\gamma}^{-1}(s)du$. Thus, by assigning $t' \coloneqq \gamma(t)$,
\begin{align}
    (\xi^{-1}\circ \gamma)(t) &= \int_0^{\gamma(t)}\Vert \dot{(\boldsymbol{c}\circ \gamma)}(s)\Vert ds \\ \nonumber
    &= \int_0^{t'} \Vert \dot{\boldsymbol{c}}(u) \Vert \vert \dot{\gamma}(s) \vert \dot{\gamma}^{-1}(s)du \\ \nonumber
    &= \int_0^{t'}\Vert \dot{\boldsymbol{c}}(u) \Vert du\\ \nonumber
    &= \xi^{-1}(t')\nonumber.
\end{align}
Therefore, $\xi^{-1}$ is \textit{invariant to reparametrization}---the range of $\xi^{-1}$ is unmodified by the composition with $\gamma$---by arbitrarily renaming $t'$ as $t$. 

Notably, we may reinterpret changes in the scale of velocity, $\vert \dot{\gamma}(s) \vert$, (speed) as an integral measure, $d\mu(s) \coloneqq \vert \dot{\gamma}(s) \vert ds$. Thus, the scale of the speed we traverse the curve can instead be identified with an integral (pushforward) measure along the curve, i.e., for $\tilde{\boldsymbol{c}} = \boldsymbol{c}\circ \gamma$ we have
\begin{align}
    \tilde{\xi}^{-1}(t) &= \int_0^t \Vert \dot{\tilde{\boldsymbol{c}}}(s) \Vert ds \\ \nonumber
    & = \int_0^{t} \Vert \dot{\boldsymbol{c}}(s) \Vert\vert \dot{\gamma}(s) \vert ds \\ \nonumber
    & = \int_0^{t} \Vert \dot{\boldsymbol{c}}(s) \Vert d\mu(s). \nonumber
\end{align}
Here, $\gamma$ is any diffeomorphism constituting a \textit{reparametrization}. As a simple identification, we take $d\mu(s) \coloneqq \rho(s)ds/\Vert \dot{\boldsymbol{c}}(s)\Vert$ for $\rho(s)$ uniform so that $\xi^{-1}(t) = t$. Thus $d\mu(s)$ becomes a natural choice of \textit{fixed} integral measure which utilizes a \textit{uniformly weighted} (weak) arc-length reparametrization as our integral measure. Moreover, by composition with differentiation and the norm, the \textit{arc-length measure is invariant to rigid actions} on $\boldsymbol{c}$. Other applications may be interested in an alternative weighting (density) $\rho$, which integrates to one, to collect points near hand-picked landmarks of interest~\cite{Doronina2023, grey2023separable}.

Alternative treatments may also identify any diffeomorphism $\gamma$, up to scale, with the cumulative distribution function and corresponding nonnegative probability density, $\rho \propto \dot{\gamma}(s)$. Moreover, this identification could be extended to the Fourier dual (characteristic function) of $\gamma$. These identifications are reminiscent of the \textit{varifolds}~\cite{younes2008metric, charon2013varifold} and applications of Fisher-Rao metrics in functional data analysis~\cite{tucker2013generative, srivastava2016functional} as a promising flexible treatment of reparametrizations.

\subsubsection{Related ambient spaces}
Most applications study an ambient space of curves with Sobolev-type metrics that are crucial for controlling regularity and other pathological issues such as vanishing distances over the (manifold) topologies of preshapes and shapes~\cite{michor2004vanishing,bauer2014overview,micheli2013matrix}. In our case, $\mathcal{C}_d(\lambda_r)$ defined with $2$-norm, $\Vert \cdot \Vert$, and weak differentiability is a simple, sufficiently general choice that is consistent with the \textit{finite} matrix manifolds of interest in related work involving Separable Shape Tensors (SST)~\cite{grey2023separable}. 

We expound on restrictions of $\mathcal{C}_d(\lambda_r)$ to inform approximations and decompositions into flexibly regularized \textit{discrete} preshapes and scales to test statistical hypotheses. Extensibility to compute with many thousands of curves is bolstered by these \emph{finite} configuration spaces and distances do not vanish over these finite dimensional manifolds~\cite{bauer2014overview}.

We drop the convention of specifying a nesting over Lipschitz constant and simply refer to the space of Lipschitz curves with $d$-components as $\mathcal{C}_d$ with some unknown, finite bounding constant that depends on the curve. We demonstrate empirically that our configuration space tends to naturally regularize shapes over the nesting $\mathcal{C}_d(\lambda_r) \subseteq \mathcal{C}_d(\lambda_{r+1})$ where $r$ becomes a chosen dimensionality of a `learned' submanifold.

%Note, we are not attempting to study a Riemannian manifold with $L^2(\mathcal{I})$ as a choice of metric. Instead, we expound on a class of integral operators defined by preshapes so that we may test statistical hypotheses, extensible to ensembles containing tens of thousands of curves. The extensibility to many thousands of curves is bolstered by \emph{discrete} transformations of landmarks as data over a \textit{fixed} arc-length reparametrization and distances do not vanish over these finite dimensional manifolds~\cite{bauer2014overview}. 

\section{Curve FIE} \label{sec:FIE}
To formalize the reviewed landmark standardization of SST, we study a second kind homogeneous Fredholm integral equation (FIE). The SVD~\eqref{eq:landmark_stnd} of SST is equivalently written as the symmetric counterpart as the (thin) scaled eigendecomposition, 
\begin{equation}
    \frac{1}{n}(X - B(X))(X - B(X))^{\top} \,\overset{SVD}{=}\, V\Sigma^2_d V^{\top}.
\end{equation}
Note, the `thin SVD' and assumed linear independence of component functions implies $\Sigma_d$ is $d$-by-$d$, $\Sigma^2_d = \text{diag}(\sigma^2_1,\dots, \sigma^2_d)$ with $\sigma_1^2 \geq \sigma_2^2 \geq \dots \geq \sigma_d^2 > 0$. A column partition $V = (\mathbf{v}_1, \mathbf{v}_2, \dots, \mathbf{v}_d)$ defines $V$ as an element of the Stiefel manifold, $V^{\top}V = \mathbb{I}_d$, such that $V_{ij} = (v_i)_j$ is the $i$-th (row) entry of the $j$-th (column) sequence---i.e., $i=1,\dots,n$ and $j=1,\dots,d$. Equivalently, recalling in~\eqref{eq:data} that rows of the matrix are curve evaluations, the weighted eigendecomposition can be expressed as
\begin{equation} \label{eq:Riemann_sum}
    \sum_{i=1}^{n}w_i ( \boldsymbol{\tau}(s_p) \cdot \boldsymbol{\tau}(s_i) ) (v_i)_j = \sigma^2_j(v_p)_j.
\end{equation}    
In~\eqref{eq:Riemann_sum}, $w_i = 1/n$ are constant (over $i$) weights in the sum, and $\boldsymbol{\tau}:\mathcal{I} \rightarrow \mathbb{R}^d$ represents a fixed reparametrization of a `translated' parametric curve generating data in the rows of $X-B(X)$, i.e., $X-B(X) \approx (\boldsymbol{\tau}(s_1),\dots, \boldsymbol{\tau}(s_n))^{\top}$.

This eludes to an interpretation that specific landmark standardizations~\cite{bryner20142d, grey2023separable} act as simple Riemann sums for integrable curves. Generally, any constant-weight quadrature rule over a fixed reparametrization according to arbitrary integral measure $d\mu(s)$ can be described as a weighted version of of~\eqref{eq:landmark_stnd} with an appropriate sequence of nodes $(s_i)$ mapping to $(\boldsymbol{x}_i)$. Thus~\eqref{eq:Riemann_sum} can be interpreted as a quadrature (numerical integration) for any integrable curve.

With this continuous extension in mind, for any integrable curve $\boldsymbol{c} \in \mathcal{C}_d$ over integral measure $d\mu$, define 
\begin{equation}
    \mathcal{T}[\boldsymbol{c}](s) \coloneqq \boldsymbol{c}(s) - \int_{\mathcal{I}} \boldsymbol{c}(u)d\mu(u)
\end{equation} 
for all $s \in \mathcal{I}$, and notice $\mathcal{T}[\boldsymbol{c} +\boldsymbol{b}](\cdot) = \mathcal{T}[\boldsymbol{c}](\cdot)$ for arbitrary $\boldsymbol{b}\in \mathbb{R}^d$ which does not depend on the integral parameter. We refer to $\boldsymbol{\tau} \coloneqq \mathcal{T}[\boldsymbol{c}]$ as the \textit{translated and/or centered curve}. Next, define
\begin{equation}\label{eq:CLO}
    k_{\mathcal{T}}[\boldsymbol{c}](s,u) \coloneqq \mathcal{T}[\boldsymbol{c}](s) \cdot \mathcal{T}[\boldsymbol{c}](u)
\end{equation} 
and, consequently,
\begin{equation} \label{eq:inner_prod_ident}
    k_{\mathcal{T}}[\boldsymbol{c}](s,u) = \Vert \mathcal{T}[\boldsymbol{c}](s) \Vert \Vert \mathcal{T}[\boldsymbol{c}](u) \Vert \cos{(\theta)}
\end{equation}
as the \emph{composite linear operator} (CLO) where $\theta$ is the angle between pairs of translated landmarks $\boldsymbol{\tau}(s)$ and $\boldsymbol{\tau}(u)$. Finally, the continuous analog of the weighted eigendecomposition~\eqref{eq:Riemann_sum} motivates an approximation of the form,
\begin{equation} \label{eq:FIE}
    \int_{\mathcal{I}}k_{\mathcal{T}}[\boldsymbol{c}](s,u)v_j(u)d\mu(u) = \sigma_j^2v_j(s),
\end{equation}
alternatively denoted $\langle k_{\mathcal{T}}[\boldsymbol{c}](s_i,\cdot), v_j\rangle = \sigma_j^2v_j(s_i)$ in the context of the canonical $L^2(\mathcal{I},\mu)$ inner product.

\begin{lemma} \label{lemma:symm}
    $k_{\mathcal{T}}[\boldsymbol{c}]$ is symmetric.
\end{lemma}

\begin{proof}
    By definition of the scalar dot product, $k_{\mathcal{T}}[\boldsymbol{c}](s,u) = \mathcal{T}[\boldsymbol{c}](s) \cdot \mathcal{T}[\boldsymbol{c}](u) = \mathcal{T}[\boldsymbol{c}](u) \cdot \mathcal{T}[\boldsymbol{c}](s) = k_{\mathcal{T}}[\boldsymbol{c}](u,s)$.
\end{proof}

\begin{lemma}\label{lemma:HS_int_opt}
    If $v \in L^2(\mathcal{I},\mu)$ and $\boldsymbol{c} \in \mathcal{C}_d$ then $\mathcal{K}_{\mathcal{T}}[\boldsymbol{c}]$ such that
    \begin{equation}\label{eq:HS_operator}
        (\mathcal{K}_{\mathcal{T}}v)[\boldsymbol{c}](\cdot) = \int_{\mathcal{I}}\left(\mathcal{T}[\boldsymbol{c}] \cdot \mathcal{T}[\boldsymbol{c}]\right)(\cdot,u)v(u) d\mu(u),
    \end{equation}
    is a Hilbert-Schmidt (HS) integral operator.
\end{lemma}

\begin{proof}
    Assuming an integrable curve implies boundedness of $\mathcal{T}[\boldsymbol{c}]$ in~\eqref{eq:inner_prod_ident} and thus
    \begin{equation}
        \int_{\mathcal{I}\times \mathcal{I}}\vert k_{\mathcal{T}}[\boldsymbol{c}](s,u)\vert^2d\mu(s)d\mu(u) < \infty.
    \end{equation}
    The remainder is simply an application of Theorem 2.1 of~\cite{fasshauer2015kernel}. In short, Theorem 2.1 in~\cite{fasshauer2015kernel} states $\mathcal{K}_{\mathcal{T}}[\boldsymbol{c}]$ is a Hilbert-Schmidt operator and every Hilbert-Schmidt operator on $L^2(\mathcal{I}, \mu)$ is of the form~\eqref{eq:HS_operator} for some unique $k_{\mathcal{T}}[\boldsymbol{c}]:(s,u) \mapsto k_{\mathcal{T}}[\boldsymbol{c}](s,u)$ in $L^2(\mathcal{I} \times \mathcal{I},\mu\times \mu)$.
\end{proof}

Consequently, by Mercer's theorem~\cite{fasshauer2015kernel}, we have general $L^2(\mathcal{I},\mu)$ eigensolutions of the HS integral operator constituting solutions of~\eqref{eq:FIE}. Thus, with Lipschitz continuity sufficient for weak arc-length integrability, $v_j$ is the $j$-th eigenfunction satisfying the homogeneous FIE of the second kind~\cite{fasshauer2015kernel} for this specific choice of curve-dependent operator. 

Notice that composition in $\mathcal{T}$ offers translation invariance of the CLO. In fact, the CLO is invariant to a larger class of transformations on the curve:

\begin{lemma}\label{lemma:CLO_invariance}
    The CLO is invariant to permutations, rotations, and reflections, $R \in O(d)$, and translations, $\boldsymbol{b} \in \mathbb{R}^d$ which are assumed independent of the curve parameter.
\end{lemma}
\begin{proof}
    Given $\mathcal{T}[\boldsymbol{c} +\boldsymbol{b}] = \mathcal{T}[\boldsymbol{c}]$ for arbitrary $\boldsymbol{b} \in \mathbb{R}^d$ independent of the curve parameter. Additionally, for arbitrary $R \in O(d)$ independent of the curve parameter,
    \begin{align}
        k_{\mathcal{T}}[R\boldsymbol{c} + \boldsymbol{b}] &= \left(R\mathcal{T}[\boldsymbol{c} + \boldsymbol{b}]\right) \cdot \left( R \mathcal{T}[\boldsymbol{c} + \boldsymbol{b}]\right)\\ \nonumber
        &= \left(R\mathcal{T}[\boldsymbol{c}]\right) \cdot \left( R \mathcal{T}[\boldsymbol{c}]\right)\\ \nonumber
        &= \mathcal{T}[\boldsymbol{c}] \cdot \left( R^{\top}R \mathcal{T}[\boldsymbol{c}]\right)\\ \nonumber
        &= \mathcal{T}[\boldsymbol{c}]\cdot \mathcal{T}[\boldsymbol{c}]\\ \nonumber
        &= k_{\mathcal{T}}[\boldsymbol{c}] \nonumber
    \end{align}
\end{proof}
Thus, the CLO and the associated integral operation with rigid invariance of the arc-length measure, fall into the class of permutation, rotation, reflection, and translation invariant (PRRTI) operators also referred to as TRI kernels~\cite{micheli2013matrix}. Lemma~\ref{lemma:CLO_invariance} implies the angle between pairs of landmarks, $\theta$, is preserved under the subset of \emph{rigid transforms} according to~\eqref{eq:inner_prod_ident}---i.e., angles between landmarks are rigid invariant. Thus, we can interpret the FIE with the HS integral operator~\eqref{eq:HS_operator} as a separation of general linear (scale) deformations and rigid invariant angles between discrete pairs of landmarks (undulations). 

%-------------------------------------------------------------------------
\subsection{Eigensolutions}\label{sec:e_solns}
Given that our assumed space of preshapes $\mathcal{C}_d$ is not a Hilbert space, we propose a restriction $\boldsymbol{c} \in \mathcal{C}_d \cap \mathcal{H}_d$ where $\mathcal{H}_d$ is a unique $\mathbb{R}^d$-valued Hilbert space of curves defined by a \textit{bounded} reproducing kernel. We elaborate on a generalized dual formulation of~\eqref{eq:FIE} and the geometric nature of the sought separability described in~\cite{grey2023separable} to establish a framework for continuous extensions with flexible regularity implied by a chosen kernel. 

The preliminaries and developments aggregated in~\cite{fiedler2023lipschitz, alpay2021new} are a particularly helpful motivation. The restriction $\boldsymbol{c} \in \mathcal{C}_d \cap \mathcal{H}_d$ may trivially simplify, $\mathcal{C}_d \cap \mathcal{H}_d = \mathcal{H}_d$, based on the choice of $\mathcal{H}_d$ implied by an appropriate (matrix-valued) reproducing kernel~\cite{drake2022implicit}. In other words, redefining $\mathcal{C}_d \coloneqq \mathcal{H}_d$ with bounded kernel is sufficient to satisfy Lipschitz continuity~\cite{fiedler2023lipschitz, alpay2021new} while strengthening regularity. Ultimately, given ambiguity about the `true' regularity of curves generating boundary landmarks as data, proposing $ \mathcal{C}_d \cap \mathcal{H}_d$ eludes to nuances regarding a flexible choice of bounded kernel uniquely defining the restriction. In this case, $\mathcal{H}_d$ can be designed to consider the most appropriate restriction for a variety of applications.

Utilizing the RKHS definition of~\cite{micheli2013matrix}, the following result establishes a novel interpretation of the eigensolutions to~\eqref{equ:eigproblem} and has a profound implication on both the selection of landmarks---that may otherwise be hand-picked~\cite{kendall2009shape}---as well as the overall complexity of computing SST's:

\begin{definition}[Evaluation functionals of RKHS,~\cite{micheli2013matrix}] \label{def:eval_func}
    Assume $(\mathcal{H}_d, \langle \cdot, \cdot \rangle_{\mathcal{H}_d})$ is a Hilbert space of $\mathbb{R}^d$-valued functions defined over $\mathcal{I}$, $\mathcal{H}_d$ is a Reproducing Kernel Hilbert Space (RKHS) if evaluation functionals,
    $$
    f_s^{\alpha}:\mathcal{H}_d \rightarrow \mathbb{R}\,:\,\boldsymbol{\tau} \mapsto \boldsymbol{\alpha} \cdot \boldsymbol{\tau}(s),
    $$
    are linear and continuous over $\mathcal{H}_d$ and in $\boldsymbol{\alpha}\in \mathbb{R}^d$ for all $s \in \mathcal{I}$, i.e., $f^{\alpha}_{s}\in \mathcal{H}^*_d$ is a dual representation of the curve.
\end{definition}

\begin{theorem} \label{theorem:eigfuncs}
    Eigenfunctions of the CLO are PRRTI orthonormal evaluation functionals of $\mathcal{C}_d \cap \mathcal{H}_d$.
\end{theorem}

\begin{proof}
    Substituting the definition of the CLO into the generalized FIE,
    \begin{align}
        \sigma_j^2v_j(s) &= \langle k_{\mathcal{T}}[\boldsymbol{c}](s,\cdot), v_j\rangle_{\mathcal{H}^*_d}\\ \nonumber
        &= \langle \boldsymbol{\tau}(s)\cdot \boldsymbol{\tau}(\cdot),  v_j\rangle_{\mathcal{H}^*_d} \\ \nonumber
        &=\langle \sum_{j'=1,\dots,d}\tau_{j'}(s)\tau_{j'}(\cdot), v_j\rangle_{\mathcal{H}^*_d}\\ \nonumber           &=\sum_{j'=1,\dots,d}\tau_{j'}(s)\langle \tau_{j'}, v_j \rangle_{\mathcal{H}^*_d} \nonumber           
    \end{align}
    Thus, solving for the eigenfunctions as the (scaled) evaluation functionals implies
    \begin{equation} \label{equ:eigfunc}
        v_j(s) = \frac{1}{\sigma^2_j}\boldsymbol{\alpha}_j \cdot \boldsymbol{\tau}(s)
    \end{equation}
    where $\boldsymbol{\alpha}^{\top}_j \coloneqq (\langle \tau_1,v_j\rangle_{\mathcal{H}^*_d},\dots,\langle \tau_d,v_j\rangle_{\mathcal{H}^*_d})$. In other words, the $j$-th eigenfunctions are a linear combination of the component functions of the curve for non-trivial $\boldsymbol{\alpha}_j \neq 0$. 
    
    Substituting the solution \eqref{equ:eigfunc} for $v_j(s)$ back into the definition of $\boldsymbol{\alpha}^{\top}_j = (\alpha_{j1},\alpha_{j2},\dots,\alpha_{jd})$ reveals an algebraic dual of \emph{eigensolutions} as paired eigenvalues and eigenfunctions:
    \begin{align} \label{eq:the_Panda}
        \boldsymbol{\alpha}_j &= \frac{1}{\sigma_j^2}\left(\begin{matrix}
            \langle \tau_1,\boldsymbol{\alpha}_j \cdot \boldsymbol{\tau}\rangle_{\mathcal{H}^*_d}\\
            \vdots \\
            \langle \tau_d,\boldsymbol{\alpha}_j \cdot \boldsymbol{\tau}\rangle_{\mathcal{H}^*_d}
        \end{matrix}\right)\\
        &= \frac{1}{\sigma_j^2}\left(\begin{matrix}
            \langle \tau_1,\sum_{j'=1,\dots,d}\alpha_{jj'}\tau_{j'}\rangle_{\mathcal{H}^*_d}\\
            \vdots\\ \nonumber
            \langle \tau_d,\sum_{j'=1,\dots,d}\alpha_{jj'}\tau_{j'}\rangle_{\mathcal{H}^*_d}
            \end{matrix}\right)\\ \nonumber
        &= \frac{1}{\sigma_j^2}\left(\begin{matrix}
           \sum_{j'=1,\dots,d}\langle \tau_1,\tau_{j'}\rangle_{\mathcal{H}^*_d}\alpha_{jj'} \\
           \vdots\\
           \sum_{j'=1,\dots,d}\langle \tau_d,\tau_{j'}\rangle_{\mathcal{H}^*_d}\alpha_{jj'}\\
            \end{matrix}\right)\\ \nonumber
        &= \frac{1}{\sigma_j^2}\left(\begin{matrix}
            \langle \tau_1,\tau_1\rangle_{\mathcal{H}^*_d} & \dots & \langle \tau_1,\tau_d\rangle_{\mathcal{H}^*_d}\\
            \vdots & \ddots & \vdots \\
            \langle \tau_d,\tau_1\rangle_{\mathcal{H}^*_d} & \dots &\langle \tau_d,\tau_d\rangle_{\mathcal{H}^*_d}
        \end{matrix}\right)\boldsymbol{\alpha}_j\\ \nonumber
        &\coloneqq \frac{1}{\sigma_j^2}\langle \boldsymbol{\tau} \otimes \boldsymbol{\tau}\rangle_{\mathcal{H}^*_d}\boldsymbol{\alpha}_j. \nonumber
    \end{align}
    We have defined $\langle \boldsymbol{\tau} \otimes \boldsymbol{\tau}\rangle_{\mathcal{H}^*_d}$ as the \textit{component-wise integration} of the outer product of centered component functions. Rearranging this expression implies the eigenproblem,
    \begin{equation} \label{equ:eigproblem}
        \left( \langle \boldsymbol{\tau} \otimes \boldsymbol{\tau}\rangle_{\mathcal{H}^*_d} - \sigma_j^2 \mathbb{I}_d \right)\boldsymbol{\alpha}_j =\boldsymbol{0}, \quad j=1,\dots,d.
    \end{equation}
    Here, $\langle \boldsymbol{\tau} \otimes \boldsymbol{\tau}\rangle_{\mathcal{H}^*_d} \coloneqq \langle \mathcal{T}[\boldsymbol{c}] \otimes \mathcal{T}[\boldsymbol{c}]\rangle_{\mathcal{H}^*_d}\in S^d_{++}$ is symmetric positive definite (full rank) by symmetry of the inner product $\langle\cdot, \cdot\rangle_{\mathcal{H}^*_d}$ and the assumed linearly independent component functions, $\tau_j$ for $j=1,\dots,d$. 
    
    Next, note that $\langle\boldsymbol{\tau}\otimes\boldsymbol{\tau}\rangle_{\mathcal{H}^*_d}$ admits strictly positive real eigendecomposition $\langle \boldsymbol{\tau} \otimes \boldsymbol{\tau} \rangle_{\mathcal{H}^*_d} = A\Sigma^2_d A^{\top} = (A\Sigma_d)(\Sigma_d A^{\top})$ where $A \in O(d)$. That is, columns of $A = (\boldsymbol{\alpha}_1, \boldsymbol{\alpha}_2,\dots, \boldsymbol{\alpha}_d)$ satisfy $\boldsymbol{\alpha}_j \cdot \boldsymbol{\alpha}_{j'} =\delta_{jj'}$ where $\delta_{jj'}$ is Kronecker's delta function. Thus, $\Sigma_d A^{\top}\boldsymbol{\alpha}_j = \sigma_j\boldsymbol{e}_j$, with $\boldsymbol{e}_j$ the $j$-th column of the identity, implies
    \begin{align}
        \langle v_j, v_{j'} \rangle_{\mathcal{H}^*_d} &= \frac{1}{\sigma_j^2 \sigma_{j'}^2}\boldsymbol{\alpha}_j^{\top}\langle \boldsymbol{\tau} \otimes \boldsymbol{\tau}\rangle_{\mathcal{H}^*_d}\boldsymbol{\alpha}_{j'} \\ \nonumber
        &= \frac{1}{\sigma_j^2 \sigma_{j'}^2}(\Sigma_d A^{\top}\boldsymbol{\alpha}_j)\cdot(\Sigma_d A^{\top}\boldsymbol{\alpha}_{j'}) \\ \nonumber
        &= \frac{1}{\sigma_j^2 \sigma_{j'}^2}(\sigma_j\boldsymbol{e}_j)\cdot (\sigma_{j'}\boldsymbol{e}_{j'})\\ \nonumber
        &= \frac{1}{\sigma_j\sigma_{j'}} \delta_{jj'}. \nonumber
    \end{align}
    Normalizing the eigenfunctions to be orthonormal such that $\Vert \widetilde{v}_j\Vert_{\mathcal{H}_d} = 1$ results in 
    \begin{equation}\label{eq:ortho_eigfunc}
        \widetilde{v}_j(s) \coloneqq v_j(s)\biggl / \sqrt{\langle v_j, v_j \rangle_{\mathcal{H}^*_d}} = \frac{1}{\sigma_j}\boldsymbol{\alpha}_j\cdot\boldsymbol{\tau}(s).
    \end{equation}
\end{proof}
An important interpretation is established by virtue of Thm.~\ref{theorem:eigfuncs}: \emph{solutions of~\eqref{equ:eigproblem} define dual (orthonormal) evaluation functionals~\eqref{eq:ortho_eigfunc} of (centered) $\mathbb{R}^d$-valued curves}. Thus, following the presentation of~\cite{micheli2013matrix}, we can naturally construct a (unique) $\mathbb{R}^d$-valued RKHS as an infinite dimensional extension constituting a \textit{space of undulations}. This establishes a direct correspondence for exploring improved shape-metrics and alternative deformations, such as curl-free and divergence-free deformations, to further improve finite---yet computationally efficient---SST manifold learning detailed in~\cite{grey2023separable}. Thus, we may equivalently refer to undulations as scale invariant \emph{PRRTI-features} of an RKHS.

Given many $\boldsymbol{c}$, as approximations or interpolations of segmented points with distinct cardinality, we can rapidly compute eigensolutions of~\eqref{equ:eigproblem} to define the corresponding orthonormal eigenfunctions \eqref{eq:ortho_eigfunc} as continuous duals. Numerically, the interpretation of Thm.~\ref{theorem:eigfuncs} motivates discrete approximations benefited by higher order quadratures of $\langle \boldsymbol{\tau} \otimes \boldsymbol{\tau}\rangle_{\mathcal{H}^*_d}$ compared to that of a Riemann sum in the original presentation of SST~\cite{grey2023separable}---i.e., highly accurate approximations without closed form solutions. This offers an improved criteria for the selection of various preshape discretizations for configuration spaces. Specifically, we can maintain row-wise evaluations stored in $X$ as \textit{collocations at quadrature nodes which only depend on a choice of fixed $d\mu$ for all curves}---i.e., equivariant to reparametrization and independent of any given $\boldsymbol{c}$.

% CAN CUT %%%%%%%%%%%%%%%%%%%%%%%%%%%%%%%%%%%%%%%%%%%%%%%%%%%%%%%%%%%%%%%%%%%%%%%%%%%%%%%%%%%%%%%%%%%%
In other cases, it may be possible to motivate closed form solutions to \eqref{equ:eigproblem}---e.g., $\boldsymbol{\tau}$ given as a specific spline/polynomial. However, we assume a high-order numerical quadrature rule will be required to compute eigenfunctions out of convenience and flexibility when experimenting with a variety of approximations and interpolations of data. Therefore, we desire numerical methods to compute quadrature weights and nodes typically with respect to parametric forms of $d\mu$ utilizing $\rho$ as some non-negative scalar-valued function integrating to one. One possibility is employing a recursive implementation by Lanczos-Stieltjes~\cite{gautschi1982generating,glaws2019gauss, constantine2012lanczos} to determine nodes and weights for arbitrary $\rho:\mathcal{I} \rightarrow \mathbb{R}_{\geq 0}$ integrating to one. Lanczos-Stieltjes methods potentially enable the exploration of more sophisticated reparametrizations but, in this setting, we continue to assume $\rho$ is uniform.

\subsubsection{Smooth approximations}
Alternatively, provided additional regularity of the curves is reasonable or useful for an application, given the $d$-tuple of orthonormal components $\widetilde{v}_j$, we can approximate dual representations arbitrarily well utilizing interpolation with orthogonal polynomials~\cite{trefethen2010householder,townsend2015continuous, trefethen2019approximation},
\begin{equation} \label{eq:cont_SST}
    \mathcal{O}(s\,;\boldsymbol{\theta},\boldsymbol{\ell}) = \widetilde{\mathcal{V}}(s \,; \boldsymbol{\theta})P(\boldsymbol{\ell}).
\end{equation}
In this case, the infinite dimensional analog of~\eqref{eq:SST} is a dual representation over coefficients $\boldsymbol{\theta}$ defining the $\infty$-by-$d$ quasi-matrix $\widetilde{\mathcal{V}} = (\widetilde{v}_1,\dots,\widetilde{v}_d)$ with `columns' represented by expansions over orthogonal Legendre polynomials---assuming uniform speed reparametrization and uniform $\rho$. The open source software package \textit{chebfun} is useful for constructing such approximations~\cite{Driscoll2014}.

When more regular curve representations are hypothesized, convenient, and beneficial---e.g., the reconstruction of reduced dimensional shapes depicted in Figure~\ref{fig:low_dim_grain}---we reference Weierstrass approximation theorem to argue that polynomial representations are still dense in $\mathcal{C}_d$, a subset of absolutely continuous curves, over compact domains~\cite{trefethen2019approximation}. Therefore, with finite measurement precision in any imaging modality, we can get arbitrarily close (under measurement precision) with quasi-matrices~\cite{grey2023separable, trefethen2010householder, townsend2015continuous}.

In short, the interpretation of Thm.~\ref{theorem:eigfuncs} motivates discretizations by $n$ curve evaluations at quadrature nodes instead of an arbitrary sequence. These discretizations are finite representations of: i) dual evaluation functionals, i.e., scale invariant PRRTI-features, in a unique matrix-valued RKHS~\cite{micheli2013matrix}, and/or ii) arbitrarily good quasi-matrix approximations built from orthogonal polynomials~\cite{trefethen2010householder,townsend2015continuous}.

%%%%%%%%%%%%%%%%%%%%%%%%%%%%%%%%%%%%%%%%%%%%%%%%%%%%%%%%%%%%%%%%%%%%%%%%%%%%%%%%%%%%%%%%%%%%%%%%%%%%%%

%-------------------------------------------------------------------------
\subsection{Geometric Interpretation}\label{subsec:geo_interp}

Let's explore some simple examples computing~\eqref{equ:eigproblem} by hand and demonstrate the nature of these PRRTI features. Parametrizing a circle with $d=2$, $\boldsymbol{c}(s) \coloneqq (\cos (2\pi s), \sin (2\pi s))^{\top}$ for $\mathcal{I} \coloneqq [0,1]$. In this case, the standardization achieved by computing $\widetilde{v}_j$ should simply scale the original component functions to have norm one. We obtain eigenvalues $\sigma^2_1 = \sigma^2_2 = 1/2$ indicating that component functions covary equally and normalized eigenfunctions $\widetilde{v}_1(s) = \pm\sqrt{2}\cos(2\pi s)$ and $\widetilde{v}_2(s) = \pm \sqrt{2}\sin(2\pi s)$ according to $\boldsymbol{\alpha}_j = \pm\boldsymbol{e}_j$. Clearly, $\Vert\widetilde{v}_j\Vert_{\mathcal{H}_d} = 1$ for $j=1,2$ and $\langle v_i, v_j \rangle_{\mathcal{H}^*_d} = \delta_{ij}$. 

Next, consider a simple linear scaling of the circle into an ellipse such that 
\begin{equation}
    \boldsymbol{c}(s) \coloneqq \left[ \begin{matrix}
    a & 0\\
    0 & b
    \end{matrix}\right] \left(\begin{matrix}
        \cos (2\pi s)\\
        \sin (2\pi s)
    \end{matrix}\right).
\end{equation}
Transforming to orthonormal eigenfunctions gives the same result as the circle, $\widetilde{v}_1(s) = \pm\sqrt{2}\cos(2\pi s)$ and $\widetilde{v}_2(s) = \pm \sqrt{2}\sin(2\pi s)$. However, the component functions co-vary differently and the eigenvalues are subsequently scaled as $\sigma_1^2 = a^2/2$ and $\sigma_2^2 = b^2/2$. With identical eigenfunctions, we conclude that a circle \textit{equally undulates} as an ellipse and their distinction is due, entirely, to linear scale variations. 

Lastly, if we now arbitrarily rotate the ellipse,
\begin{equation}
    \boldsymbol{c}(s) \coloneqq \left[ \begin{matrix}
    \cos(\theta) & -\sin(\theta)\\
    \sin(\theta) & \cos(\theta)
    \end{matrix}\right]\left(\begin{matrix}
        a\cos (2\pi s)\\
        b\sin (2\pi s)
    \end{matrix}\right),
\end{equation}
we obtain phase shifted eigenfunctions, $\widetilde{v}_1(s) = \pm\sqrt{2}\cos(2\pi s + \theta)$ and $\widetilde{v}_2(s) = \pm \sqrt{2}\sin(2\pi s + \theta)$, but (up to shifted reparametrization) our conclusion persists---all three of these curves are equally undulating with the only distinction being linear scale variations according to $\sigma_1^2 = a^2/2$ and $\sigma_2^2 = b^2/2$. 

Notice that area, perimeter, and unsigned scalar curvature of the ellipse are distinct from the circle. Hence, \emph{functionals of hand-picked features alone may indicate distinctions between each of these equally undulating curves} but their distinction, up to phase shift, is due entirely to scale parameters $a$ and $b$---which may or may not be well-defined by a subset of hand-picked features. In our case, a Hilbert space of bounded shift invariant kernels (convolutions or mollifiers) will eliminate the possibility of distinguishing between phase shifted eigenfunctions.

These examples are intended to motivate that the eigenproblem \eqref{equ:eigproblem} is a principal components analysis (PCA) or eigendecompsition of the second central moments (variation) of the continuous curve,
$
\langle \mathcal{T}[\boldsymbol{c}] \otimes \mathcal{T}[\boldsymbol{c}]\rangle_{\mathcal{H}^*_d} = A\Sigma^2 A^{\top}.
$
However, the eigenfunctions of the corresponding integral kernel offer a definition of undulation as dual evaluation functionals of the curve which are `factored' by linear scale variations informed by simple PCA. 

Based on these conclusions, orthogonal (standardized) parametric forms satisfying $\Vert \cdot \Vert_{\mathcal{H}_d} = 1$ represented by dual evaluation functionals with (bounded) \emph{periodic shift invariant kernel} equivalently constitute \emph{expansions of undulation}. We may refer to such an expansion as a \emph{Mercer series of undulation} over the construction of a dual space of PRRTI-features to approximate eigenfunctions of shape with any desired regularity.

%%%%%%%%%%%%%%%%%%%%%%%%%%%%%%%%%%%%%%%%%%%%%%%%%%%%%%%%%%%%%%%%%%%%%%%%%%%%%%%%%%%%%%%%%%%%%%%%%%%%%%%%%%%%%%%%%%%%
%\textcolor{purple}{(ZG) TODO: (up for grabs) build these Mercer series with some $\mathbb{R}^d$-valued kernels and relate the decay in coefficients to that of the spectral content from a Fourier series to provide interpreations as they related to the nested spaces sought in~\eqref{eq:reg_Hilbert_space}. We only have an empirical explanation to this end (discussed in the text introducing Fig~\ref{fig:low_dim_grain}).}

\section{PRRTI Numerics}\label{sec:numerics}
Note that our results are predicated on access to a reliable and robust segmentation. Without this preprocessing, we cannot compute an ensemble of segmented curves but modern techniques may be a useful supplement. Moreover, additional methods are required to align and register landmark data, e.g.~\cite{dogan2015fast, al2013continuous, rangarajan1997softassign, srivastava2016functional}, which are vital in these contexts. At a minimum, our method assumes ample image processing to produce a set of ordered landmarks for interpolation or ensembles of integrable curves as input. 

In section~\ref{subsec:func_approx}, we describe the numerical methods for approximating PRRTI functionals inspired by a Nystr\"om's method for curves. We then briefly introduce \textit{cyclic Procrustes} for registering segmented landmarks to align data against a fixed archetype in section~\ref{subsec:cycl_Pro} to account for arbitrarily phase-shifted results over closed curves. Finally, in section~\ref{subsec:mfld_learn}, we utilize approximated functionals informed by the full ensemble of segmented curves from images to define a product submanifold learning detailed in~\cite{grey2023separable}.

\subsection{Nystr{\"o}m's Method for Curves}\label{subsec:Nystrom}
With explicit knowledge of the curve's construction, we may be able to compute eigenfunctions exactly. However, \emph{absent explicit definitions of curves} and for convenience in exploring a variety of sufficiently smooth (regular) interpolations and approximations, we briefly elaborate on a spectral method as a re-weighting of the SVD~\eqref{eq:landmark_stnd} for approximating eigenfunctions.

The development draws heavily from \cite{DeMarchi2013} and Section 12.1.5 of \cite{fasshauer2015kernel}. Essentially, we want the SVD of a symmetric positive definite matrix $K \in \mathbb{R}^{n \times n}$, e.g., $K_{ii'} = k_{\mathcal{T}}[\boldsymbol{c}](s_i,s_{i'})$ and Lemma~\ref{lemma:symm}, akin to the eigendecomposition~\eqref{eq:Riemann_sum} as the symmetric counterpart of~\eqref{eq:landmark_stnd}, 
\begin{equation} \label{eq:kernel_matrix}
    K \,\overset{SVD}{=}\, V \Sigma_d^2 V^{\top}.    
\end{equation}
In~\cite{DeMarchi2013,fasshauer2015kernel}, it is noted that the columns of $V\Sigma_d^{-1}$ form a more stable basis for computation than those of the matrix $K$. This is especially important when the matrix $K$ is fixed by some given data. The main contribution of \cite{DeMarchi2013} was to note that a \emph{weighted} SVD could be computed by viewing the above SVD as a discretization of the HS integral eigenvalue problem. We can reinterpret this work to our ends. That is, rather than a more computationally stable basis for the native space of the kernel, we seek a more accurate approximation of the eigenfunctions than those originally proposed for SST. The result is simultaneous approximation and collocation (at quadrature nodes) of eigenfunctions as PRRTI-features of curve.

As a matrix factorization of $K$, utilizing the weighted SVD,
\begin{equation}
    K \,\overset{SVD}{=}\, W^{-1/2}V\Sigma_d^2V^{\top}W^{-1/2}.
\end{equation}
Thus, $V\Sigma_d^2V^{\top}$ is the SVD of $W^{1/2}KW^{1/2}$ instead of $K$. Here, $W$ is the diagonal matrix of positive weights that arise from the discretization of  $(\mathcal{K}_{\mathcal{T}}v)[\boldsymbol{c}]$ using the quadrature rule, e.g.,
\begin{equation}\label{eq:quad}
    \int_{\mathcal{I}} \boldsymbol{\tau}(s) d\mu(s)  \approx \sum_{i = 1}^nw_i \boldsymbol{\tau}(s_i),
\end{equation}
such that weights and nodes are determined by the fixed, weighted arc-length integral measure, $d\mu(s)$. To work out the appropriate weights for the specific case of uniformly weighted arc-length measures with variable speed, we make a few key observations. Since the translated curve $\boldsymbol{\tau}(s)$ is closed, the components of its parametrization are periodic. Thus, when given a sufficiently regular curve and sampling a uniform partition of $\mathcal{I}$ over arc-length, the necessary derivatives can be approximated to spectral accuracy using Fourier differentiation \cite{trefethen2000}. 

Given periodic integrand, integration over the period, and sampling on a uniform partition, the trapezoid rule will exhibit spectral accuracy assuming sufficient regularity of the curves. Thus, taking $ W \coloneqq W_{\Delta}W_\mu$ where the positive diagonal matrices $W_{\Delta}$ and $W_\mu$ contain the weights of the trapezoid rule and the contribution of the (pushforward) measure $\mu$, respectively. In particular, $W_{\Delta} = \Delta u\mathbb{I}_n$ where $\Delta u$ represents the arc-length gauge of the resulting discretization of the curve. For example, if we utilize a uniform partition with $n$-subintervals of $\mathcal{I} \coloneqq [-\pi,\pi]$ to discretize then $\Delta u = 2\pi/n$. As a convention, we assume $n$ sub-intervals for numerical integration and omit the duplicate $n+1$ point which closes the polygonal representation---i.e., $\boldsymbol{x}_1 = \boldsymbol{x}_{n+1}$ is the closure condition for our planar curves. 

For example, computing weights of the pushforward measure can be accomplished using periodic spectral differentiation over a uniform grid,
\begin{equation} 
    W_{\mu} = \text{diag}\left( \left((D_nT_n) \odot (D_nT_n)\right)\mathbf 1_{d,1} \right)^{-1/2},
\end{equation}
where $\odot$ is the Hadamard product and 
$$
T_n \coloneqq \left( \begin{matrix}
    \boldsymbol{\tau}^{\top}(s_1)\\
    \vdots\\
    \boldsymbol{\tau}^{\top}(s_n)
\end{matrix}\right) =  \left( \begin{matrix}
    \tau_1(s_1) & \dots & \tau_d(s_1) \\
    \tau_1(s_2) & \dots & \tau_d(s_2) \\
   \vdots & \vdots &  \vdots\\
    \tau_1(s_n) & \dots & \tau_d(s_n)
\end{matrix}\right) \in \mathbb{R}^{n\times d}_*.
$$
Additionally, $D_n$ is the $n\times n$ Fourier differentiation matrix with entries, for all $i,i'=1,\dots,n$,
\begin{equation}
  (D_n)_{ii'} = \begin{cases}
        0, \quad i = i'\\
        \frac{1}{2}(-1)^{i+i'}\cot{\left(\frac{s_i - s_{i'}}{2}\right)}, \quad i \neq i',
    \end{cases}
\end{equation}
where $n$ is even, $s_{i} = i\Delta u$, and  $\mathbf 1_{d,1}$ is the column vector of all ones with length $d$~\cite{trefethen2000}. Making this choice results in the standard eigenvalue problem,
\begin{equation}\label{eq:not_naive_SVD}
    \sum_{i = 1}^n w_i k(\boldsymbol{\tau}(s_{i'}), \boldsymbol{\tau}(s_i)) v(s_i) =   \widetilde\sigma^2 v(s_{i'})
\end{equation}
with linear kernel $k(\boldsymbol{\tau}(s_{i'}), \boldsymbol{\tau}(s_i)) = \boldsymbol{\tau}(s_{i'})\cdot \boldsymbol{\tau}(s_i)$, generalizing the weighted extension posed in~\eqref{eq:Riemann_sum} and written in matrix-vector format as $KW\boldsymbol{v} = \widetilde\sigma^2 \boldsymbol{v}$ where $\widetilde\sigma^2 $ denotes the eigenvalues of the discrete problem. Moreover, this problem is not symmetric. Thus, we write
\begin{equation}\label{eq:weighted_evecs}
    W^{1/2}KW^{1/2} \widetilde{\boldsymbol{v}}  = \widetilde\sigma^2 \widetilde{\boldsymbol{v}}
\end{equation}
where the symmetric matrix $W^{1/2}KW^{1/2}$ has the eigenvector $\widetilde{\boldsymbol{v}} =W^{1/2}\boldsymbol{v}$, %\textcolor{purple}{(ZG) Need to double check dimensionality here... confused about $\widetilde{\boldsymbol{v}} =\boldsymbol{v}W^{1/2} $ shouldn't it be $\widetilde{\boldsymbol{v}} = W^{1/2}\boldsymbol{v}$?} 
a \textit{weighted} version of  $\boldsymbol{v}$. 

\subsection{Functional Approximation}\label{subsec:func_approx}
Extending the discussion of Nystr\"om method for curves, we develop numerical approximations of \eqref{equ:eigfunc}-\eqref{equ:eigproblem} by simplifying to a $d\times d$ eigen-problem as opposed to the larger $n\times n$ problem in~\eqref{eq:weighted_evecs}. 

From (12.15) in \cite{fasshauer2015kernel} and following the approach of section \ref{sec:e_solns} to obtain numerical approximations of \eqref{eq:the_Panda} and \eqref{equ:eigproblem}, we have the SVD basis from the Nystr\"om method written as
\begin{equation}\label{equ:discrete_eigfunc}
\boldsymbol{v}^{\top}(s) = \mathbf{k}^{\top}(s) W V\Sigma^{-2}_d
\end{equation}
with 
\begin{align*}
    \mathbf{k}^{\top}(s) & = (K( s,  s_1), \cdots, K( s, s_n)) \\
                      & =  \boldsymbol{\tau}^{\top}(s) \left(\begin{matrix} \tau_1(s_1) & \tau_1(s_2) & \cdots&  \tau_1(s_n) \\ 
                     \vdots & \vdots & \cdots & \vdots\\
                     \tau_d(s_1) & \tau_d(s_2) & \cdots &  \tau_d(s_n) \end{matrix}\right)\\
                     & = \boldsymbol{\tau}^{\top}(s)T_n^{\top}.
\end{align*}
The expression $\boldsymbol{v}^{\top}(s) \coloneqq (v_1(s),\dots,v_d(s))$ represents the $d$-tuple of evaluation functionals \eqref{equ:eigfunc}, as a row vector, approximated with a quadrature rule and collocated at $s$. Now, as in section \ref{sec:e_solns}, we will rewrite this expression without explicit dependence on the eigenvectors, $V \in \mathbb{R}_*^{n\times d}$, on the right-hand side. Importantly, without a closed form expression for the curve, $\boldsymbol{\tau}(s)$, \eqref{equ:discrete_eigfunc} only returns eigenfunctions evaluated (i.e., collocated) at the quadrature nodes.

Note that $K$ of \eqref{eq:kernel_matrix} satisfies rank$(K) =d$ by assumption and, hence, 
\begin{align} \label{equ:eigfunc_eval}
    \boldsymbol{v}^{\top}( s) &=\mathbf{k}^{\top}(s) W V\Sigma^{-2}_d\\ \nonumber
        &=\boldsymbol{\tau}^{\top}(s)T_n^{\top}W(\mathbf{v}_1/\sigma_1^2, \mathbf{v}_2/\sigma_2^2, \cdots , \mathbf{v}_d /\sigma_d^2) \nonumber
\end{align}
where $\mathbf{v}_j$, $j=1,\dots,d$, are the column vectors of $V$ collocated at quadrature nodes. Moreover, with a quadrature of $\boldsymbol \alpha_j$ as defined in the proof of Thm.~\ref{theorem:eigfuncs},
$
\boldsymbol{\alpha}_j \coloneqq T_n^{\top}W \mathbf{v}_j,
$
and substituting into~\eqref{equ:eigfunc_eval} we obtain 
\begin{equation}\label{eq:evalvs}
\boldsymbol{v}^{\top}( s) = \boldsymbol{\tau}^{\top}(s) (\boldsymbol \alpha_1/\sigma_1^2,  \boldsymbol \alpha_2/\sigma_2^2, \cdots , \boldsymbol \alpha_d /\sigma_d^2).
\end{equation}
Numerically,~\eqref{eq:evalvs} is \eqref{equ:eigfunc} in the context of the Nystr\"om method and evaluates eigenfunctions given the translated curve, $\boldsymbol{\tau}$, and approximations of $\boldsymbol{\alpha}_j$'s. 

Letting $A \coloneqq (\boldsymbol{\alpha}_1,\boldsymbol{\alpha}_2,\cdots,\boldsymbol{\alpha}_d)\in \mathbb{R}^{d \times d}$ and collocating~\eqref{eq:evalvs} at quadrature nodes results in
$$
V = \left( \begin{matrix}
    \boldsymbol{v}^{\top}(s_1)\\
    \vdots\\
    \boldsymbol{v}^{\top}(s_n)
\end{matrix} \right) = T_n A \Sigma_d^{-2},
$$
again, with $\Sigma_d^{-2} \coloneqq \text{diag}(1/\sigma_1^2, 1/\sigma_2^2 , \dots, 1/\sigma^2_d)$. Recalling $
\boldsymbol{\alpha}_j \coloneqq T_n^{\top}W \mathbf{v}_j
$
implies
$
T_n^{\top} W V = A
$
and substituting our collocation $V = T_n A \Sigma_d^{-2}$, in place of~\eqref{eq:the_Panda}, we have
$
T_n^{\top}W(T_nA \Sigma_d^{-2}) = A.
$
Finally, rearranging reveals the anticipated identification as a weighted eigendecompostion,
\begin{equation}\label{eq:discrete_evalprob}
T^{\top}_nWT_n \,\overset{SVD}{=}\, A\Sigma^2_dA^{\top},
\end{equation}
specifically such that $A \in O(d)$. In short, orthonormal $\boldsymbol \alpha_j$, $1\leq j \leq d$ as the columns of orthogonal $A$, are obtained simultaneously using an appropriate scheme to compute the \emph{weighted} decomposition \eqref{eq:discrete_evalprob}, instead of separately via \eqref{equ:eigproblem}, and \eqref{eq:evalvs} may be subsequently employed to evaluate the eigenfunctions at any point $s$, if needed. 

Moreover, to obtain $\widetilde{V}$ with orthonormal columns for numerical implementations~\cite{grey2023separable}, we can equivalently compute the thin weighted-SVD, or \textit{square-root quadrature decomposition}, (SRQD)
\begin{equation} \label{eq:SRQD}
    T_n^{\top}W^{1/2}\,\,\overset{SVD}{=}\, A\Sigma_d\widetilde{V}^{\top},
\end{equation}
for all curves in the ensemble utilizing rapid rank-$d$ decomposition schemes. 

The resulting right singular vectors $\widetilde{V}$ are the weighted eigenvectors, $\widetilde{V} = W^{1/2}V$, of~\eqref{eq:weighted_evecs} in the Nystr\"om method. Thus, \textit{fast rank-$d$ SRQDs can efficiently map thousands of preshapes to discrete PRRTI-features}. Equivalently, with $A$ orthogonal, this fixed-order numerical integration constitutes $\lbrace T_n \rbrace \mapsto \lbrace (\widetilde{X}, P)\rbrace$ as a \textit{weighted polar decomposition}, 
\begin{align} \label{eq:derived_polar_decomp}
    W^{1/2}T_n &= \widetilde{V} \Sigma_dA^{\top}\\ \nonumber
    & = \widetilde{V}(A^{\top}A)\Sigma_dA^{\top}\\ \nonumber
    &=\widetilde{X}P, \nonumber
\end{align}
where $\widetilde{X} \coloneqq \widetilde{V}A^{\top}$ are rotations and reflections of the eigenfunctions into a \textit{unique view} and $P \coloneqq A \Sigma_dA^{\top}$ are the corresponding decomposed (separated) generalized scale variations. 

Thus, the SRQD acts as a mapping into separated pieces of generalized (anisotropic) scale variations $P \in S^d_{++}$, and representative orthonormal (preshape) undulations $\widetilde{X}$ of the Stiefel manifold, i.e., $\widetilde{X}^{\top}\widetilde{X} = \mathbb{I}_d$. Consequently, Thm.~\ref{theorem:eigfuncs} motivates rapid and accurate weighted decompositions collocated at $n$ quadrature nodes, \textit{which depend exclusively on the chosen integral measure}, instead of increasing sequences $(s_i)_{i=1}^n$ from alternative reparametrizations or optimization schemes to align pairs of curves. We saturate accuracy in the approximations over all curves in the ensemble by by taking $n$ large, e.g., $n = 500$, while only increasing the computational expense linearly over increasing $n$ given fixed $d$.

To validate our interpretation, we offer an example in Figure~\ref{fig:nystrom_eg}. Given $\boldsymbol{\tau}$ as a Cassini oval, we utilize~\eqref{eq:discrete_evalprob} to inform approximations of~\eqref{eq:evalvs} which are reevaluated at $10,000$ consistent points over the curve parameter to study convergence rates. Figure~\eqref{fig:nystrom_eg} emphasizes the spectral nature of convergence. In practice, we utilize interpolations of data in composition with high-order arc-length reparametrization approximations to achieve unit speed and uniformly distributed landmarks---i.e., trivializing $W_\mu \propto \mathbb{I}_n$. 

\begin{figure}
    \centering
    \includegraphics[width=0.85\linewidth]{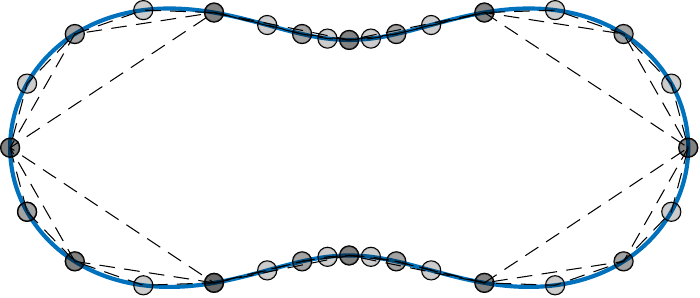}
    \\
    \vspace*{1em}
    \hspace*{-2.5em}
    \includegraphics[width=0.85\linewidth]{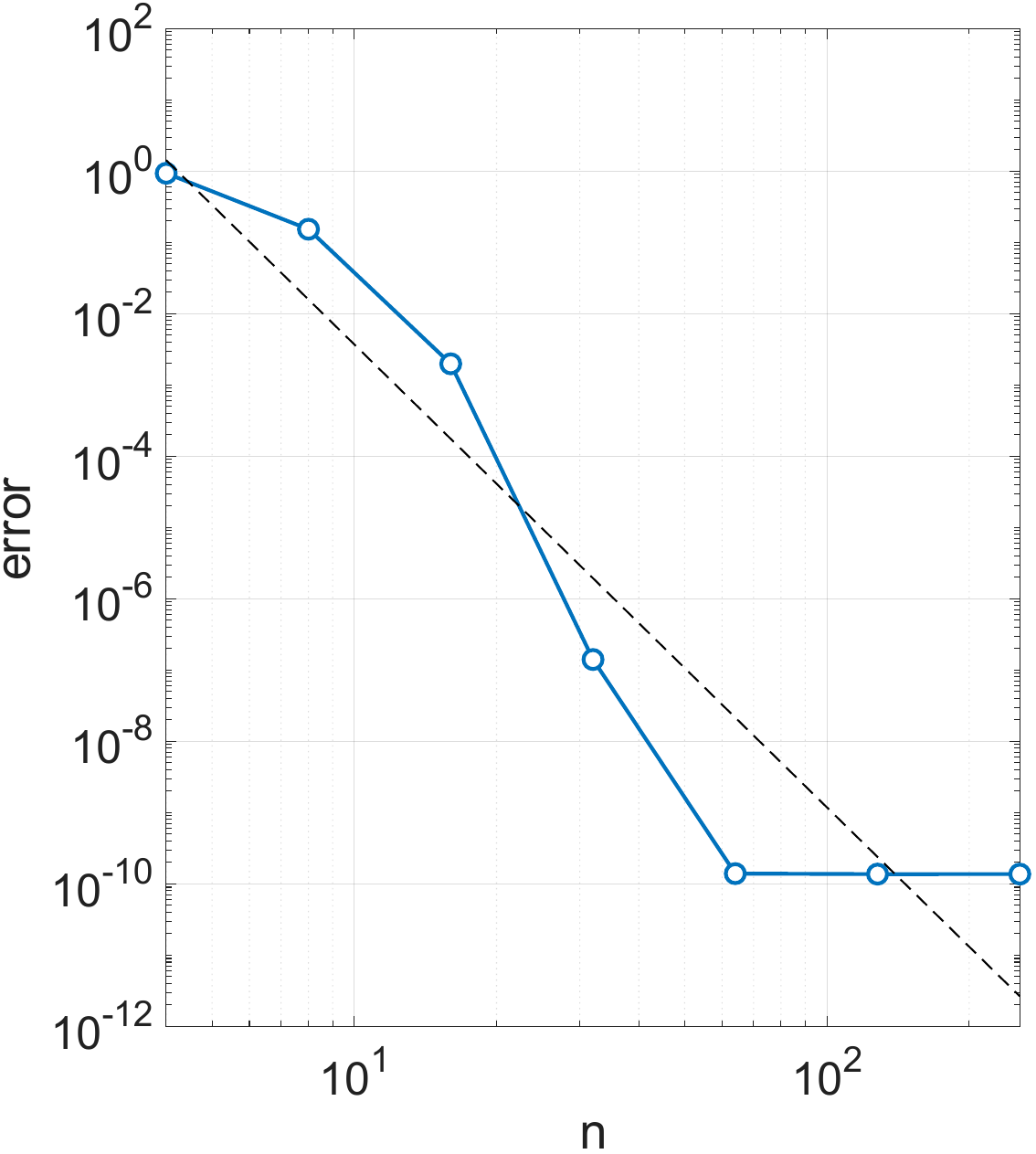}
    \caption{(top) An increasing number of landmarks collocated at quadrature nodes over a Cassini oval. Dashed lines are shown as a visual cue connecting nodes in a particular order corresponding to $8$, $16$, and $32$ landmarks each. Replicated nodes, $\lbrace \boldsymbol{\tau}(s_i)\rbrace_{n=8} \subset \lbrace \boldsymbol{\tau}(s_i)\rbrace_{n=16} \subset \lbrace \boldsymbol{\tau}(s_i)\rbrace_{n=32}$, over the increasing total number of nodes have darker shading. (bottom) Convergence plot over the same Cassini oval with increasing number of quadrature nodes, $n$. Error is the maximum $2$-norm difference in component functions over the curve parameter taken between approximated PRRTI-features and a reference solution with $8192$ nodes. A black dashed curve is shown with corresponding rate of $\sim 6.5$ to emphasize the spectral nature of convergence.}
    \label{fig:nystrom_eg}
\end{figure}

\subsection{Cyclic Procrustes} \label{subsec:cycl_Pro}
\begin{figure*}
    \centering
    \includegraphics[width=0.275\textwidth]{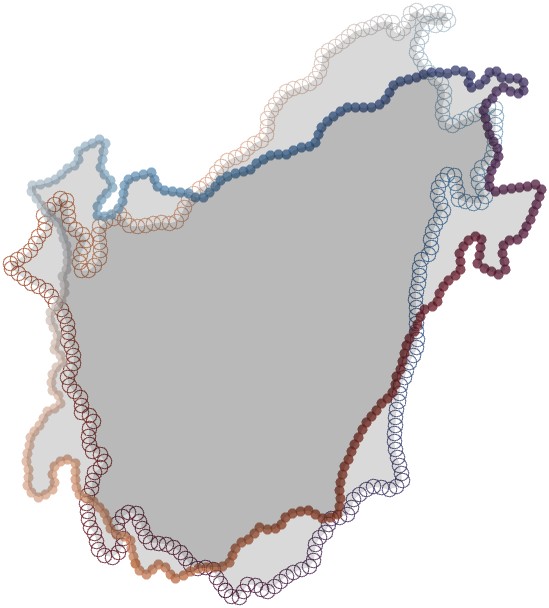}
    \hspace{1cm}
    \includegraphics[width=0.25\textwidth]{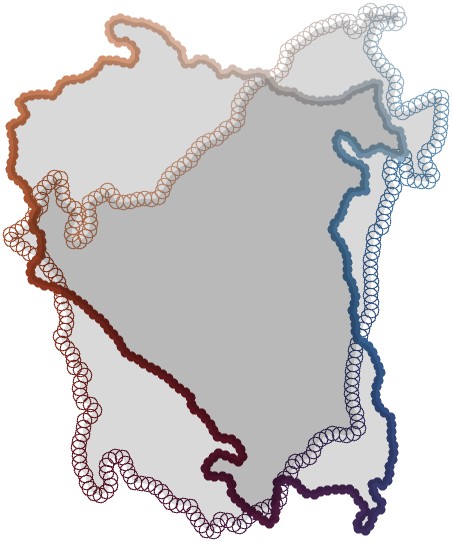}
    \hspace{1cm}
    \includegraphics[width=0.25\textwidth]{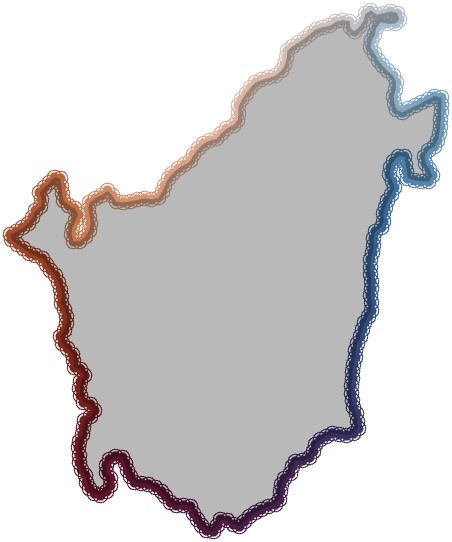}
    \caption{(left) An arbitrary cyclic permutation and rotation of a uniformly discrete grain shape with colors indicating row index, $n=500$. (center) An orthogonal Procrustes match. (right) A brute force cyclic Procrustes match.}
    \label{fig:cyc_Procrustes}
\end{figure*}

\begin{figure}
    \centering
    \includegraphics[width=1\linewidth]{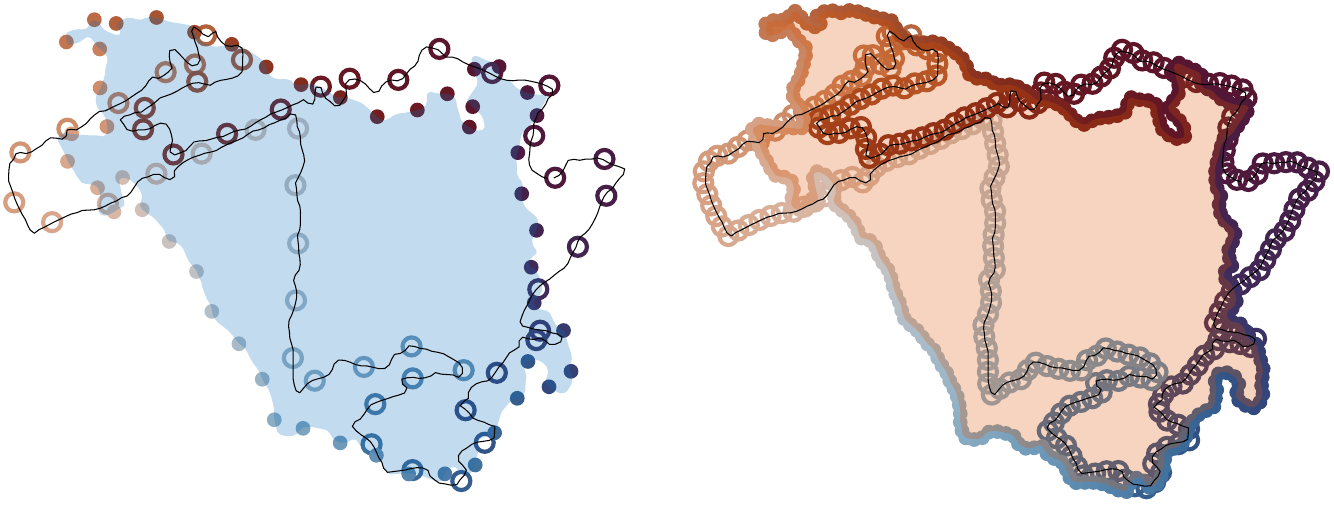}
    \includegraphics[width=1\linewidth]{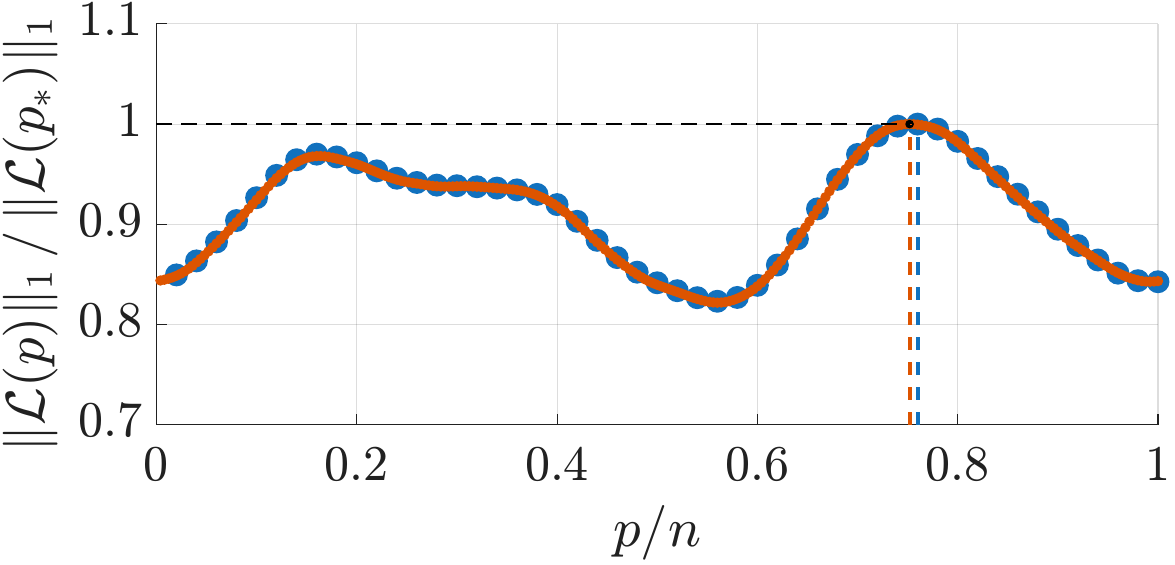}
    \caption{Cyclic Procrustes matching at low ($n=50$) and high ($n=250$) levels of uniform arc-length reparametrization. The fixed archetype is shown with a black curve. The fill color of the matched shape corresponds to the levels of refinement ($n=50$ blue and $n=250$ orange) in the discrete objective function evaluations. The landmark colors correspond to the registered indices of the shape against the archetype.}
    \label{fig:archetype_match}
\end{figure}

In other applications, like aerodynamic design~\cite{grey2023separable}, ensembles may be supplied with some fixed reparametrization establishing a convention for the initial starting landmark and orientation of curves. However, when an ensemble is the result of an image segmentation, the data lack a consistent starting landmark for registration/alignment. In other words, the first row of the discretization, $T_n$ or $X$, is entirely ambiguous as undulation is unique up to phase shift and often undefined for closed segmented curves from an image.

When we lack extrinsic alignment conventions to phase shift collocations, we propose solutions to the following (discrete) alignment problem for registering preshape data apriori:

\begin{theorem}[Separable Cyclic Procrustes]\label{thm:cycl_Procrustes}
    Given translated discrete shapes $X,Y \in \mathbb{R}^{n\times d}_*$ (full rank), $\mathcal{N} = \lbrace 1,\dots,n\rbrace \subset \mathbb{N}$, $p$-cycles as powers of the permutation matrix
    $$
    C(p) = \begin{bmatrix}
      \boldsymbol{0}& \mathbb{I}_{n-1} \\
     1 &\boldsymbol{0}^{\top}
    \end{bmatrix}^p
    $$ for all $p\in\mathcal{N}$, and $\mathcal{L}(p) \coloneqq Y^{\top}C(p)X$ full rank, $(p_*, R_*)$ is an optimal solution of
    $$
    \underset{p \in \mathcal{N},\, R \in O(d)}{\text{minimize}}\,\Vert C(p)X - YR\Vert_{F}
    $$
    if, and only if,
    $$
    p_* = \underset{p\in \mathcal{N}}{\text{argmax}} \,\Vert \mathcal{L}(p) \Vert_1
    $$ 
    where $\mathcal{L}({p_*}) \overset{SVD}{=} U_{*}\Omega_{*}V_{*}^{\top}$ such that $R_* = U_{*}V^{\top}_{*}$.
\end{theorem}
\begin{proof}
    Maximization of our derived parametric problem statement is equivalent to maximization of the form of Umeyama's Lagrangian---defined in the proof of the Lemma in~\cite{umeyama1991leastsquares} and denoted here as $\mathcal{U}$. Rewritten absent the sign convention for general $R\in O(d)$ and composition over $p$, Umeyama's Lagrangian is
    \begin{align}
        \mathcal{U}(p) &= \Vert C(p)X\Vert^2_F + \Vert YR\Vert^2_F - 2\Vert \mathcal{L}(p)\Vert_1 \\
        &= \Vert X\Vert^2_F + \Vert Y\Vert^2_F - 2\Vert \mathcal{L}(p)\Vert_1.\nonumber
    \end{align}
    Next, consider an equivalent problem stated with an extraneous constraint identifying parameter separability similar to~\cite{golub1973differentiation},
    $$
    \begin{matrix}   
    \underset{p \in \mathcal{N}}{\text{maximize}} & \,\Vert \mathcal{L}(p)\Vert_1\\
    \text{such that} & R = \pi_{\mathcal{U}}(\mathcal{L}(p)).
    \end{matrix}
    $$
    However, distinct from~\cite{golub1973differentiation}, the bijection, $$\pi_{\mathcal{U}}(\cdot) \coloneqq (\cdot)(V_p\Omega_p^{-1}V_p^{\top}),$$ is defined by the SVD, $\mathcal{L}(p) \,\overset{SVD}{=} U_p\Omega_pV_p^{\top}$, such that $\Omega_p = \text{diag}(\omega_1(p),\omega_2(p),\dots,\omega_d(p))$ with $\omega_1(p) \geq \omega_2(p)\geq \dots \geq \omega_d(p)>0$.

    Finally, we argue that the extraneous parametric constraint $\pi_{\mathcal{U}}(\mathcal{L}(p))$ for any $p$ is best. By definition, $\pi_{\mathcal{U}}(\mathcal{L}(p))$ is best if and only if $\langle \pi_{\mathcal{U}}(\mathcal{L}(p)),\mathcal{L}(p)\rangle_F > \langle R ,\mathcal{L}(p)\rangle_F$ for all $R \neq \pi_{\mathcal{U}}(\mathcal{L}(p)) \in O(d)$ and $p \in \mathcal{N}$. Equivalently, it is sufficient to show $\langle R - \pi_{\mathcal{U}}(\mathcal{L}(p)), \mathcal{L}(p) \rangle_F<0$ for all $R \neq \pi_{\mathcal{U}}(\mathcal{L}(p)) \in O(d)$. Utilizing the definition of $\pi_{\mathcal{U}}$,
    \begin{align*}
        \langle \pi_{\mathcal{U}}(\mathcal{L}(p)), \mathcal{L}(p)\rangle_F &= \langle \mathcal{L}(p)(V_p\Omega_p^{-1}V_p^{\top}), \mathcal{L}(p) \rangle_F\\
        & = \langle U_p\Omega_pV_p^{\top}(V_p\Omega_p^{-1}V_p^{\top}), \mathcal{L}(p) \rangle_F\\
        & = \langle U_pV_p^{\top}, U_p\Omega_pV_p^{\top} \rangle_F\\
        & = \text{tr}(V_p\Omega_pV_p^{\top})\\
        &= \Vert \mathcal{L}(p) \Vert_1
    \end{align*}
    and, thus,
    \begin{align*}
        \langle R - \pi_{\mathcal{U}}(\mathcal{L}(p)),\mathcal{L}(p)\rangle_F &= \langle R,\mathcal{L}(p)\rangle_F - \langle \pi_{\mathcal{U}}(\mathcal{L}(p)), \mathcal{L}(p)\rangle_F \\
        &= \langle R,\mathcal{L}(p)\rangle_F - \Vert \mathcal{L}(p) \Vert_1.
    \end{align*}
    By the original argument of orthogonal Procrustes solution~\cite{gibson1962least, schonemann1966generalized, lawrence2019purely}, 
    \begin{align*}
        \langle R,\mathcal{L}(p)\rangle_F &= \text{tr}(R^{\top}U_p\Omega_pV_p^{\top}) \\
        &= \text{tr}(V_p^{\top}R^{\top}U_p\Omega_p) \\
        &= \langle Q_p, \Omega_p\rangle_F
    \end{align*}
    for some $Q_p =  U_p^{\top}RV_p \in O(d)$ which satisfies $\langle Q_p, \Omega_p\rangle_F \leq \langle \mathbb{I}_d, \Omega_p\rangle_F = \Vert \mathcal{L}(p) \Vert_1$ with equality when $Q_p = \mathbb{I}_d$. However, $Q_p = \mathbb{I}_d$ if and only if $R = U_pV_p^{\top} = \pi_{\mathcal{U}}(\mathcal{L}(p))$, thus $\langle R,\mathcal{L}(p)\rangle_F - \Vert \mathcal{L}(p) \Vert_1 < 0$ for all $R \neq \pi_{\mathcal{U}}(\mathcal{L}(p)) \in O(d)$.
\end{proof}

With brute force solutions to Theorem~\ref{thm:cycl_Procrustes}, we match segmented curves with a modified type of Procrustean metric~\cite{srivastava2016functional} and, more importantly, align data against an fixed `archetype', $Y$ collocated at consistent quadrature nodes. To promote unique solutions to Theorem~\ref{thm:cycl_Procrustes} at a fixed $n$, we desire asymmetric archetypes or archetypes with known symmetries\footnote{Notice a `featureless' circle results in a constant objective function.} to search a subset of permutations. However, for many material micrographs, it is sufficient to take a random draw from the full ensemble of grains to identify a suitable asymmetric archetype for global registration. Additionally, we normalize both discrete preshapes $X$ and $Y$ to have unit length and dilate the archetype to match the uniform scale of the second input preshape.

These solutions offer a convention for aligning an ensemble of discrete curves and implicitly defining initial landmark(s) for an ordering over the rows of $X$. Note that a lack of symmetries in shapes dictate the uniqueness of these solutions. Related work~\cite{dogan2015fast} elaborate on a continuous treatment of the alignment problem; offering an interpretation with significant speed up utilizing fast Fourier transforms. An example registering a shape to within machine precision subject to arbitrary rotation and cyclic permutation of itself is illustrated in Fig.~\ref{fig:cyc_Procrustes}.

%-------------------------------------------------------------------------

\subsection{Ensemble Manifold Learning} \label{subsec:mfld_learn}
With a set of accurate, discrete, cyclic aligned, preshape PRRTI-features $\lbrace \widetilde{X} \rbrace$ and corresponding scales $\lbrace P\rbrace$ aggregated from a pair of images or database, we briefly review an approach to learning an underlying (latent) space of features from thousands or tens of thousands of curves.

Naturally, to retain the sought properties of Lemma~\ref{lemma:CLO_invariance} (PRRTI-features) and separation of generalized scale $P \in S^d_{++} \cong GL_+(d, \mathbb{R})/SO(d)$ in the representative \textit{preshape} Stiefel discretizations $\widetilde{X}$, we must project approximated eigenfunctions onto the equivalence classes $[\widetilde{X}] \in \mathbb{R}^{n \times d}_*/GL_+(d, \mathbb{R})$ constituting the \textit{shape} of undulations. 

As described in~\cite{absil2008optimization}, this quotient space is identified with the Grassmannian, $Gr(d,n) \cong \mathbb{R}^{n \times d}_*/GL_+(d, \mathbb{R})$. Note, $[\widetilde{V}] = [\widetilde{X}]$ but $\widetilde{X}$ serves as the \textit{representative} preshape paired with scale variations $P$ in the polar decomposition. With $\lbrace (\widetilde{X}, P)\rbrace$ as SRQD transformed data from a the aggregate ensemble informed by both segmented images, we leverage the \textit{tangent PCA}\footnote{Note a correction of the misnomer. This analysis is more precisely described as \textit{tangent PCA} as opposed to previous descriptions as \textit{Principal Geodesic Analysis} (PGA).} submanifold learning procedure~\cite{grey2023separable,fletcher2004principal}, separately, over underlying matrix-manifolds $\lbrace [\widetilde{X}] \rbrace \subset Gr(d,n)$ and $\lbrace P \rbrace \subset S^d_{++}$.  This procedure, with thousands of shapes, executes in seconds on a conventional laptop.

For example, tangent PCA proceeds with an ensemble of discrete matrix-valued PRRTI-features $\lbrace [\widetilde{X}] \rbrace$ first approximating the Fr\'echet mean over $Gr(n,d)$, denoted $[\widetilde{V}_0]$. With the Fr\'echet mean establishing a local origin for a coordinate frame, we apply PCA to the image of the inverse exponential at $[\widetilde{V}_0]$ to determine normal coordinates, $\boldsymbol{t}\in \mathbb{R}^r$, over an $r$-dimensional subspace, $1\leq r\leq d(n-d)$. The span of this subspace parametrizes undulations over the matrix-manifold as a local section $\mathcal{G}_r \subseteq Gr(d,n)$, i.e.,
\begin{equation}
    \mathcal{G}_r \coloneqq \lbrace [X] \in Gr(d,n) \,:\, [X] = \text{Exp}_{[\widetilde{V}_0]}(E_r\boldsymbol{t}) \rbrace,
\end{equation}
where $E_r$ is the `learned' matrix\footnote{Note $\text{Exp}_{[\widetilde{V}_0]}(\cdot)$ is composed with an appropriate reshaping of the matrix-vector multiply~\cite{grey2023separable}.} with columns constituting an ordered orthonormal basis of $\mathbb{R}^r \cong T_{[\widetilde{V}_0]}\mathcal{G}_r$. When selecting representative preshapes, \textit{classic orthogonal Procrustes} is very useful for searching equivalence classes $[\widetilde{X}]$ to align a preshape $\widetilde{X}$ as desired~\cite{grey2023separable}.

This same learning procedure is repeated, separately, with $\lbrace P \rbrace$ to determine the intrinsic mean $P_0 \in \mathcal{P} \subseteq S^d_{++}$ and tangent basis spanning $T_{P_0}\mathcal{P} \subseteq T_{P_0}S^d_{++}$ facilitated by the algorithms in~\cite{fletcher2003statistics}. In our materials application, $\mathcal{P}$ is taken to be all of $S^d_{++}$ but other applications may benefit from coordiantes along subspaces of $T_{P_0}S^d_{++}$. Finally, consistent with~\cite{grey2023separable}, we utilize $\mathcal{G}_r \times \mathcal{P}$ as a \emph{product submanifold} of separable (pre-)shape tensors parametrized with smooth right inverse $T_n(\boldsymbol{t},\boldsymbol{\ell}) = \widetilde{X}(\boldsymbol{t})P(\boldsymbol{\ell})$ over normal coordinates $(\boldsymbol{t},\boldsymbol{\ell}) \in T_{[\widetilde{V}_0]}\mathcal{G}_r \times T_{P_0}\mathcal{P}$.

The result of this learning is a mapping into local normal coordinates,
$
\lbrace ([\widetilde{X}_q], P_q) \rbrace \mapsto\lbrace (\boldsymbol{t}_q, \boldsymbol{\ell}_q) \rbrace \subset \mathbb{R}^{r + 3},
$
over an $r$-submanifold, $\mathcal{G}_r \subset Gr(2,n)$. In composition, given an ensemble of curves $\lbrace \boldsymbol{c}_q\rbrace$ for $q=1,2,\dots,N$ and fixed quadrature nodes $\lbrace s_i \rbrace $ for $i = 1,2,\dots,n$, the developed interpretation informs a procedure,
\begin{align} \label{eq:mfld_projection}
    \lbrace \mathcal{T}[\boldsymbol{c}_q](\lbrace s_i \rbrace)\rbrace &\overset{SRQD\,\,\,\,}{\mapsto} \lbrace \widetilde{X}_qP_q \rbrace\\ 
    &\quad \mapsto \lbrace ([\widetilde{X}_q], P_q) \rbrace \\
    &\quad \mapsto \lbrace (\boldsymbol{t}_q, \boldsymbol{\ell}_q) \rbrace \subset \mathbb{R}^{r + 3}. \label{eq:nml_coords}
\end{align}
All algorithmic details mapping $N$ SRQD's into normal coordinates, progressing from~\eqref{eq:mfld_projection} to~\eqref{eq:nml_coords}, are described in~\cite{grey2023separable}.

Changing dimensionality $r$ by accumulating terms in the ordered basis expansion over normal coordinates, $\boldsymbol{t}\in \mathbb{R}^r \cong T_{[\widetilde{V}_0]}\mathcal{G}_r$, can empirically achieve the sought regularization of~\eqref{eq:reg_Hilbert_space} in the representation of the shapes. This observed mechanism can promote a biasing of undulating shape features away from potentially noisy variations in the segmented boundaries. An example of the empirical effects of this regularization are shown in Figure~\ref{fig:low_dim_grain}. We anticipate that this choice of undulation dimensionality can be used to make subsequent inferences more robust against noisy measurements and image processing.% We explore the implications of selecting various dimensionality for explainable binary classification in the sections to follow.

\begin{figure}
    \centering
    \includegraphics[width=0.5\textwidth]{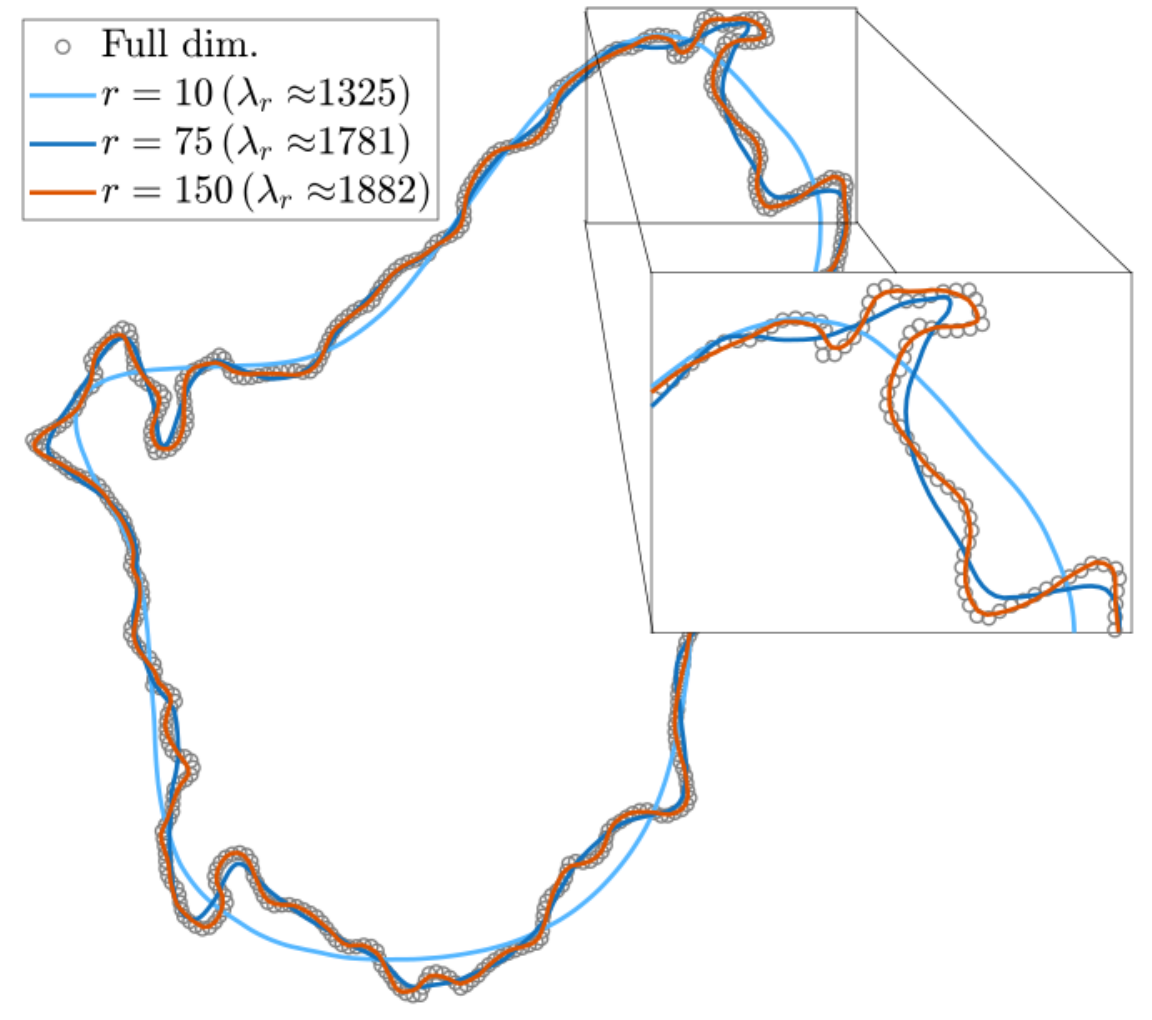}
    \caption{The empirical effect of the sought shape regularization~\eqref{eq:reg_Hilbert_space} over changing submanifold dimensionality, $r$, with $n=500$ landmarks from a reparametrized spline. Lipschitz constants of the curves are approximated using \emph{chebfun}~\cite{Driscoll2014}, an open source toolbox being utilized to compute quasi-matrix interpolation~\eqref{eq:cont_SST} of the reduced dimension analogs.}
    \label{fig:low_dim_grain}
\end{figure}

\subsection{Maximum Mean Discrepancy} \label{subsec:MMD}
Having formalized the approximation of curve dual evaluation functionals with (shape) discretizations parametrized over a learned product submanifold, we explore an explainable binary classification of the learned undulation and scale coordinates $(\boldsymbol{t},\boldsymbol{\ell}) \in T_{[\widetilde{V}_0]}\mathcal{G}_r \times T_{P_0}\mathcal{P}$ defined at respective intrinsic means $[\widetilde{V}_0] \in \mathcal{G}_r\subseteq Gr(n,d)$ and $P_0 \in \mathcal{P}\subseteq S^d_{++}$. Related work~\cite{zhang2022nonparametric} explore similar concepts, albeit, modulo curve dilation over elastic metrics and Kendall shape spaces while utilizing a more general measure theory coined DISCO analysis~\cite{RizzoDISCO2010}. 

We take the Borel $\sigma$-algebra over $T_{[\widetilde{V}_0]}\mathcal{G}_r \times T_{P_0}\mathcal{P}$ in a normal coordinate neighborhood---i.e., Hopf-Rinow theorem implies geodesic completeness\footnote{We utilize $\mathcal{G}_r$ with the usual tangential metric~\cite{edelman1998geometry,bendokat2020grassmann} and $S^d_{++}$ with the affine-invariant metric~\cite{fletcher2003statistics,pennec2020manifold} for geodesic completeness.} is equivalent to a metric space which is sufficient to generate the $\sigma$-algebra~\cite{pennec1999probabilities}. Thus, consider $(\boldsymbol{t}, \boldsymbol{\ell}) \sim \boldsymbol{\rho}$ as some joint distribution given by the smallest $\sigma$-finite (product) measure of random shape (normal) coordinates generated by the metric (feature) space.

\begin{table}
    \centering
    \begin{tabular}{c|c|c||l}
        {Case} &
        $(\rho_{\boldsymbol{t}}=\widehat{\rho}_{\boldsymbol{t}})$ & 
        $(\rho_{\boldsymbol{\ell}}=\widehat{\rho}_{\boldsymbol{\ell}})$ & $\text{pMMD}[\mathcal{H} \oplus \mathcal{Q},\cdot,\cdot]$ \\
        \hline
        $(1 \wedge 1)$ &
        Accept & Accept & $\implies$ Accept\\
        $(1 \wedge 0)$ &
        Accept & Reject & $\implies$ Reject \\
        $(0 \wedge 1)$ &
        Reject & Accept & $\implies$ Reject \\
        $(0 \wedge 0)$ &
        Reject & Reject & $\implies$ Reject \\
    \end{tabular}
    \caption{A logical conjunction (AND gate) for the two separable hypothesis tests (left of the double vertical lines) using pMMD. In keeping with the definition of `truth' in a hypothesis test, `failure to reject' the null hypothesis is considered a logical `true' condition and is arbitrarily named `Accept.'}
    \label{tab:logical_conjunction}
\end{table}

The goal of maximum mean discrepancy (MMD)~\cite{fortet1953convergence,gretton2012kernel} is to determine if finite observations of independent and identically distributed (i.i.d.) random variables defined on a topological space, with respective probability measures, coincide or not. That is, given observations $\lbrace (\boldsymbol{t}_1, \boldsymbol{\ell}_1),\dots,(\boldsymbol{t}_N, \boldsymbol{\ell}_N)\rbrace$ and $\lbrace (\widehat{\boldsymbol{t}}_1, \widehat{\boldsymbol{\ell}}_1),\dots,(\widehat{\boldsymbol{t}}_{\widehat{N}}, \widehat{\boldsymbol{\ell}}_{\widehat{N}}) \rbrace$ assumed i.i.d. from $\boldsymbol{\rho}$ and $\widehat{\boldsymbol{\rho}}$ respectively, can we test whether $\boldsymbol{\rho} \neq \widehat{\boldsymbol{\rho}}$? Moreover, with separable parameters and in a complementary fashion, can we test whether $(\rho_{\boldsymbol{t}} = \widehat{\rho}_{\boldsymbol{t}})$ \emph{and} $(\rho_{\boldsymbol{\ell}} = \widehat{\rho}_{\boldsymbol{\ell}})$ given shape features informed by an ensemble extracted from a pair of images? 

This formal problem statement is the quantitative analog of the proposed explainable binary classification: \textit{test if the ensemble of random shape coordinates partitioned between two images are discrepant, ignoring rigid motions, and whether those discrepancies result from distinct distributions over undulation, $\lbrace \boldsymbol{t} \rbrace$ from the first image versus $\lbrace \widehat{\boldsymbol{t}} \rbrace$ from the second, \emph{or} generalized scale, $\lbrace \boldsymbol{\ell} \rbrace$ from the first image versus $\lbrace \widehat{\boldsymbol{\ell}} \rbrace$} from the second.

This complementary separable formalism is simply De Morgan's theorem, $(\rho_{\boldsymbol{t}} = \widehat{\rho}_{\boldsymbol{t}})\wedge (\rho_{\boldsymbol{\ell}} = \widehat{\rho}_{\boldsymbol{\ell}}) \leftrightarrow \neg ((\rho_{\boldsymbol{t}} \neq \widehat{\rho}_{\boldsymbol{t}}) \vee (\rho_{\boldsymbol{\ell}} \neq \widehat{\rho}_{\boldsymbol{\ell}}))$. Truth Table~\ref{tab:logical_conjunction} details this formalism where \textit{logical truth} corresponds to \textit{failing to reject the null hypothesis}, named `Accept' as $(\rho_{(\cdot)} = \widehat{\rho}_{(\cdot)})$, when comparing against derived thresholds as discussed in~\cite{gretton2012kernel}.

Here, as in~\cite{gretton2012kernel}, we utilize MMD defined over a unit ball in an appropriate RKHS, $\mathcal{F}$, 
\begin{equation} \label{eq:MMD}
    \text{MMD}[\mathcal{F},\boldsymbol{\rho},\widehat{\boldsymbol{\rho}}\,] \coloneqq \sup_{\psi \in \mathcal{F}} \vert \mathbb{E}_{\boldsymbol{\rho}}[\psi] - \mathbb{E}_{ \widehat{\boldsymbol{\rho}}}[\psi] \vert.
\end{equation}
We note that $\mathcal{F}$ is, presently, distinct from the $\mathbb{R}^d$-valued RKHSs motivated by Thm.~\ref{theorem:eigfuncs} but the measure theory is extensible to spaces of curves and alternative metric spaces~\cite{zhang2022nonparametric, RizzoDISCO2010}. For simplicity and enabling the empirical regularization depicted in Figure~\ref{fig:low_dim_grain}, we take $\mathcal{F}$ as an RKHS representing features of `learned' parameter distributions over reduced dimensional normal coordinates informed by the tangent PCA of undulations and generalized scales. Future work is aimed at reconciling the two RKHS's---one as a functional space of curve equivalences and the other representing a distribution of learned features.

In~\eqref{eq:MMD}, $\mathbb{E}_{(\cdot)}$ represents expectation (integration) with respect to the identified probability measures, $\boldsymbol{\rho}$ versus $\widehat{\boldsymbol{\rho}}$, which are informed by samples from distinct measurements or observations---i.e., distributions of shape features from distinct images in our case. Solutions to~\eqref{eq:MMD} are motivated empirically utilizing a choice of symmetric positive definite kernel, $f \in \mathcal{F}$, and Lemma 6 in~\cite{gretton2012kernel} such that we can express the squared population MMD over $\mathcal{F}$ as
\begin{equation}
    \text{MMD}^2[\mathcal{F},\boldsymbol{\rho},\widehat{\boldsymbol{\rho}}\,] = \boldsymbol{E}_{\boldsymbol{\rho}}[f] + \boldsymbol{E}_{\widehat{\boldsymbol{\rho}}}[f] - 2\boldsymbol{E}_{\boldsymbol{\rho}\widehat{\boldsymbol{\rho}}}[f].
\end{equation}
This rewrite utilizes a simplification via the `tower property' such that $\boldsymbol{E}_{\boldsymbol{\rho}}[f] \coloneqq \mathbb{E}_{\boldsymbol{\rho}}[\,\mathbb{E}_{\boldsymbol{\rho}}[f\vert \boldsymbol{\theta}]\,]$ is a double expectation over $\boldsymbol{\rho}$, similarly for $\boldsymbol{E}_{\widehat{\boldsymbol{\rho}}}$, and $\boldsymbol{E}_{\boldsymbol{\rho}\widehat{\boldsymbol{\rho}}} \coloneqq \mathbb{E}_{\boldsymbol{\rho}}[\,\mathbb{E}_{\widehat{\boldsymbol{\rho}}}[f\vert \boldsymbol{\theta}]\,]$ is a double expectation over both.

Out of convenience, we simply extend this framework with the \textit{product metric} for any $p$-norm of pairs of separate MMD metrics,
\begin{multline} \label{eq:pMMD}
    \text{pMMD}[\mathcal{H} \oplus \mathcal{Q},(\rho_{\boldsymbol{t}},\rho_{\boldsymbol{\ell}}),(\widehat{\rho}_{\boldsymbol{t}},\widehat{\rho}_{\boldsymbol{\ell}})] \coloneqq \\
    \biggl\Vert \left(\begin{matrix}
    \text{MMD}[\mathcal{H},\rho_{\boldsymbol{t}},\widehat{\rho}_{\boldsymbol{t}}]\\
    \text{MMD}[\mathcal{Q},\rho_{\boldsymbol{\ell}},\widehat{\rho}_{\boldsymbol{\ell}}]
    \end{matrix}\right)\biggl\Vert_p.
\end{multline}
This extension trivially satisfies the logical implications depicted in the Truth Table~\ref{tab:logical_conjunction} by composition with the corresponding norm of thresholds. Moreover, the use of the product metric is a somewhat natural choice by virtue of the product submanifold construction for the space of separable shape tensors. However, definition of $\text{pMMD}$ is merely a formalism to demonstrate the existence of a logical conjunction over the separate feature tests comparing the images. In practice, we simply compute the hypothesis tests separately and draw the aggregated conclusion as a consequence of this formal construction.

In summary, given ensembles of segmented curve features as data from a pair of images, separate hypothesis tests over one image $\lbrace \boldsymbol{t} \rbrace \sim \rho_{\boldsymbol{t}}$ versus the other $\lbrace \widehat{\boldsymbol{t}}\rbrace \sim \widehat{\rho}_{\boldsymbol{t}}$, and likewise $\lbrace \boldsymbol{\ell} \rbrace \sim \rho_{\boldsymbol{\ell}}$ versus $\lbrace \widehat{\boldsymbol{\ell}}\rbrace \sim \widehat{\rho}_{\boldsymbol{\ell}}$, offer logical tests to \emph{explain} differences aggregated with the product maximum mean discrepancy (pMMD). Numerical examples are depicted and described in Figures~\ref{fig:compare_grains_111}-\ref{fig:compare_grains_000} emphasizing each row of the logical conjunction shown in Truth Table~\ref{tab:logical_conjunction}.

%-------------------------------------------------------------------------
\begin{figure}
    \centering
    \includegraphics[width=1\linewidth]{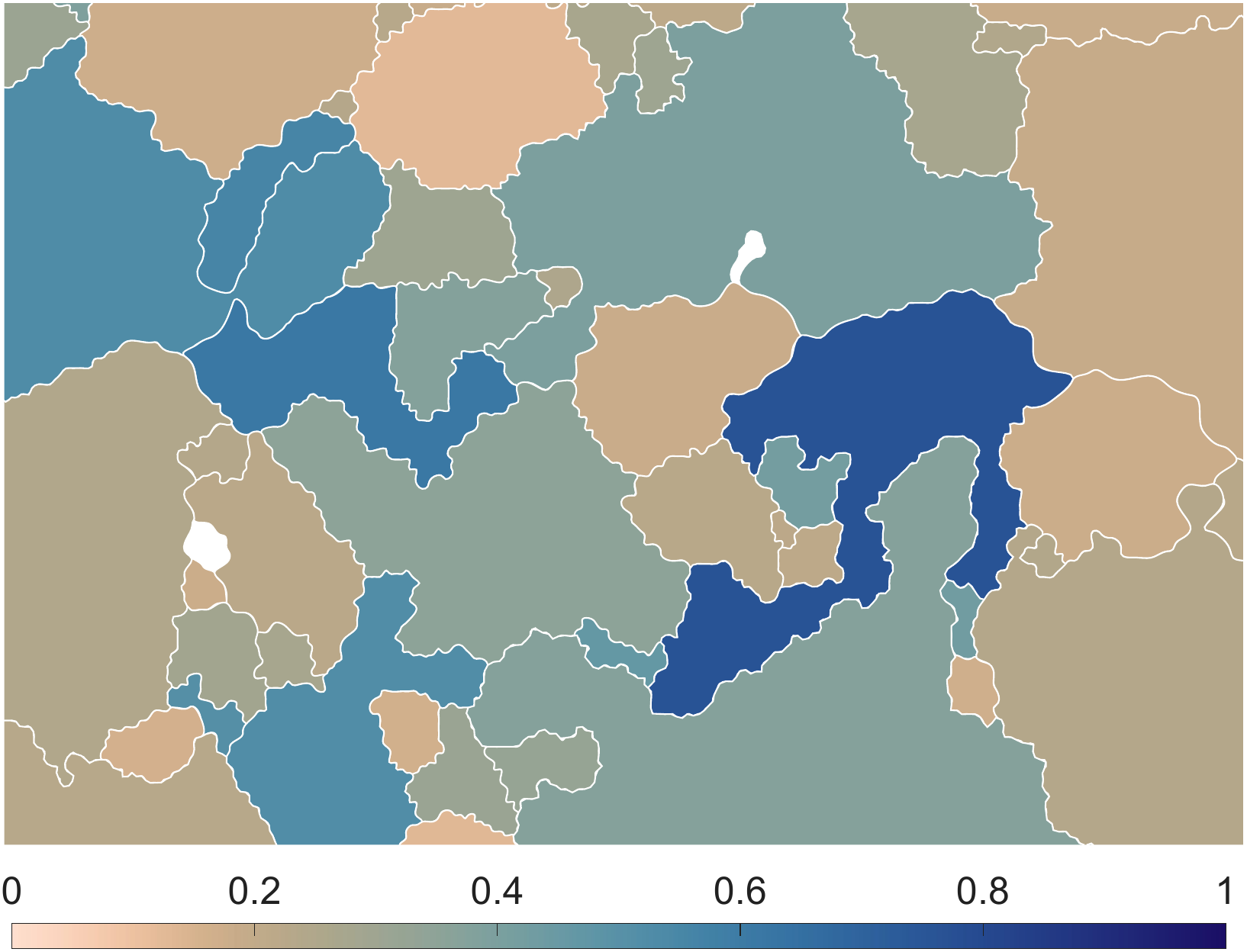}
    \\
    \vspace*{0.75em}
    \includegraphics[width=1\linewidth]{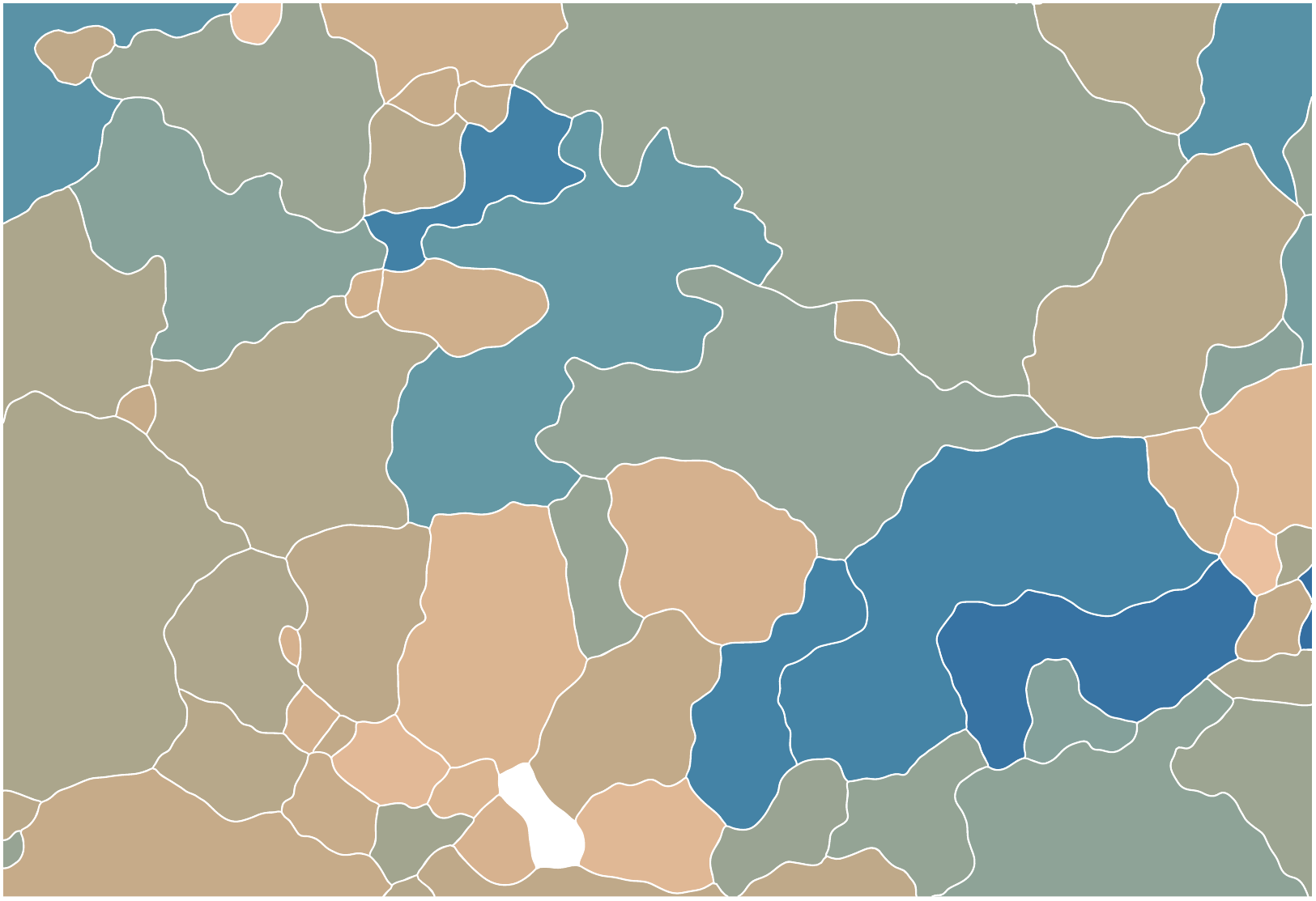}
    \caption{(Example $0 \wedge 1$) Magnified inspection of small scale grain structures subjected to different image preprocessing. Colors of the grains correspond to the normalized distances of scale invariant PRRTI features from the Fr\'echet mean over the learned Grassmannian submanifold. The top image has no smoothing applied to the input orientation map before segmentation. The bottom image utilizes increased smoothing of the boundaries according to MTEX algorithms (parameter equal to five).
    }
    \label{fig:101_zoomed_compare}
\end{figure}

\section{Numerical Experiments}\label{sec:experiments}

We begin with a summary investigating material micrographs with EBSD followed by a study involving a more general gray-scale imaging modality.

Descriptions and results of EBSD numerical experiments are illustrated and explained in Figures~\ref{fig:compare_grains_111}-\ref{fig:compare_grains_000}. The combined numerical experiments enable explainable binary classifications of ice data~\cite{Fan2020, FAN2021116810} and TRIP780 steel data measured at the National Institute of Standards and Technology (NIST) Material Measurement Laboratory (MML). In aggregate, all examples span Truth Table~\ref{tab:logical_conjunction}.

The preprocessing of EBSD data to extract sorted vertices of segmented boundaries is accomplished with the open-source toolbox MTEX~\cite{Hielscher:cg5083}, version 5.10.2. In conjunction with~\cite{gretton2012kernel}, for all examples, a Gaussian kernel is utilized to approximate MMD and the corresponding scale parameter is inferred using the median Euclidean distance heuristic over a subset of normal coordinates. Coordinates are not normalized---i.e., scaled by approximate components of variation inferred from a subset of coordinates---for the results in Figures~\ref{fig:compare_grains_111}-\ref{fig:compare_grains_000} but the effect of normalization is studied in Figure~\ref{fig:010_decision-normalized}. All decisions in the following experiments are based on a significance level of 1\%.

All of the depicted numerical experiments utilize an undulation dimensionality of $r=150$, $n=500$ quadrature nodes---see Figure~\ref{fig:low_dim_grain} for a representative approximation with this dimensionality---and three (full) dimensions of scale. The filled colors of the shapes in Figures~\ref{fig:compare_grains_111}-\ref{fig:compare_grains_000} represent a normalized shape distance over the product submanifold measured from the approximate mean shape. Note all experiments and subsequent conclusions are predicated on the version of the MTEX software (5.10.2) used to segment and smooth grain boundaries.

These numerical experiment emphasizes four scenarios where our methods deliver benefits to materials scientists and practitioners investigating micrographs with EBSD: 
\begin{enumerate}
    \item Figure~\ref{fig:compare_grains_111}, \textit{the method is not suspected to be overly sensitive to statistical rejection}: we identified data which concluded, empirically, that there were no useful statistical discrepancies in the compared images.
    \item Figure~\ref{fig:compare_grains_100}, \textit{the method emulates conclusions consistent with human observations}: when provided with data which has very clear human recognizable distinctions described as `scale variations,' the method detects statistical differences in the measured shapes and results in a theoretical interpretation that scale is the significant factor. 
    \item Figure~\ref{fig:compare_grains_010}, \textit{the method detects differences below human-scale visual observations}: given an image processing routine that smooths an EBSD image differently, we demonstrate how the approach detects nonlinear variations in shapes which would otherwise go overlooked by cursory human inspection.
    \item Figure~\ref{fig:compare_grains_000}, \textit{the method detects a presumed faulty measurement}: when provided with data which was suspected to be `out-of-focus' utilizing the state-of-the-art measurement instrumentation, we were able to detect significant differences to indicate the presence of a problem. 
\end{enumerate}

Notably, in the third case $(0 \wedge 1)$ depicted and described in Figure~\ref{fig:compare_grains_010}, we suspect a cursory human inspection of these small scale differences in shape would likely go overlooked. Thus, the method is capable of inferring differences below the resolution of manual inspections which makes the classification superior to an alternative approach requiring human intervention. Moreover, any method relying on initial inspections to determine hand-picked characteristics of interest may overlook these small-scale differences. In application, this is helpful towards \textit{detecting the presence of any preprocessing or synthetic generation of data}. Specifically, this example was designed by taking the first case, ($1 \wedge 1$), then processing the full set of boundaries differently---i.e., by smoothing or not. Figure~\ref{fig:101_zoomed_compare} offers a closer inspection of differences in preprocessing. 

Despite these nearly visually indistinguishable discrepancies in undulation, we reject the null hypothesis with a significance level of 1\% and conclude that smoothing changes the distribution of shapes. This is a powerful example emphasizing the granularity of decision making in our proposed framework which can detect changes in the preprocessing of the image/data over remarkably small scales of undulation.

\subsection{Decision Landscapes}

The explainable binary classification in all cases is predicated on the selected shape dimensionality, $r$, number of quadrature nodes, $n$, and kernel heuristics like scale parameter selection and coordinate normalization. However, given the computational efficiency of the method, we can easily motivate parameter studies to understand the \textit{decision landscape} over either set of separate coordinates. In the case $(0 \wedge 1)$, we study decisions over Grassmannian undulations, $\text{MMD}(\mathcal{H},\rho_{\boldsymbol{t}},\widehat{\rho}_{\boldsymbol{t}}) $.

As an example decision landscape, we reuse PIL184 data from the designed case $(0 \wedge 1)$ of Figure~\ref{fig:compare_grains_010} with different image processing but now include duplicate grain shapes in the comparison. In other words, the first set of undulation coordinates $\lbrace \boldsymbol{t} \rbrace$ are computed without smoothing while the second set of undulation coordinates $\lbrace \widehat{\boldsymbol{t}} \rbrace$ are computed using a nominal level of smoothing---MTEX smoothing parameter equal to five. Duplicate grain shapes are included---i.e., comparing two identical sample sets of PIL184 boundaries with $N= 933$---up to smoothing thus emphasizing discrepancies attributed entirely to boundary smoothness while eliminating variability due to partitioning the data as in Figure~\ref{fig:compare_grains_010}. 

Figures~\ref{fig:010_decision} and~\ref{fig:010_decision-normalized} depict \textit{decision landscapes} over changing shape dimensionality and quadrature nodes. Figure~\ref{fig:010_decision} utilizes cyclic Procrustes registrations against a random archetype from the full ensemble but does not include (normal) coordinate normalization. In contrast, Figure~\ref{fig:010_decision-normalized} depicts results obtained by scaling normal coordinates but randomly permuting rows of the discrete grain shapes. Coordinate normalization is achieved by scaling all coordinates with the approximated component-wise variation of the noisy samples, $\lbrace \boldsymbol{t}\rbrace$. 

\begin{figure}
    \centering
    \includegraphics[width=0.515\textwidth]{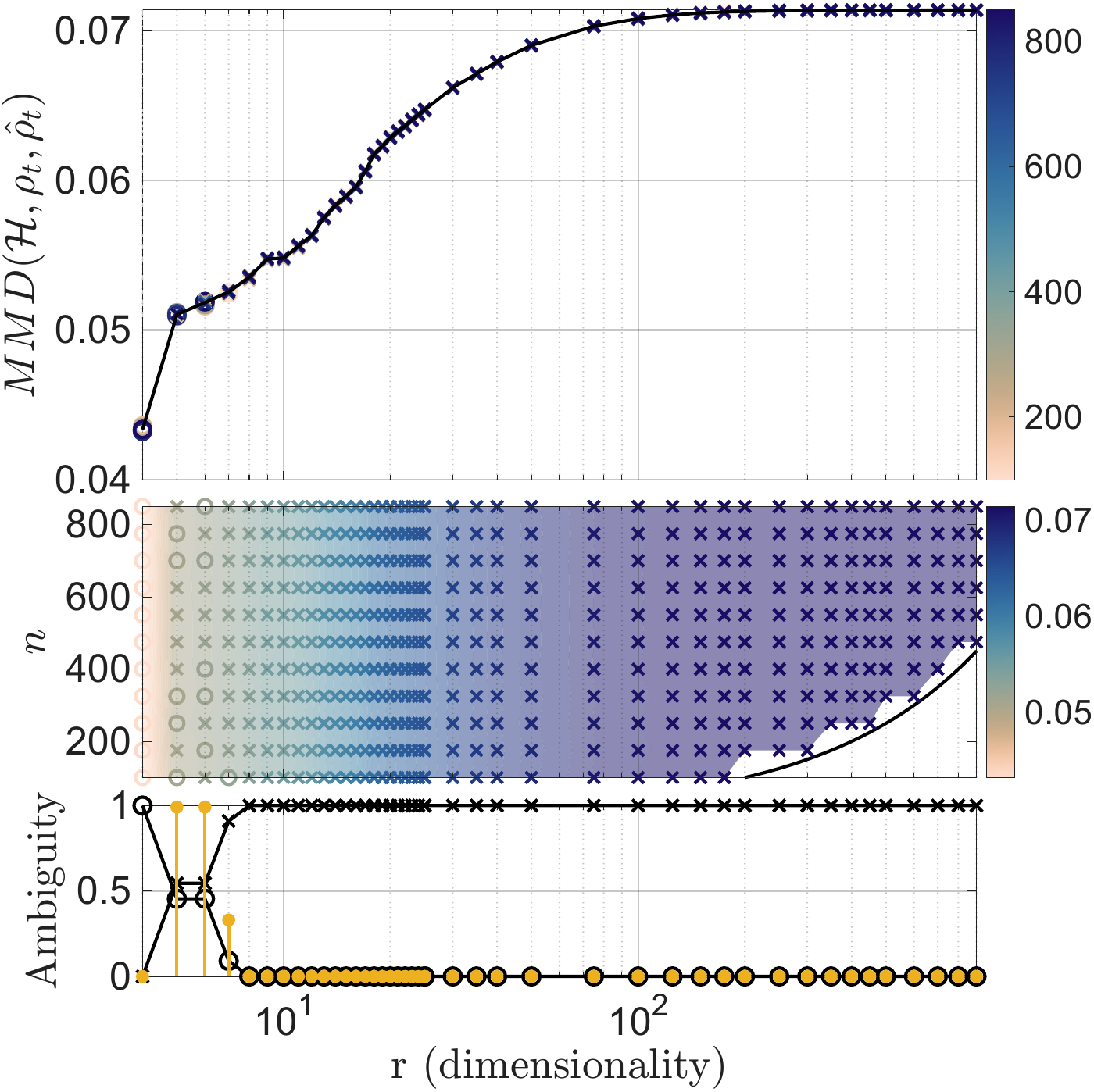}
    \caption{(Example $0 \wedge 1$) Changes in the undulation MMD value over shape dimensionality, $r$, and number of quadrature nodes, $n$, without coordinate normalization but including cyclic Procrustes registrations. (top) Colors indicate the number of quadrature nodes while crosses correspond to rejecting the null hypothesis and circles correspond to failure to reject based on a significance level of 1\%. The solid black line is the conditional average of MMD over quadrature nodes. (middle) Colors indicate MMD value over both $r$ (horizontal axis) and $n$ (vertical axis). The black line in the bottom right corner of the contour plot represents the upper bound of the Grassmannian intrinsic dimensionality. (bottom) Ambiguity in decision making emphasizing the proportion of failure to reject (circles) and reject (crosses) over collocation levels. The yellow stem plot represents the \textit{ambiguity} as the scaled product of these two proportions.
    }
    \label{fig:010_decision}
\end{figure}
\begin{figure}
    \centering
    \includegraphics[width=0.5\textwidth]{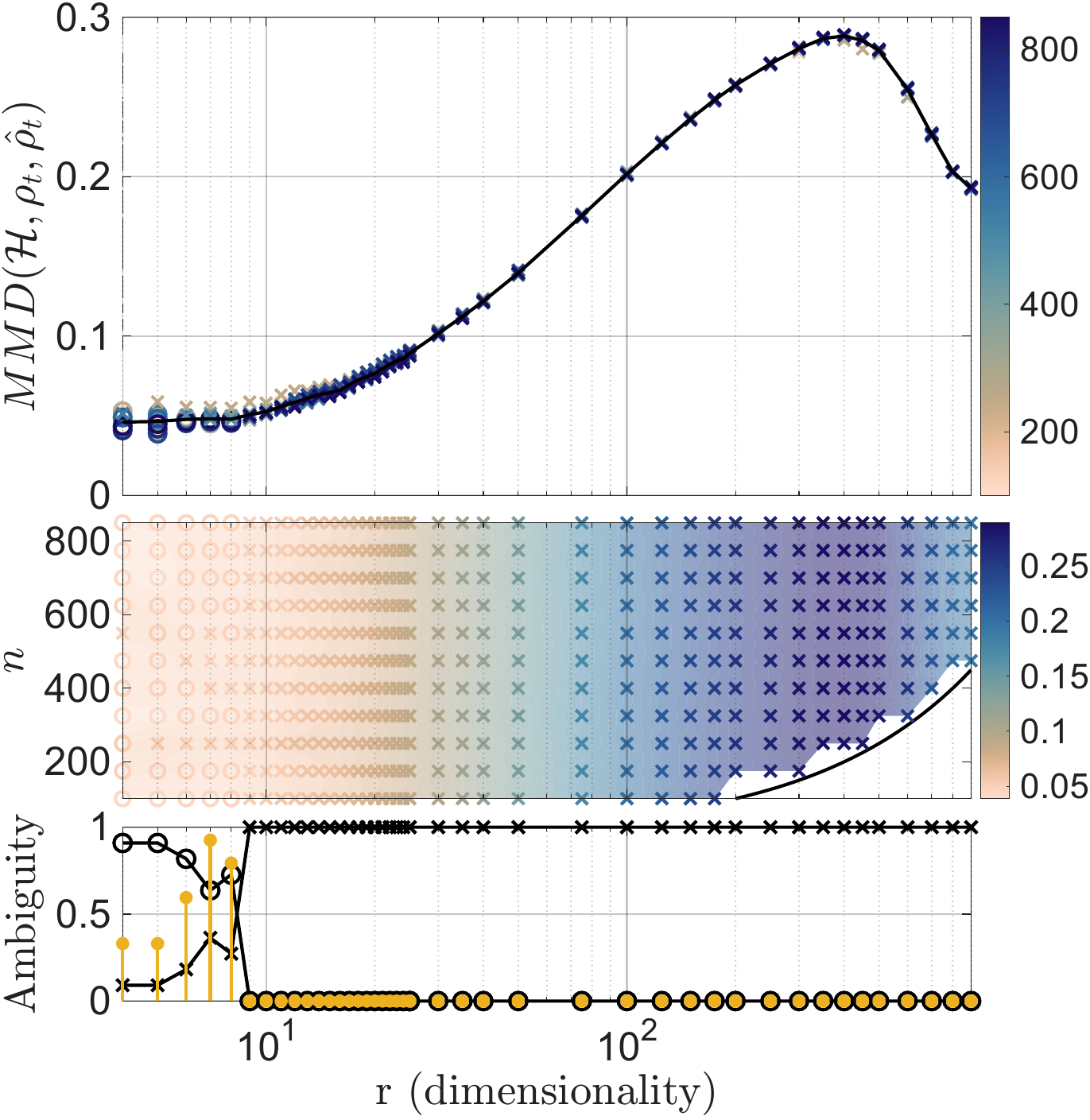}    
    \caption{(Example $0 \wedge 1$) Changes in the undulation MMD value over shape dimensionality, $r$, and number of quadrature nodes, $n$, with coordinate normalization. Shape data are not registered with cyclic Procrustes---i.e., data represent random row-wise permutations. Description is consistent with Figure~\ref{fig:010_decision}. Note the distinct MMD scales compared to Figure~\ref{fig:010_decision}.
    }
    \label{fig:010_decision-normalized}
\end{figure}

Given the changing decision landscape, it is helpful to introduce \textit{ambiguity},
$
4 \mathbb{P}_n( H = 1 \vert r)(1 - \mathbb{P}_n( H = 1 \vert r)),
$
where $H \in \lbrace 0, 1\rbrace$ is a Bernoulli trial representing the binary decision contrasting distributions of undulation coordinates, $\boldsymbol{t},\widehat{\boldsymbol{t}} \in T_{[\widetilde{V}_0]}\mathcal{G}_r$, and $\mathbb{P}_n(\cdot \vert r)$ is the conditional probability of failing to reject the null hypothesis at a particular shape dimensionality. The probability is estimated as the sum of binary `accept' conditions over the total number of quadrature discretizations at the corresponding level of $r$ (dimensionality). By definition, an ambiguity of one corresponds to an entirely random guess which should not be trusted in practice.

Examining Figure~\ref{fig:010_decision}, empirical regularization over reduced shape dimensionality---as described in Figure~\ref{fig:low_dim_grain}---results in no statistically significant differences between the distributions of undulation when $r \leq 7$. In other words, without sufficient dimensionality representing undulations, empirical regularization results in no significant discrepancy. Moreover, there is very limited variability in MMD values over different quadrature collocations. This is a testament to the stability of the cyclic Procrustes random archetype registrations over changing $n$. At $r=8$ and beyond, the binary decision consistently rejects the null hypothesis indicating statistically significant differences in undulation. Ambiguity is nonzero over $5\leq r\leq 7$ and peaks at $r=5,6$ with maximum ambiguity of approximately one.

Examining Figure~\ref{fig:010_decision-normalized}, coordinate normalization also appears to modulate variability in approximated MMD values despite randomly permuted quadrature nodes. The only notable (visible) conditional variation in MMD value occurs over small shape dimensionality, $r \leq 20$. We reject the null hypothesis consistently beyond $r = 8$ while $r=7$ corresponds to maximum ambiguity of approximately one. Surprisingly, MMD values in Figure~\ref{fig:010_decision-normalized} do not exhibit the anticipated monotonicity over $r$ as in Figure~\ref{fig:010_decision}. This may suggest that increased dimensionality can promote cancellations in discrepancy given otherwise random curve alignments unlike the cyclic Procrustes alignments of Figure~\ref{fig:010_decision} exhibiting cumulative (monotonic) discrepancy over dimensionality. 

The global trends between the two decision landscapes are distinct but the ability to circumvent row-wise registrations by simply normalizing coordinates could offer massive computational advantages---i.e., optimal alignments may be less important than previously hypothesized if utilizing coordinate normalization. In practice, decision landscapes should be utilized to understand sensitivities to decision making in any replicated or extended work.

\subsection{Segmentation Efficacy}
\begin{figure}
    \centering
    \includegraphics[width=1\linewidth]{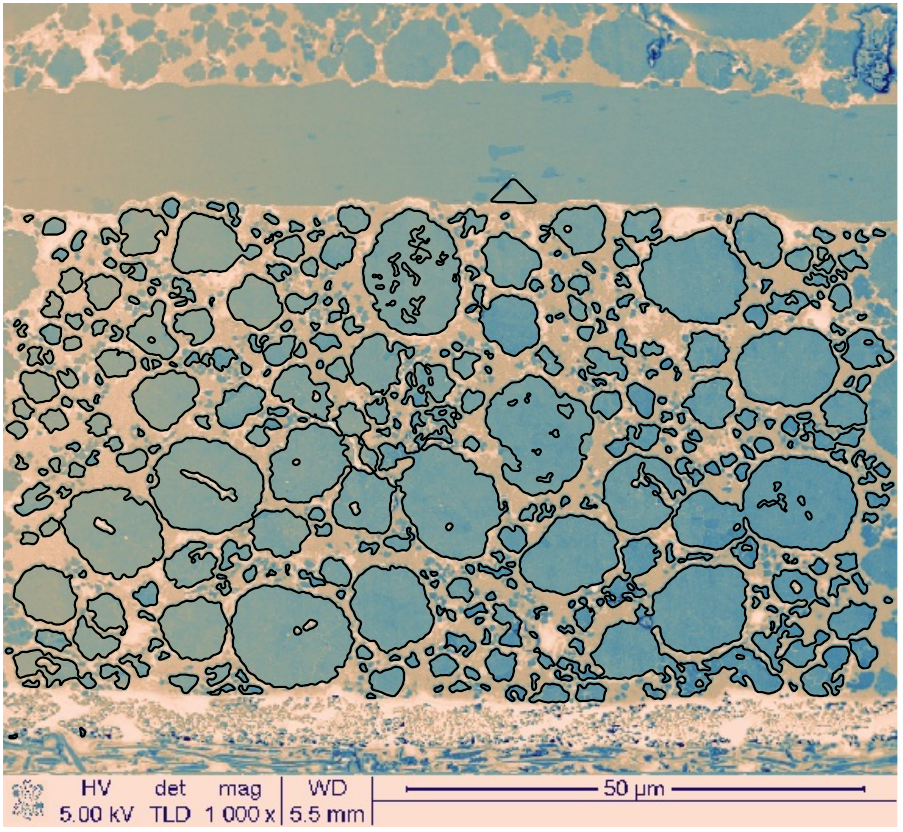}
    \includegraphics[width=1\linewidth]{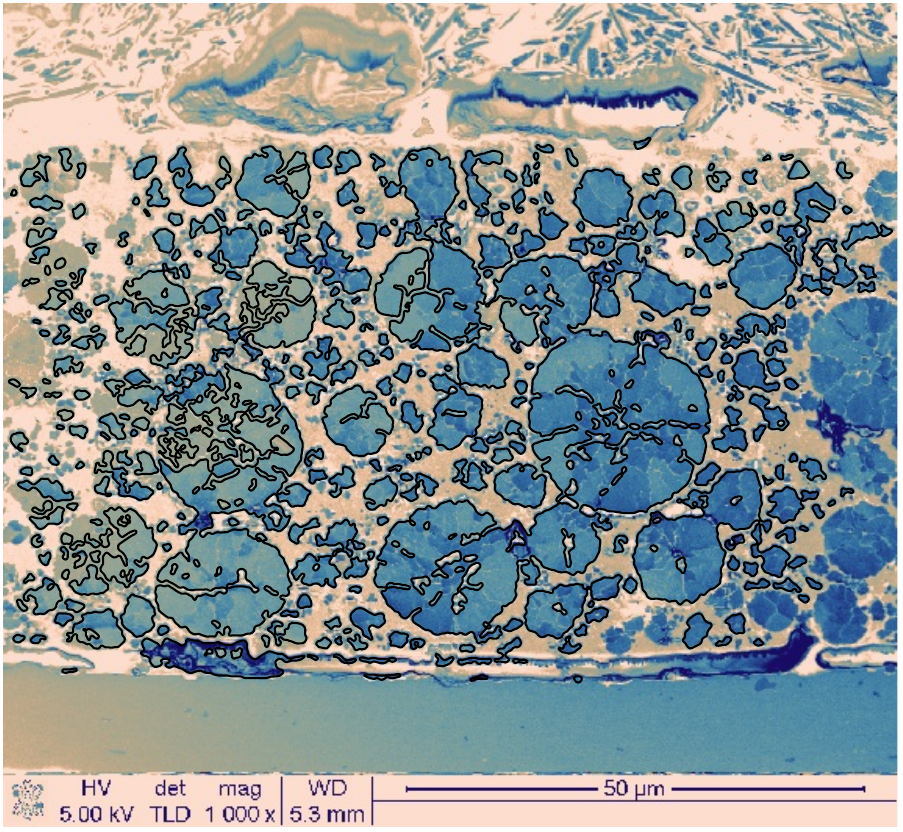}
    \caption{(top) Cross-section of a `new' condition lithium-ion battery with segmented black curves over scalar-valued image intensity. (bottom) Cross-section of a 'used' condition lithium-ion battery with segmented curves (black boundaries) over scalar-valued image intensity. The new condition battery exhibits less fracturing in the NMC particles compared to the battery which has been subjected to accelerated cycling in the lab. A smoothing length scale $\sigma = 5$ is used to extract the depicted curves.}
    \label{fig:new_v_old_bat}
\end{figure}
\begin{figure*}
    \centering
    \includegraphics[width=0.2\linewidth]{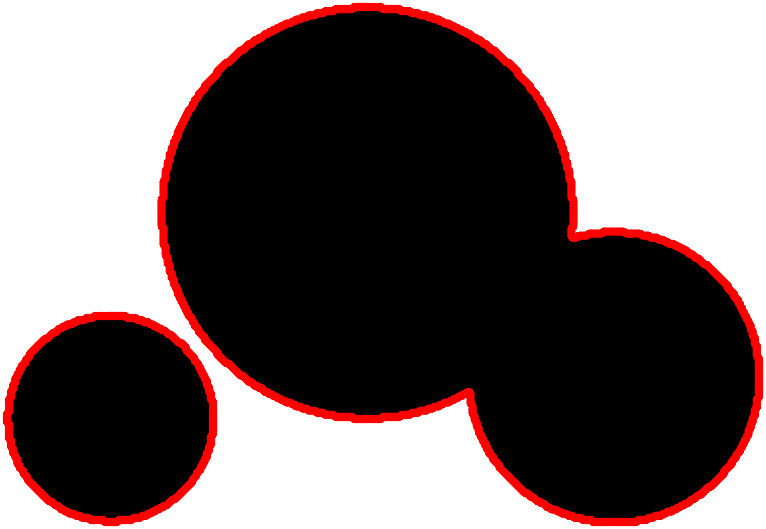}
    \hspace*{2em}
    \includegraphics[width=0.2\linewidth]{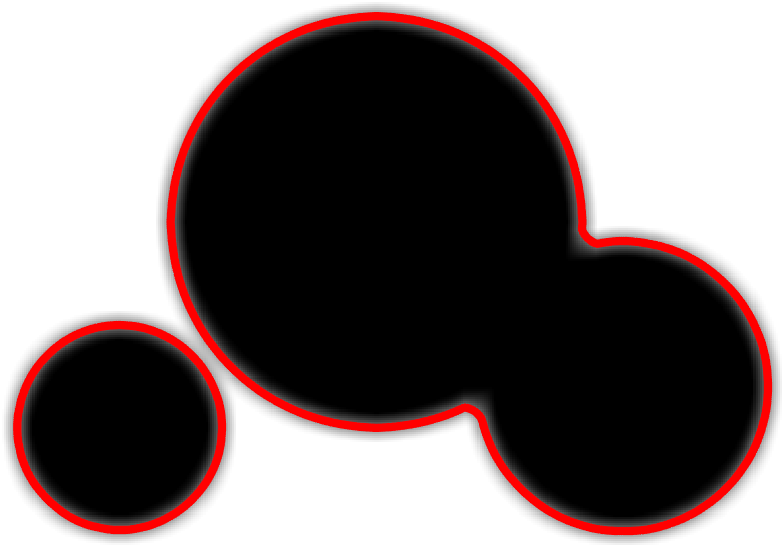}
    \hspace*{2em}
    \includegraphics[width=0.2\linewidth]{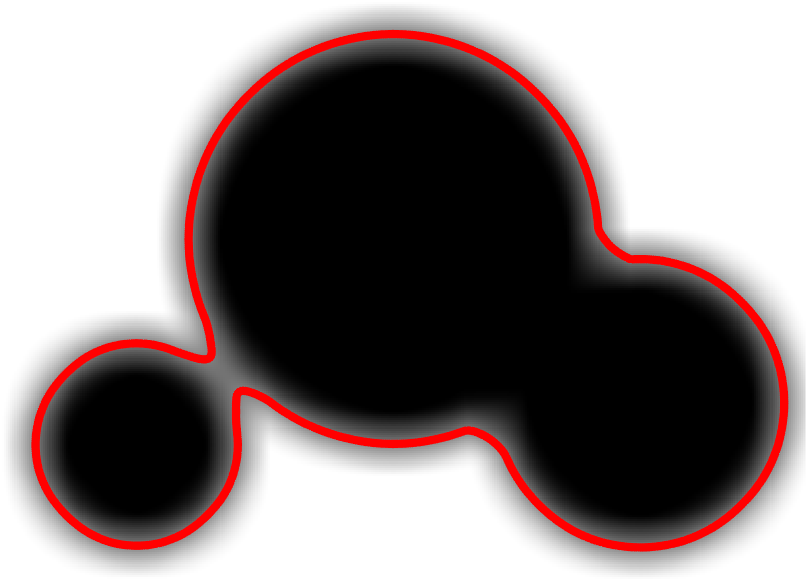}
    \\
    \includegraphics[width=0.475\linewidth]{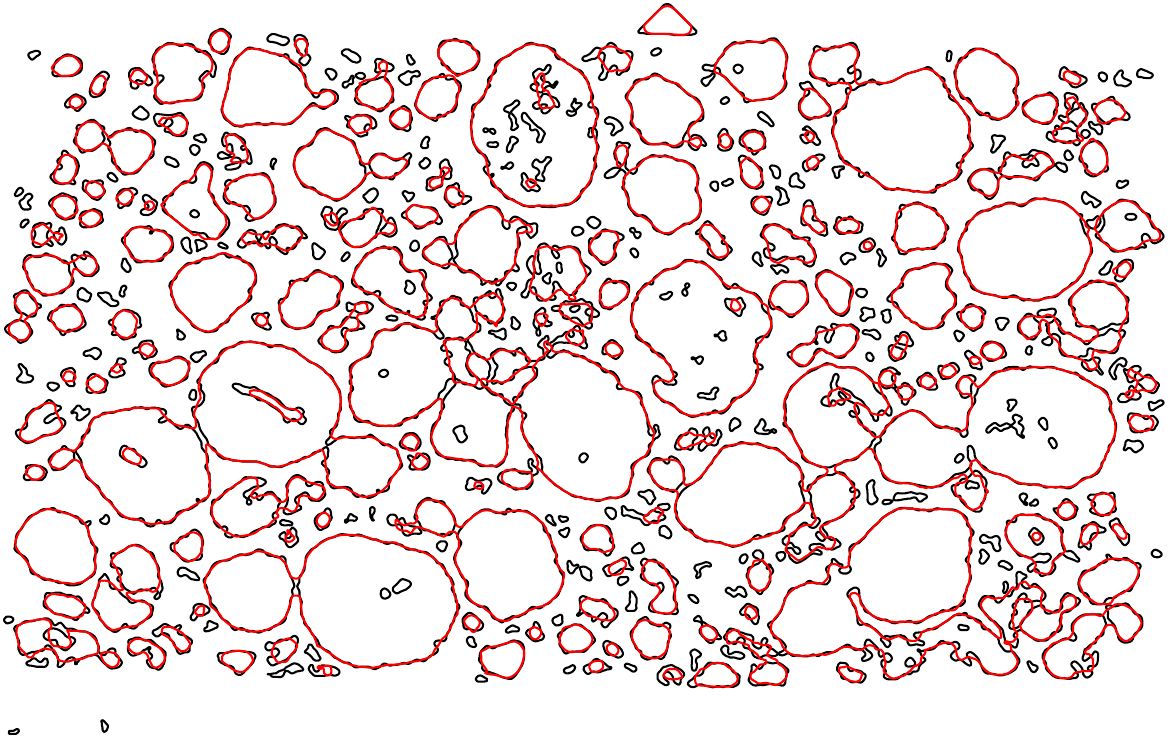}
    \hspace*{0.5em}
    \includegraphics[width=0.475\linewidth]{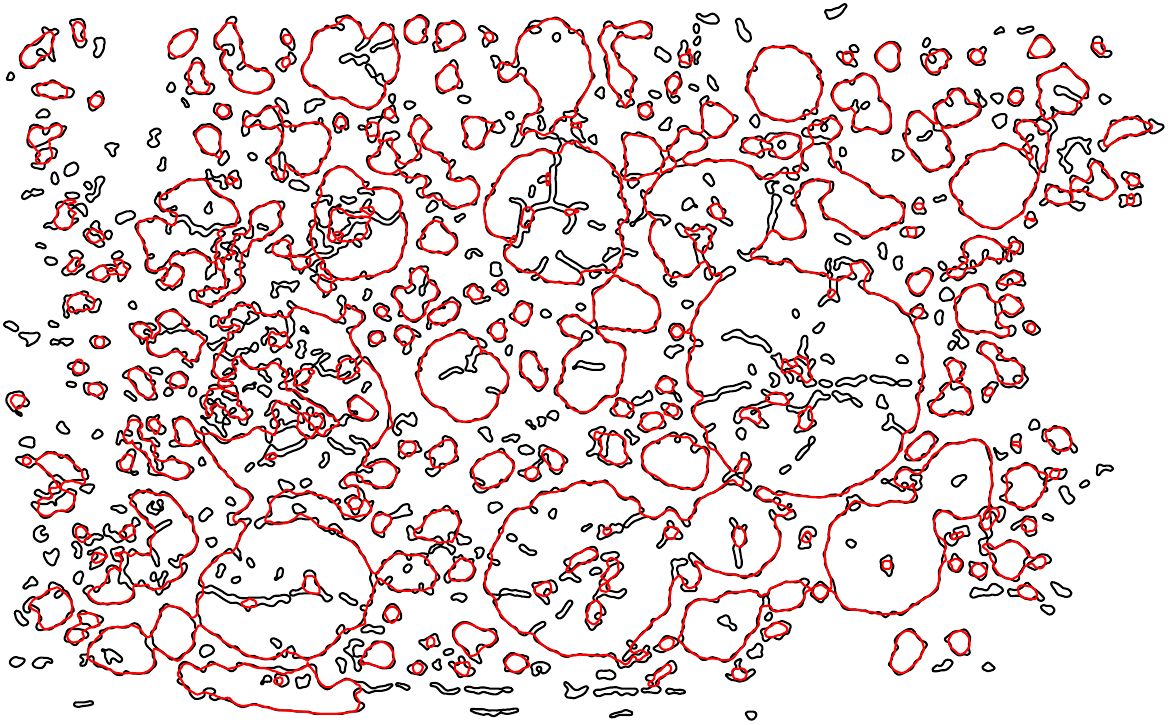}
    \caption{(top) The effect of increased Gaussian length scale convolutions applied to a simple binary mask. From left to right, the scale is increased as $\sigma=1,15,30$ over a $400$-by$400$ pixel grid---i.e., the relative size of these simple circular objects versus the image resolution is much greater than the NMC particles in an SEM. At larger scales, increased smoothing begins to blur and merge boundaries reducing the ensemble size. (bottom) For reference, ensembles from the same pair of images in Fig.~\ref{fig:new_v_old_bat} are segmented at $\sigma = 5$ (black boundaries) and $\sigma = 15$ (red boundaries). Clearly, increased smoothing begins to obscure the important distinguishing fractures and fissures of the 'used' condition battery.}
    \label{fig:Gauss_Smooth_Mask}
\end{figure*}

\begin{figure}
    \centering
    \includegraphics[width=1\linewidth]{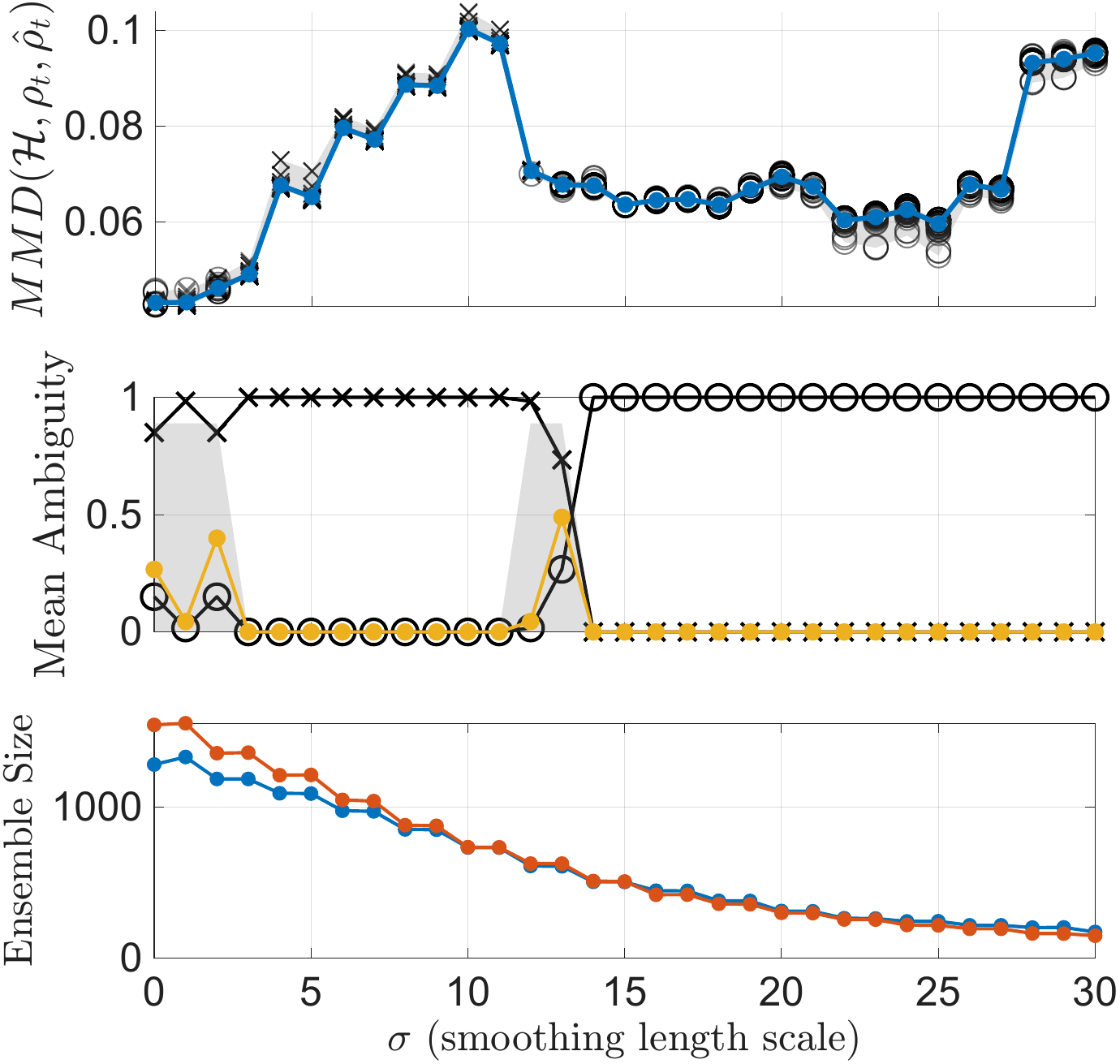}
    \caption{(top) Conditionally averaged undulation MMD over smoothing scales applied to the segmentation binary mask. Raw data is depicted with crosses or circles depending on the decision to reject or failure to reject the null hypothesis, respectively. (middle) Mean ambiguity over dimensionality at the corresponding convolution length scale. The gray shaded region is a min-max envelope over all combinations of $r$ and $n$. (bottom) Changing ensemble sizes from three segmented images of `new' condition batteries (blue) versus three segmented images of `used' condition batteries (orange) over smoothing length scales.}
    \label{fig:avg_decision_land}
\end{figure}

Across imaging applications, feature segmentation is a challenging step in analysis workflows. Despite modern advancements, segmentation methods remain highly dependent on multiple factors, including image quality and hyperparameters---e.g., smoothing parameters, thresholding values, training with hand-labeled data, etc. Variability in segmentation algorithms can result in different ensembles of shapes extracted from a given image, which can complicate subsequent assessments of patterns in the image. Here, we briefly examine the utility of explainable binary classification for introducing novel, quantifiable metrics to compare and contrast the efficacy of a given image segmentation.

% Battery background - 1 paragraph
The use-case for this study is motivated by the need to characterize degradation in lithium-ion batteries to better understand battery aging and performance decay. Improved degradation models can lead to improved battery design, charging, and replacement strategies. Scanning electron microscope (SEM) imaging is often used to enabling high-resolution visualizations of batteries throughout aging. In particular, fracturing and formation of fissures in the nickle-manganese-cobalt (NMC) particles visible in SEM are notable markers of this decay. Robust analysis of the degradation process relies on meaningful characterization of the damage to these NMC particles, making this an appropriate test case of studying the interplay between segmentation approaches and explainable binary classification.

% Experiment setup
To quantify the nature of the fractures and fissures which form during lithium-ion battery degradation, we run simple morphological segmentation over SEM images for three `new' condition battery and three `used' condition battery. An example pair of new and used condition image segmentations are shown in Figure~\ref{fig:new_v_old_bat}. Each image is $1024$-by-$1024$ and down-selected to a subregion of interest containing the NMC particles.

The output of these segmentations is a binary mask which is then smoothed with a Gaussian convolution corresponding to a parametric length scale, $\sigma$, to remove pixel `stepping.' The length scale is sized according to the number of pixels and we use the $0.5$ level-set of the smoothed mask to extract corresponding boundary curves. Figure~\ref{fig:Gauss_Smooth_Mask} depicts the effect of smoothing the binary mask to reduce noisy pixel stepping in the segmentation.

Notice the segmentation of any given image is imperfect---some particle boundaries are inappropriately merged, others are omitted entirely or partially, and artifacts are present. However, in aggregate, the ensemble of extracted curves appear to represent distinct patterns which admit a seemingly obvious dependency on the fractured state of the NMC particles. The goldilocks paradigm of interest: \textit{how much smoothing is required to denoise and achieve consistent decision making prior to excessively blurring the mask?} Too little smoothing and high frequency undulations may obscure relevant decision making. Too much smoothing and we obfuscate the segmentation results.

We apply explainable binary classification to the separable shape ensembles extracted from simple, smoothed morphological segmentations over three pairs of new and used battery SEM images---i.e., six SEM images in total. Shape dimensionality, $r$, is varied from $5$ to $100$ in increments of $5$, (a separate view of the data emphasizes converged MMD values at approximately $r=50$) quadrature levels $n$ are taken sparsely at $100$, $500$, $1000$ given limited variability in MMD values over $n$, and we sweep over the convolution length scale, $\sigma=0,\dots,30$, in increments of one. The undefined value $\sigma = 0$ corresponds to simple black-white boundary extraction without smoothing applied to the binary mask. 

% Decision of results
Approximated MMD results depicted and aggregated as conditional averages at corresponding smoothing scales are visualized in Figure~\ref{fig:avg_decision_land}. We notice small amounts of ambiguity initially with steadily increasing discrepancy. Over the range $4\leq \sigma \leq 11$, all decisions result in a consistent rejection of the null hypothesis. Finally, at $\sigma = 13$, the decision landscape inverts presumably due to the increasingly merged and smoothed shape objects. As anticipated by the simple example in Figure~\ref{fig:Gauss_Smooth_Mask}, ensemble sample sizes are shown to steadily decrease with increased smoothing length scales. MMD also appears to increase again at the increasingly smaller ensemble sizes---i.e., the few remaining shapes are distinct but the distributions of normal coordinates are no longer significantly discrepant. Our intuition beyond $\sigma =13$ is that the ensemble has digressed into a sparse collection of seemingly random undulations with increasing discrepancy, albeit insignificant, due to the reduced ensemble size.

In practice, this analysis would suggest that achieving the largest number of segmented shapes in an ensemble with consistent decision making occurs near $\sigma = 4$ or $5$. This study could be used to set parameter ranges for more effective segmentation when detecting failure modalities in subsequent battery experiments. It appears that $r\approx 50$, $n\approx500$, and $\sigma \approx 5$ are suitable candidates to detect the fractured nature for this level of battery degradation while maximizing the number of segmented objects.

\section{Conclusion}\label{sec:conclude}
We have established a formal interpretation between separable shape tensors (SST) and dual evaluation functionals of~\cite{micheli2013matrix}. We demonstrate that spectral methods, based on an interpretation with the Nyst\"om method, can be utilized to offer highly accurate approximations of SSTs provided sufficient regularity of curves. We then elaborate on the alignment of SSTs and caveats of utilizing a maximum mean discrepancy (MMD) analysis to inform explainable binary classification of data.

We explore four real-world numerical examples spanning the logical Truth Table~\ref{tab:logical_conjunction} found in EBSD measurements; identified by simple observations (selected for qualitative appearance), differences in preprocessing (smoothing), and detection of a presumed measurement issue (instrument out-of-focus). Finally, we briefly explore the implications of utilizing these metrics to compare and contrast the effectiveness of gray-scale image segmentation results applied to Lithium-ion batteries.

Future work will expound on the identified relationships between finite matrix manifold representations of shape and the formal interpretations of infinite dimensional analogs detailed in~\cite{micheli2013matrix}. In short, our efforts will be focused towards: i) formalizing bounds on the approximation error of reduced dimension SSTs, and ii) infinite dimensional extensions of pMMD explainable binary classifications over the $\mathbb{R}^d$-valued RKHSs of~\cite{micheli2013matrix}. We also hope to improve explanations in the classification problem to understand which patterns of shape, in more detail, contribute to approximated discrepancies.

\backmatter

\bmhead{Supplementary information}
Not applicable.

\bmhead{Acknowledgements}
We greatly appreciate Adam Creuziger and G\"unay Do\u{g}an at NIST for their guidance regarding applications and important data considerations.

This work was authored in part by the National Laboratory of the Rockies for the U.S. Department of Energy (DOE), operated under Contract No. DE-AC36-08GO28308. This work was supported by the AI4WIND project sponsored by the U.S. Department of Energy Office of Critical Minerals and Energy Innovation Wind Energy Technologies Office. The views expressed in the article do not necessarily represent the views of the DOE or the U.S. Government. This work is U.S. Government work and not protected by U.S. copyright. The U.S. Government retains and the publisher, by accepting the article for publication, acknowledges that the U.S. Government retains a nonexclusive, paid-up, irrevocable, worldwide license to publish or reproduce the published form of this work, or allow others to do so, for U.S. Government purposes.

%Please refer to Journal-level guidance for any specific requirements.

\section*{Declarations}

\begin{itemize}
\item Funding: This work was partially support by a National Institute of Standards \& Technology (NIST) Building the Future (BTF) internal funding initiative and the second author is supported in part by funds from the National Science Foundation (NSF) Research Training Group (RTG) grant (DMS-2136228).
\item Conflict of interest/Competing interests: All authors certify that they have no affiliations with or involvement in any organization or entity with any financial interest in the subject matter or materials discussed in this manuscript.
\item Data availability: Ice micrograph data is available at~\cite{Fan2020}. 
\item Author contribution: The first and second authors contributed equally to this manuscript. The third author facilitated data collection along with significant revisions and edits.
\end{itemize}

\noindent

%%===========================================================================================%%
%% If you are submitting to one of the Nature Portfolio journals, using the eJP submission   %%
%% system, please include the references within the manuscript file itself. You may do this  %%
%% by copying the reference list from your .bbl file, paste it into the main manuscript .tex %%
%% file, and delete the associated \verb+\bibliography+ commands.                            %%
%%===========================================================================================%%

\bibliographystyle{sn-aps}
%% if required, the content of .bbl file can be included here once bbl is generated
%\input sn-article.bbl

\begin{appendices}

\section{EBSD Examples}\label{secA1}

This appendix contains high resolution images for the numerical experiments discussed in the numerical experiments section of the paper. Captions contain relevant descriptions of results spanning Table~\ref{tab:logical_conjunction}.

\begin{figure*}
    \centering
    \includegraphics[width=1\textwidth]{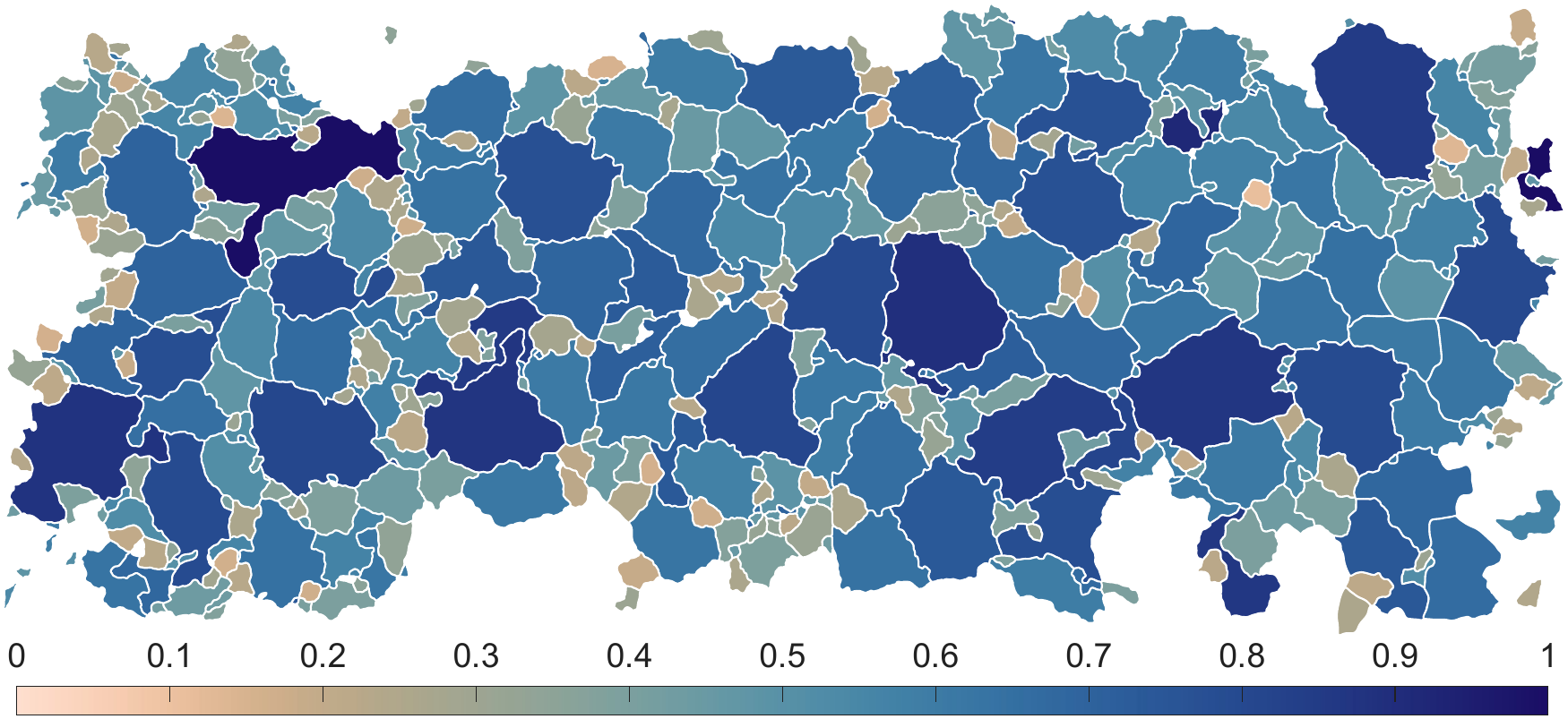} %/img/hlf1_PIL184-smth5.pdf
    \includegraphics[width=1\textwidth]{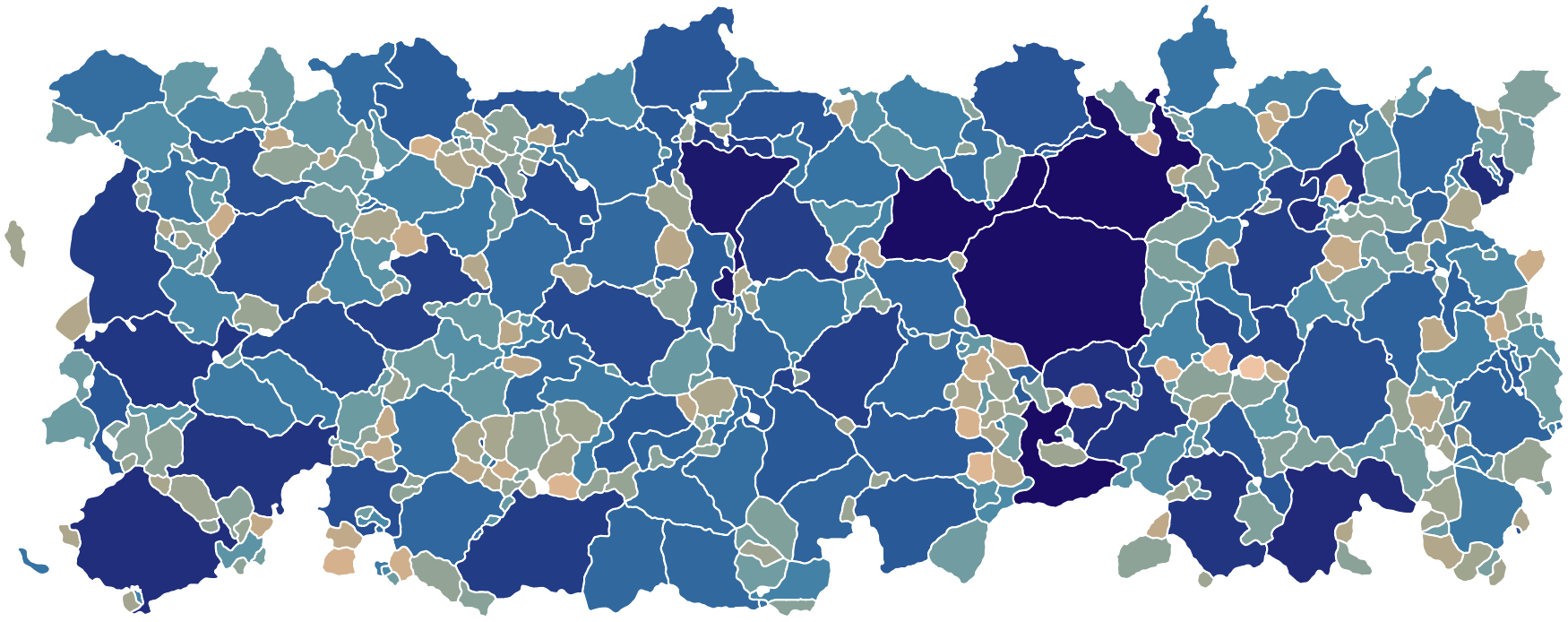}
    \caption{(Example $1 \wedge 1$) Comparison of locally partitioned subsets of a larger EBSD images from a common ice sample~\cite{Fan2020} (PIL184) with $N=933$ total curves---i.e., the set of curves is split approximately in half. No grains are duplicated between the partition in this example. The approximate inferences suggest the same undulations and same scales with significance level of 1\%. This constitutes an example of the first row of Truth Table~\ref{tab:logical_conjunction}, i.e., $1 \wedge 1$.}
    \label{fig:compare_grains_111}
\end{figure*}

\begin{figure*}
    \centering
    \includegraphics[width=1\textwidth]{img/hlf1_PIL184-smth5.pdf}
    \includegraphics[width=1\textwidth]{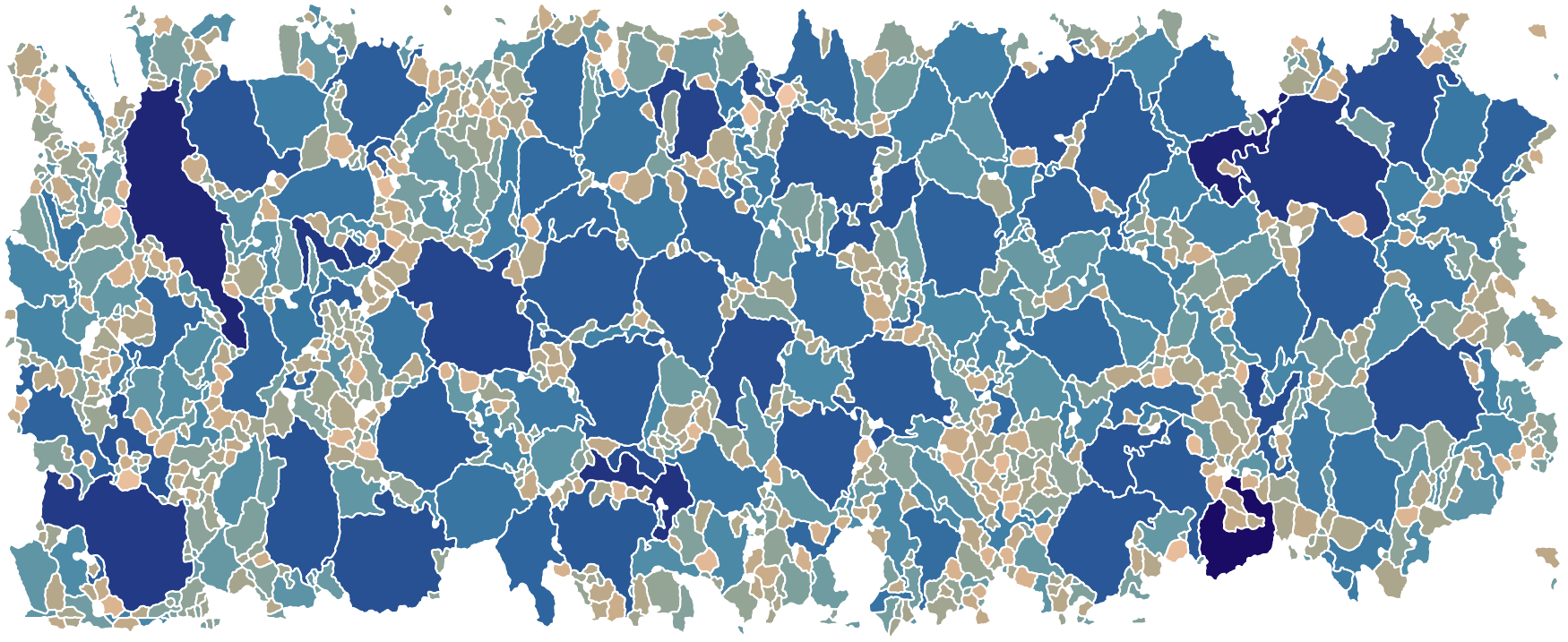}
    \caption{(Example $1 \wedge 0$) Comparison of two different EBSD images of ice samples~\cite{Fan2020} (PIL184 and PIL185) with $N=2566$ curves in total. Both images depict cropped shape selections. (there are additional shapes beyond the selection depicted in the images) The approximate inferences suggest the same undulations but different scales between the two images with significance level of 1\%. This constitutes an example of the second row of Truth Table~\ref{tab:logical_conjunction}, i.e., $1 \wedge 0$.}
    \label{fig:compare_grains_100}
\end{figure*}
\begin{figure*}
    \centering
    \includegraphics[width=1\textwidth]{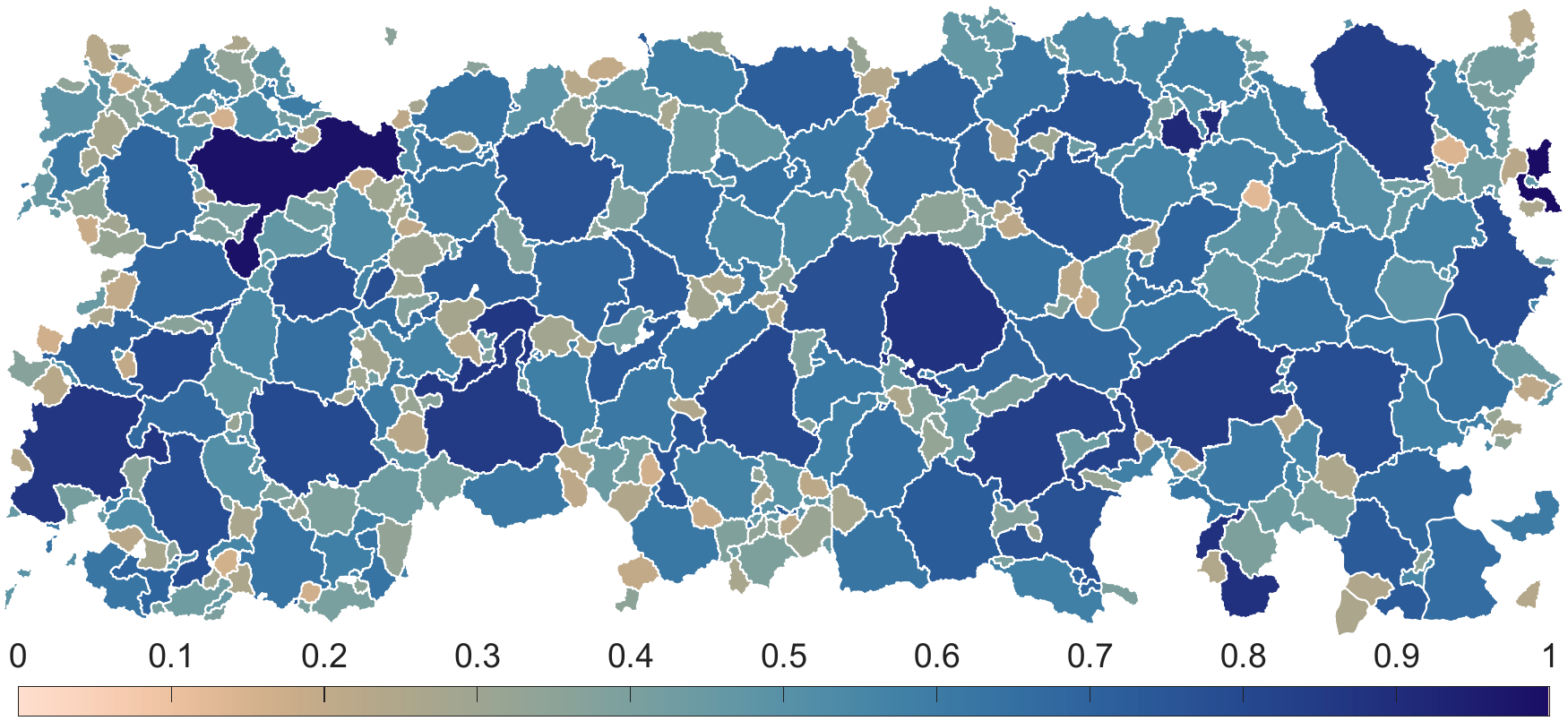}
    \includegraphics[width=1\textwidth]{img/hlf2_PIL184-smth5.pdf}
    \caption{(Example $0 \wedge 1$) Comparison of locally partitioned subsets of a larger EBSD images from a common ice sample~\cite{Fan2020} (PIL184) with $N=933$ total curves. No shapes or boundaries are duplicated between the partition and this partition is consistent with that of Figure~\ref{fig:compare_grains_111}. However, using a default MTEX `smooth' routine~\cite{bachmann2010texture}, the upper partition (top image) is not smoothed while the bottom is smoothed to an extent (MTEX smoothing parameter equal to $5$). Magnifying and contrasting the top image with the bottom image of Figure~\ref{fig:compare_grains_100} it is possible to see small differences in the nonlinearities of shapes. The difference is nearly visually indistinguishable. Yet, the approximate inferences suggest differences in undulations and common scales. This constitutes an example of the third row of Truth Table~\ref{tab:logical_conjunction}, i.e., $0 \wedge 1$, with significance level of 1\%. We note, given the partition into distinct grain subsets and subsequently reduced sample size, this decision landscape has a more volatile ambiguity profile than comparing identical grains up to smoothing.}
    \label{fig:compare_grains_010}
\end{figure*}

\begin{figure*}
    \centering
    \includegraphics[width=1\textwidth]{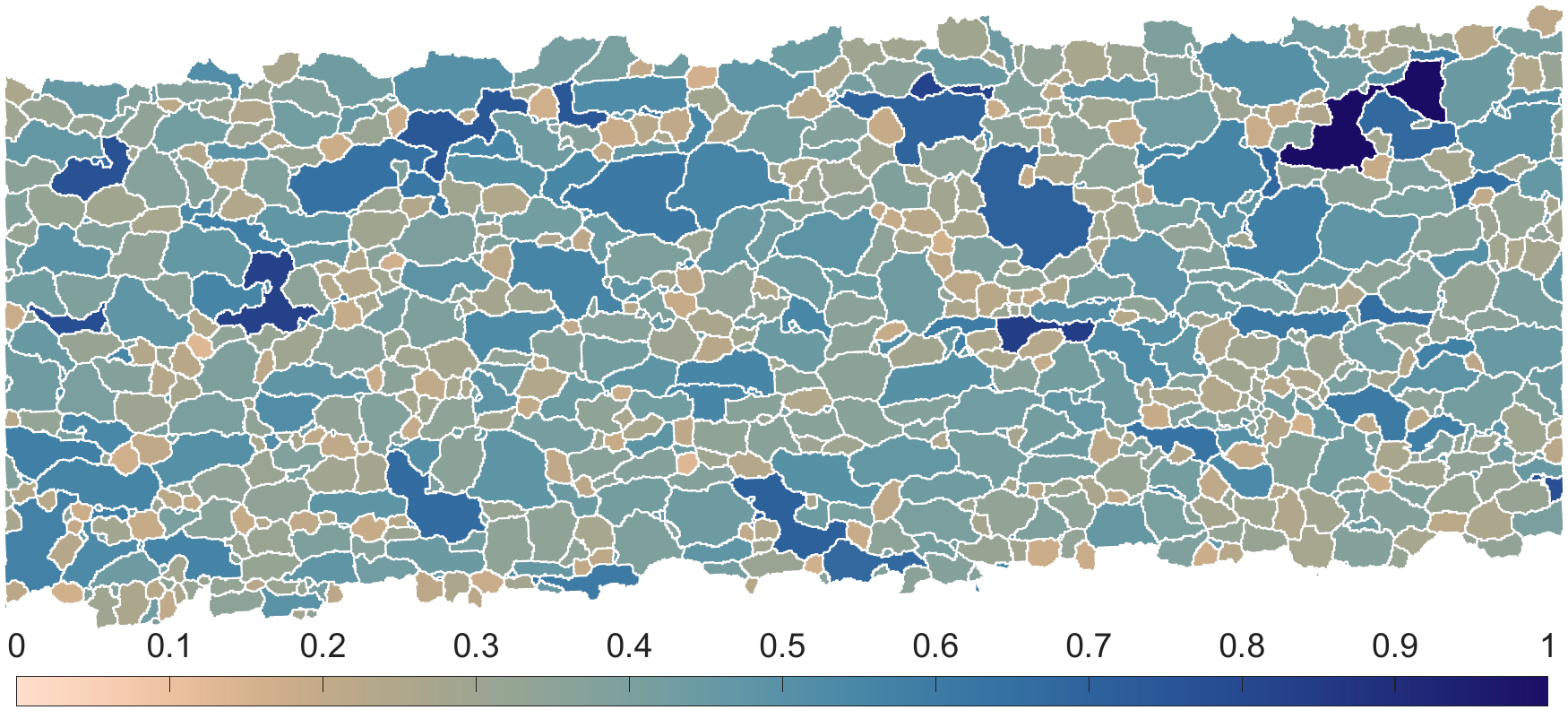}
    \includegraphics[width=1\textwidth]{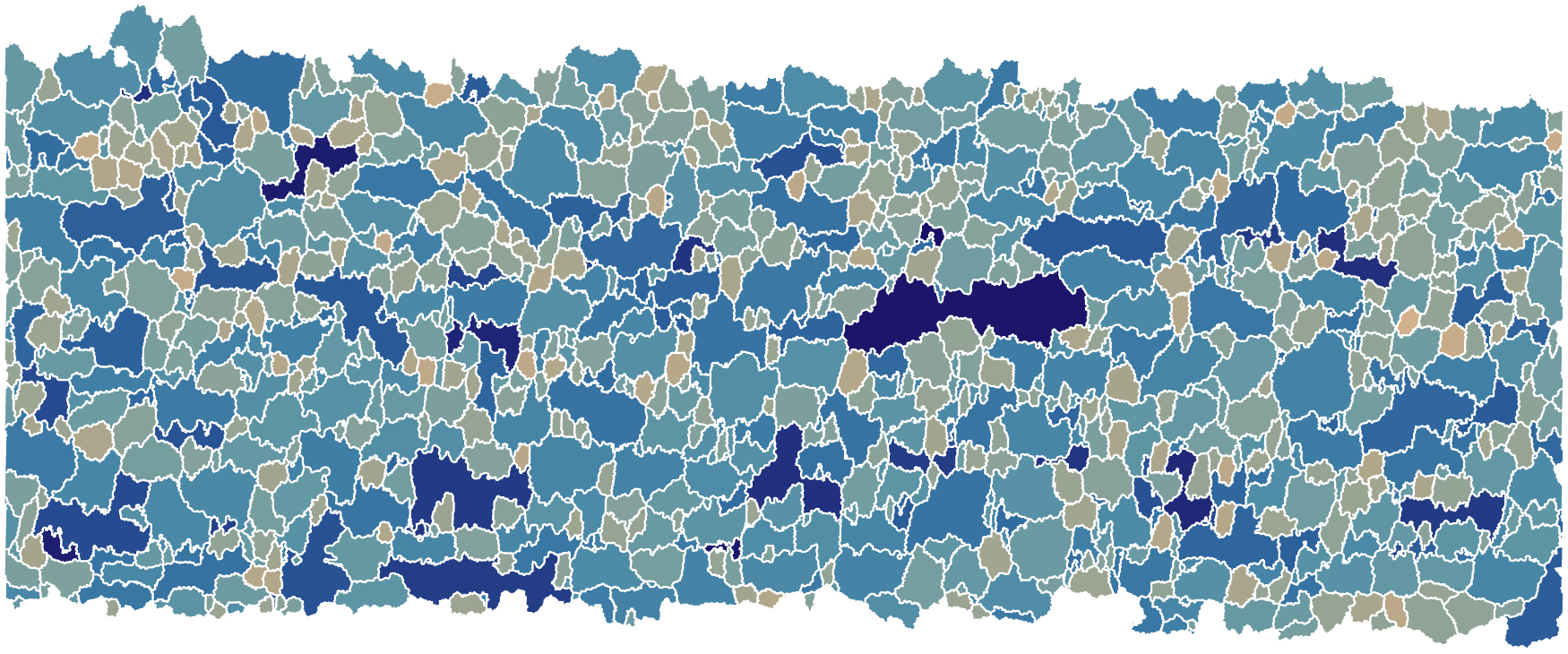}
    \caption{(Example $0 \wedge 0$) Comparison of two different EBSD images of TRIP780 steel samples with $N=5605$ curves in total. Both images depict cropped shape selections to place emphasis on small scale features of undulation. The approximate inferences suggest differences in both undulations and scales between the two images with significance level of 1\%. This constitutes an example of the last (fourth) row of Truth Table~\ref{tab:logical_conjunction}, i.e., $0 \wedge 0$. Note, the hypothesis from a NIST materials scientist, Adam Creuziger, is that the bottom sample was collected from a measurement which was `out-of-focus.'}
    \label{fig:compare_grains_000}
\end{figure*}

%%=============================================%%
%% For submissions to Nature Portfolio Journals %%
%% please use the heading ``Extended Data''.   %%
%%=============================================%%

%%=============================================================%%
%% Sample for another appendix section			       %%
%%=============================================================%%

%% \section{Example of another appendix section}\label{secA2}%
%% Appendices may be used for helpful, supporting or essential material that would otherwise 
%% clutter, break up or be distracting to the text. Appendices can consist of sections, figures, 
%% tables and equations etc.

\end{appendices}
%If any of the sections are not relevant to your manuscript, please include the heading and write `Not applicable' for that section. 

\end{document}